\documentclass[lettersize,journal]{IEEEtran}
\usepackage{cite}
\usepackage{amsmath,amssymb,amsfonts}
\usepackage{amsthm}
\def\softmax{\operatornamewithlimits{%
  \mathchoice{\vcenter{\hbox{ softmax}}}
             {\vcenter{\hbox{ softmax}}}
             {\mathrm{softmax}}
             {\mathrm{softmax}}}}
\usepackage{graphicx}
\usepackage{textcomp}
\usepackage{xcolor}
\usepackage{balance} 
\usepackage{dsfont}
\usepackage{booktabs}
\usepackage{float}
\usepackage{subfigure}
\usepackage{scrextend}
\usepackage{url}
\usepackage{algorithm}
\usepackage{algorithmicx,algpseudocode}
\usepackage{todonotes}
\usepackage{multirow}
\usepackage{algpseudocode}
\usepackage{booktabs}

\usepackage{array} 
\usepackage{tikz,balance}
\usepackage{enumitem,colortbl}
\usepackage{mathtools}

\newtheorem{theorem}{Theorem}

\newtheorem{assumption}{Assumption}
\newtheorem{case}{Case}

\newcommand{\sS}{\mathcal{S}}
\newcommand{\sA}{\mathcal{A}}
\newcommand{\sP}{\mathcal{P}}
\newcommand{\sR}{\mathcal{R}}
\newcommand{\sC}{\mathcal{C}}

\begin{document}

\title{Robust Dynamic Material Handling via Adaptive Constrained Evolutionary Reinforcement Learning}
\author{Chengpeng~Hu, Ziming~Wang,
Bo~Yuan,~\IEEEmembership{Member,~IEEE,}
Jialin~Liu,~\IEEEmembership{Senior Member,~IEEE,}
Chengqi~Zhang,~\IEEEmembership{Senior Member,~IEEE,}
Xin~Yao,~\IEEEmembership{Fellow,~IEEE}
\thanks{C. Hu is with Eindhoven University of Technology, Eindhoven, Netherlands. Z. Wang and B. Yuan are with the Southern University of Science and Technology, Shenzhen, China. J. Liu and X. Yao are with the Lingnan University, Hong Kong SAR, China. C. Zhang is with the Hong Kong Polytechnic University, Hong Kong SAR, China. 
}
\thanks{This work is accepted by IEEE Transactions on Neural Networks and Learning Systems}
}

\markboth{Journal of \LaTeX\ Class Files,~Vol.~14, No.~8, August~2021}%
{Shell \MakeLowercase{\textit{et al.}}: A Sample Article Using IEEEtran.cls for IEEE Journals}


\maketitle

\begin{abstract}
Dynamic material handling (DMH) involves the assignment of dynamically arriving material transporting tasks to suitable vehicles in real time for minimising makespan and tardiness.
In real-world scenarios, historical task records are usually available, which enables the training of a decision policy on multiple instances consisting of historical records.
Recently, reinforcement learning has been applied to solve DMH. Due to the occurrence of dynamic events such as new tasks, adaptability is highly required. Solving DMH is challenging since constraints including task delay should be satisfied. A feedback is received only when all tasks are served, which leads to sparse reward. Besides, making the best use of limited computational resources and historical records for training a robust policy is crucial. The time allocated to different problem instances would highly impact the learning process.
To tackle those challenges, this paper proposes a novel adaptive constrained evolutionary reinforcement learning (ACERL) approach, which maintains a population of actors for diverse exploration. ACERL accesses each actor for tackling sparse rewards and constraint violation to restrict the behaviour of the policy. Moreover, ACERL adaptively selects the most beneficial training instances for improving the policy. Extensive experiments on eight training and eight unseen test instances demonstrate the outstanding performance of ACERL compared with several state-of-the-art algorithms. Policies trained by ACERL can schedule the vehicles while fully satisfying the constraints. Additional experiments on 40 unseen noised instances show the robust performance of ACERL.  Cross-validation further presents the overall effectiveness of ACREL. Besides, a rigorous ablation study highlights the coordination and benefits of each ingredient of ACERL. 
\end{abstract}


\section{Introduction}
\label{sec:intro}
In modern smart logistics such as flexible manufacturing systems and warehouse floors, the need of automated guided vehicles (AGVs) has grown fastly. Dynamic material handling (DMH)~\cite{zou2021effective,singh2022matheuristic} involves scheduling a fleet of AGVs to serve dynamically arriving transporting tasks in real time.
Transporting tasks typically include lifting and moving material between workstations using AGVs. The purpose of the DMH is to minimise makespan and tardiness, i.e., the maximal task finishing time and delay of tasks, in response to unexpected events and strict problem constraints.
Due to the unexpected occurrence of dynamic events such as vehicle breakdowns and new tasks in DMH,
the scheduling plans have to be determined frequently in real time~\cite{ouelhadj2009survey,kaplanouglu2015multi}. Problem constraints such as the delay of tasks should be also guaranteed to ensure the operation of manufacturing~\cite{hu2023dmh}.

Usually, some historical completed task records, e.g., previous task contexts and AGV information, are available in real life~\cite{hu2020deep,jeong2021reinforcement}.
Training a decision policy using multiple instances consisting of historical records enhances the generalisation, but the trade-off across different instances is introduced.
Selecting suitable instances at different training stages is crucial with consideration of limited computational budgets~\cite{dennis2020emergent}. It is also hard to determine the unique contribution of each task assignment to the performance of the entire system since the overall performance can only be determined once all tasks are completed, which leads to a sparse feedback.

Dispatching rule is a classic and common method for handling DMH~\cite{blackstone1982state,sabuncuoglu1998study}. Simple yet practicable mechanism makes them easy to deploy to real-world operations quickly. But this simplicity leads to limited performance and poor adaptability to real-world scenarios. Search-based methods, such as evolutionary algorithms (EAs) have been used for handling DMH when dynamic events occur, which restarts the search for a new solution~\cite{chryssolouris2001dynamic}. However, a new search needs to be run from scratch for every single new scenario. The long search time merely meets the requirement of a fast response.

Recently, reinforcement learning (RL) makes some promising progress in DMH~\cite{li2018simulation,hu2020deep,jeong2021reinforcement} by providing prompt on-line responses with trained polices. In RL setting, 
DMH is formulated as a Markov decision process (MDP), where the reward function is manually constructed based on the makespan~\cite{hu2020deep}.
It takes massive effort and human knowledge to construct a ``good'' reward function since RL-based methods~\cite{li2018simulation,jeong2021reinforcement} often suffer from sparse feedback~\cite{dewey2014reinforcement}.
However, those designed reward functions cannot guarantee the consistence with the original objective function and fail to generalise on multiple instances.

How to handle constraints such as the availability of AGVs and task delays in DMH is crucial. The work of \cite{hu2020deep} repeatedly sampled the task assignment until the sampled assignment is feasible. However, ensuring adherence to feasible task assignments is hard by re-sampling. It is also not plausible to satisfy the long-term task delay via simply penalising rewards~\cite{ng1999policy,achiam2017constrained}. To handle constraints, the work of \cite{hu2023dmh} formulated the DMH problem as a constrained MDP~\cite{altman1999constrained} and proposed reward constrained policy optimisation with masking (RCPOM) approach. However, RCPOM suffers from sparse feedback from both reward and constraint sides. And the robust performance across multiple DMH problem instances is not fully addressed. A policy may perform well on the given instances, while failing on others, as it might not have been sufficiently trained on certain instances with limited computational budget~\cite{hu2023dmh}.

In this paper, we consider DMH problem with uncertainties and sparse feedback in the context of multiple instances.
We proposed a robust adaptive constrained evolutionary reinforcement learning (ACERL) approach to achieve real-time decision-making in DMH with adaptability and effectiveness. Unlike regular RL methods, ACREL inherits the gradient-free characteristic of natural evolution strategies (NES)~\cite{salimans2017evolution}, which tackles sparse feedback in both rewards and constraint violation penalties intuitively. No gradient calculation related to the backpropagation is required.
ACREL maintains a population of actors for diverse exploration.
Intrinsic stochastic ranking using the rank-based fitness is proposed to evaluate the actors. Those rank-based fitness values facilitate the estimation of natural gradients, which implies the optimisation direction for maximising rewards and satisfying constraints simultaneously. The limited computational budget leads to a trade-off in the utilisation of multiple instances consisting of historical records for training a robust policy.
To tackle the trade-off, we propose a novel training mode that adaptively selects a subset of training instances with which a population of actors interacts. 
The selected instances introduce a bias in estimating the natural gradients after being ranked by the intrinsic stochastic ranking, which contributes to a remarkable performance among multiple unseen instances. 

Main contributions of this paper are summarised as follows:
\begin{itemize}
    \item We propose ACERL to address DMH with uncertainties and sparse feedback. ACERL balances the reward maximisation and constraint satisfaction, even when the agent receives only limited feedback. Given multiple instances with different contexts, ACERL still provides robust scheduling solutions in real time.

    \item We experimentally demonstrate the limitations of using a single instance or randomly selecting from multiple instances for training decision policies. Under this observation, we propose an adaptive training mode that deploys an adaptive instance sampler to break the trade-off of multiple instances. The limited computational resources are allocated by adaptively choosing the most beneficial instances for training a robust policy. 
    \item ACERL requires no domain knowledge and makes no assumption on reward or constraint-related functions. It is suitable to solve real-world problems with sparse feedback and constraints.
    \item Extensive experiments show that ACERL outperforms five state-of-the-art algorithms on eight training and eight unseen test instances. Additional experiments on 40 instances with uncertainties and leave-one-out cross-validation between heterogeneous training instances present its robust performance, in terms of maximising the reward and satisfying constraints. The functionality and coordination of each ingredient in ACERL are further demonstrated through a rigorous ablation study.
 
\end{itemize}

The remainder of this paper is organised as follows. Section~\ref{sec:back} introduces the DMH problem and related work. 
Section~\ref{sec:alg} details the proposed ACERL and its components. 
Section~\ref{sec:exp} presents the experimental studies.

Section~\ref{sec:conclusion} concludes and discusses some future directions.

\section{Background}
\label{sec:back}
Section \ref{back_dmh} describes and formulates DMH.
Then, related works on scheduling DMH and evolutionary reinforcement learning are presented in Sections \ref{back_sdmh} and \ref{back_erl}, respectively.

\subsection{Dynamic material handling}
\label{back_dmh}
DMH are widely found in manufacturing, warehouse and other systems with transporting scenarios~\cite{zou2021effective,singh2022matheuristic}. 
The problem is involved with transporting some goods from their storage sites to some delivery sites with regard to dynamic events including new tasks and vehicle breakdowns. 
A policy is responsible for assigning dynamically arriving transporting tasks to AGVs of different types in real-time. The objectives are minimising the makespan and restricting the task delay within a tolerant threshold. 
\subsubsection{Problem description}
Fig.~\ref{fig:dmh_ren} presents an example of a material handling scenario on a manufacturing floor.
The manufacturing floor is formed as a graph $G(\mathcal{L},\mathcal{T})$, where $\mathcal{L}$ and $\mathcal{T}$ denote the sets of sites and paths, respectively. Sites are stop or working points like pickup points, delivery points and parking points for AGVs. Each workstation serves as either a pickup or delivery point, where AGVs collect some material or transport carried material. The warehouse can only be the pickup point.
Task $u_1,u_2,\dots$ can be released at any moment. The total number of tasks $m$ is unknown initially. A task $u$ is determined by pickup point $u.s$, delivery point $u.e$, arrival time $u.o$ and expiry time $u.\tau$, when it is released.
A fleet of AGVs $\mathcal{V}$ is arranged to complete tasks by a policy $\pi$. 
An AGV has one of the three possible states, namely \textit{Idle}, \textit{Working} and \textit{Broken}. Only available AGVs, i.e., in \textit{Idle} state, can serve tasks, which is an \textit{instantaneous constraint}. An AGV can handle one and only one task at once while \textit{Working}. If an AGV is broken, it should release the assigned task and be repaired in place for a certain amount of time $v.rp$ before being available again.

The longest total task finishing time among all AGVs after finishing serving $m$ tasks, \textit{makespan} $F_m(\pi)$, is formulated in Eq. \eqref{eq:makespan}: 
\begin{equation}
F_{m}(\pi) = \max_{v \in \mathcal{V}^{\pi}} FT(u^v_{|HL(v)|},v),\label{eq:makespan}
\end{equation}
where $HL(v)$ refers to the historical list of all tasks served by $v$ and $FT(u,v)$ is the time of finishing serving task $u$ by $v$.

The average delay of all tasks, \textit{tardiness} $F_t(\pi)$, is defined as  formulated as:
\begin{equation}
F_{t}(\pi) = \frac{1}{m}\sum_{v\in\mathcal{V}^{\pi}}\sum_{i=1}^{|HL(v)|}\max\{FT(u^v_{i},v)-u^v_{i}.o-u^v_{i}.\tau,0\},\label{eq:tardiness}
\end{equation}
where $\mathcal{V}^{\pi}$ denotes the fleet scheduled by policy $\pi$ and $u^v_{i}$ represent the task $u_i$ served by AGV $v$.

\begin{figure}[htbp]
    \centering
    \includegraphics[width=.8\linewidth]{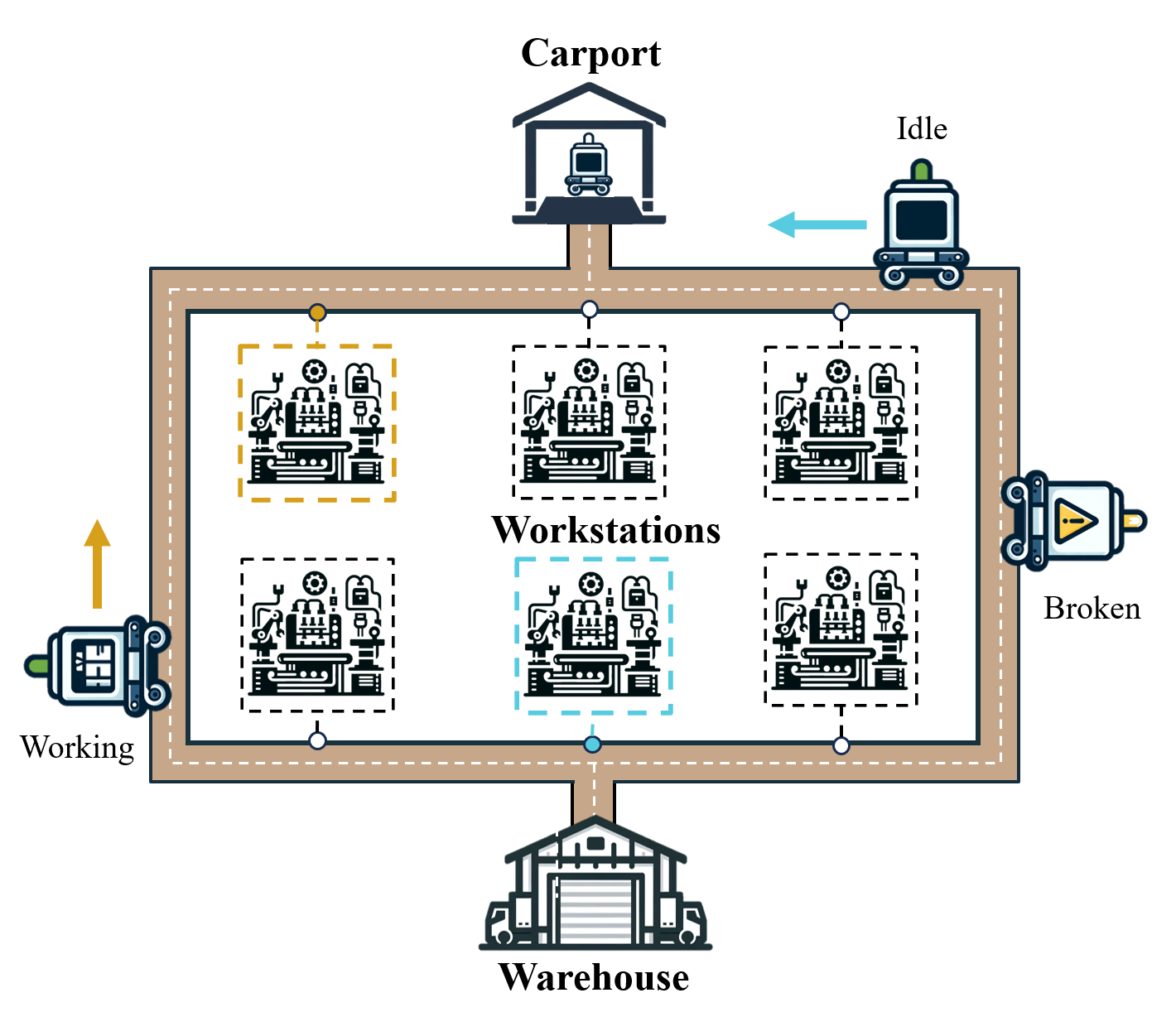}
    \caption{An illustration of DMH with AGVs that are in one of the three possible states \textit{Idle}, \textit{Working} and \textit{Broken} at different time step}.
   \label{fig:dmh_ren}
\end{figure}
\subsubsection{Constrained MDP formulation}

DMH is formulated as a constrained MDP (CMDP)~\cite{hu2023dmh,altman1999constrained}, denoted by the tuple $\langle\sS,\sA,\sR,\sC,\sP,\gamma\rangle$, where $\sS$ is the set of states, $\sA$ is the set of actions, and $\sR:\sS\times \sA \times \sS \mapsto \mathbb{R}$ is the reward function. $\sC$ is the penalty functions related to the constraint with $\sC:\sS\times \sA \times \sS \mapsto \mathbb{R}$. $\sP:\sS\times \sA \times \sS \mapsto [0,1]$ is the transition probability function. 
$\gamma \in (0,1)$ is the discount factor. 
In DMH, the state space, $\sS$, is encoded by task and AGV information, such as the remaining time before timeout and waiting time of unassigned tasks in the task pool at decision time $t$. The action space, $\sA =\mathcal{D} \times \mathcal{V}_t $, is a hybrid space combining AGVs information $\mathcal{V}_t$ at $t$ and dispatching rules $\mathcal{D}$; for instance, it involves deciding a specific dispatching rule to assign a task to a chosen AGV. To maintain consistency between the reward function and objective function, the policy receives only the negative number of the final makespan as its reward, i.e., $-F_m$, while receive zero at other time. Similarly, the penalty function, $\sC$, is constructed based on the tardiness, where the policy gets $-F_t$ once all tasks are served~\cite{hu2022constrained}.

Policy $\pi(a_t|s_t)$ is the probability of taking action $a_t$ in state $s_t$ at time $t$. 
Usually, a cumulative constraint $\sC = g(c(s_0,a_0,s_{1}),\dots,c(s_{t},a_{t},s_{t+1}))$ is restricted by a threshold $\xi$ where $c(s,a,s'): \sS \times \sA \times \sS \mapsto \mathbb{R}$ is a per-step penalty. 
Let $J^{\pi}_{\sC}$ denote the expectation of the cumulative constraint, formulated as $J^{\pi}_{\sC} = \mathbb{E}_{\tau \sim \pi}[\sC]$, where $\tau\sim\pi$ denotes a trajectory $(s_0,a_0,s_1,a_1,s_2,\dots)$ sampled from $\pi$. 
In DMH, tardiness, formulated in Eq. \eqref{eq:tardiness}, is considered as the \emph{cumulative constraint}.
The instantaneous constraint considers if an action $a_t$ is legal at state $s_t$, i.e., only choose available vehicles.
The policy $\pi_\theta$, parameterised by $\theta$, aims at maximising the discounted cumulative reward while satisfying the constraints, formulated as~\cite{altman1999constrained}:
\begin{eqnarray}
\max_{\theta}~J_{\sR}^{\pi_\theta}&=&\mathbb{E}_{\tau \sim\pi_\theta}[\sum_{t=0}^{\infty} \gamma^t \sR(s_t,a_t,s_{t+1})], \label{eq:jr}\\
s.t.&& J^{\pi}_\sC\leq \xi, \\
&& a_t \text{ is legal}, \forall t=0,1,2,\dots,
\end{eqnarray}
where $\xi$ represents the constraint threshold.

\subsection{Scheduling dynamic material handling}
\label{back_sdmh}

Dispatching rules is a classic approach for DMH. Rules such as first come first serve (FCFS), earliest due date first (EDD), and nearest vehicle first (NVF), usually have simple hand-crafted mechanisms~\cite{sabuncuoglu1998study}. Due to their simplicity, they can be quickly implemented and deployed in simple manufacturing systems.
However, the dispatching rules can hardly be improved and adapted to complex situations. Motivated by the poor generalisation of using one single rule at once~\cite{blackstone1982state}, some work combined multiple dispatching rules to make decisions~\cite{tay2008evolving,chen2011multiple}. However, how to effectively coordinate multiple dispatching rules presents a new challenge.

Iterative methods can also be applied to solve DMH. According to the dynamic events, Iterative methods divide the scheduling problem into several subproblems. The process restarts to optimise each subproblem when a dynamic event occurs.
Liu et al.~\cite{liu2018dynamic} proposed a dynamic framework to schedule a fleet of robots for material transportation. When a new task arrives, the framework is triggered to search for a new scheduling solution with a mixed integer programming (MIP) model. Yan et al.~\cite{yan2018dynamic} integrated the MIP model to an iterated algorithm for job shop in material handling with multiple new tasks. 

Instead of using a MIP model, population-based search methods have been applied to solve DMH by searching subproblm.  
 Chryssolouris and Subramaniam~\cite{chryssolouris2001dynamic} assumed that the number of tasks is known and the operations of each job may vary. 
 Task assignments are encoded as permutation-based solutions, then a genetic algorithm searches for promising operation sequences whenever a dynamic event happens~\cite{chryssolouris2001dynamic}. 
Umar et al. ~\cite{umar2015hybrid} proposed a hybrid genetic algorithm with a weighted sum fitness function of multiple objectives using random key representation for DMH.
Wang et al.~\cite{wang2020proactive} optimised travel distance and energy consumption at the same time with the non-dominated sorting genetic algorithm. If a trolley breaks or a task is cancelled, the optimisation restarts~\cite{wang2020proactive}.
An adaptive parameter adjustment with discrete invasive weed
optimisation algorithm was investigated by Li et al.~\cite{li2022dynamic} to handle the dynamic scheduling, in which multiple dynamic events such as new task, emergent tasks and task cancellation are considered.
Although the aforementioned search-based methods may obtain promising solutions for the subproblems triggered by every occurrence of dynamic events, they are limited by the long optimisation time~\cite{chryssolouris2001dynamic}. Thus, they hardly perform a fast and adaptive response for some real-world scenarios. Moreover, search-based methods typically require a specific problem representation, which is challenging to directly transfer for optimising other problems, particularly those involving additional dynamic events and constraints.

It has been witnessed that RL presents a competitive level beyond humans on some sequential decision-making problems including video games~\cite{mnih2015human, hu2023dorl}, Go~\cite{silver2016mastering}, and robotic control~\cite{haarnoja2018soft}. RL-based methods have been applied to solve DMH problems for their advantages over sequential decision-making problems.
Chen et al.~\cite{chen2015reinforcement} proposed a Q$_\lambda$ algorithm with forecasted information. The RL-based dispatching policy decides the tasks of the dolly train, considering multiple loads.
Xue et al.~\cite{xue2018reinforcement} considered a flow shop scenario with multiple AGVs and applied Q-learning to find a suitable match between jobs and vehicles.
Kardos et al.~\cite{kardos2021dynamic} optimised the choice of workstations for products, using an RL-based method.
Instead of directly choosing tasks, Hu et al.~\cite{hu2020deep} adapted deep Q-learning (DQN) to choose a pair of dispatching rule and AGV, inspired by the work of Chen et al.~\cite{chen2011multiple}. The chosen rule then assigns a waiting task to the paired AGV, aiming at minimising makespan and delay ratio. 
Li et al.~\cite{li2025real} modelled the scheduling as a multi-agent scenarios, where dispatching tasks and selecting vehicles are controlled by two agents, respectively.

However, RL-based methods have to put extra effort to design dense reward functions for specific problems due to the inherent sparsity of real-world problems. Although there are some techniques like intrinsic rewards~\cite{bellemare2016unifying,houthooft2016vime} to address the issue, they introduce a bias to the true reward function. 
Besides, RL-based methods are not originally designed to handle constraints such as vehicle breakdowns, which makes them hard to be applied to real-world problems directly. To address those issues, Hu et al.~\cite{hu2023dmh} formulated the problem as a CMDP~\cite{altman1999constrained} and proposed a constrained RL method, named RCPOM, by incorporating reward shaping and invalid action masking into reward constraint policy optimisation (RCPO)~\cite{tessler2018reward}. 
Although RCPOM has shown superior performance on diverse DMH instances, it doesn't present enough constraint satisfaction~\cite{hu2023dmh}, due to the temporal credit assignment problem with both sparse rewards and constraint violation. Moreover, as the RL agent is exposed to multiple instances during training, a critical question arises regarding the optimal allocation of training resources across these instances. The proportion of time or episodes dedicated to each instance can significantly impact the learning process. As highlighted by~\cite{jiang2021prioritized}, different sampling proportions lead to variations in the distribution of experiences, which in turn introduces biases in gradient calculations. This bias can steer the policy optimisation in suboptimal directions, potentially compromising the agent's generalisation capabilities.

\subsection{Evolutionary reinforcement learning}
\label{back_erl}
Considering neural network optimisation (weights or architecture) as a black-box problem, evolutionary reinforcement learning (ERL) directly applies EAs~\cite{yao1999evolving} or integrates EA to RL to search for parameters of an actor~\cite{sigaud2022combining,bai2023evolutionary}. ERL typically uses a fitness-based metric for parent selection and survivor selection. It particularly works in the non-differentiable case since no gradient calculation related to the backpropagation is necessarily required.

Evolution strategy (ES), a typical EA for numerical optimisation, is often used to train neural network policies for RL tasks. Alimans et al.~\cite{salimans2017evolution} leveraged natural evolution strategies for policy optimisation. Noises are sampled from a factored Gaussian distribution and then added to the policy network to generate a population. The closed form of the gradient estimator is also given by Alimans et al.~\cite{salimans2017evolution}.
Such et al.~\cite{such2017deep} applied a genetic algorithm to evolve networks, in which the parameters of a neural network are treated as an individual. Conti et al.~\cite{conti2018improving} validated the effectiveness of novelty search and quality diversity assisted with ES, when meeting sparse and deceptive reward functions.
Yang et al.~\cite{yang2022evolutionary} proposed a cooperative co-evolution algorithm based on negatively correlated search to optimise parameters of policy network. 
Khadka and Tumer~\cite{khadka2018evolution} proposed a hybrid framework that combines EA and MDP-based RL, addressing the sparse reward and exploration issues. Hu et al.~\cite{hu2023ecrl} proposed an evolutionary constrained reinforcement learning algorithm for robotic control. Constraint handling technique is incorporated to the EA to handle constrained RL problems. However, on the other hand, the inherent population and elite survival prevent it from training on multiple scenarios at the same time.

\section{Adaptive constrained evolutionary reinforcement learning}
\label{sec:alg}
\begin{figure*}[!h]  
  \centering
  \includegraphics[width=0.9\linewidth]{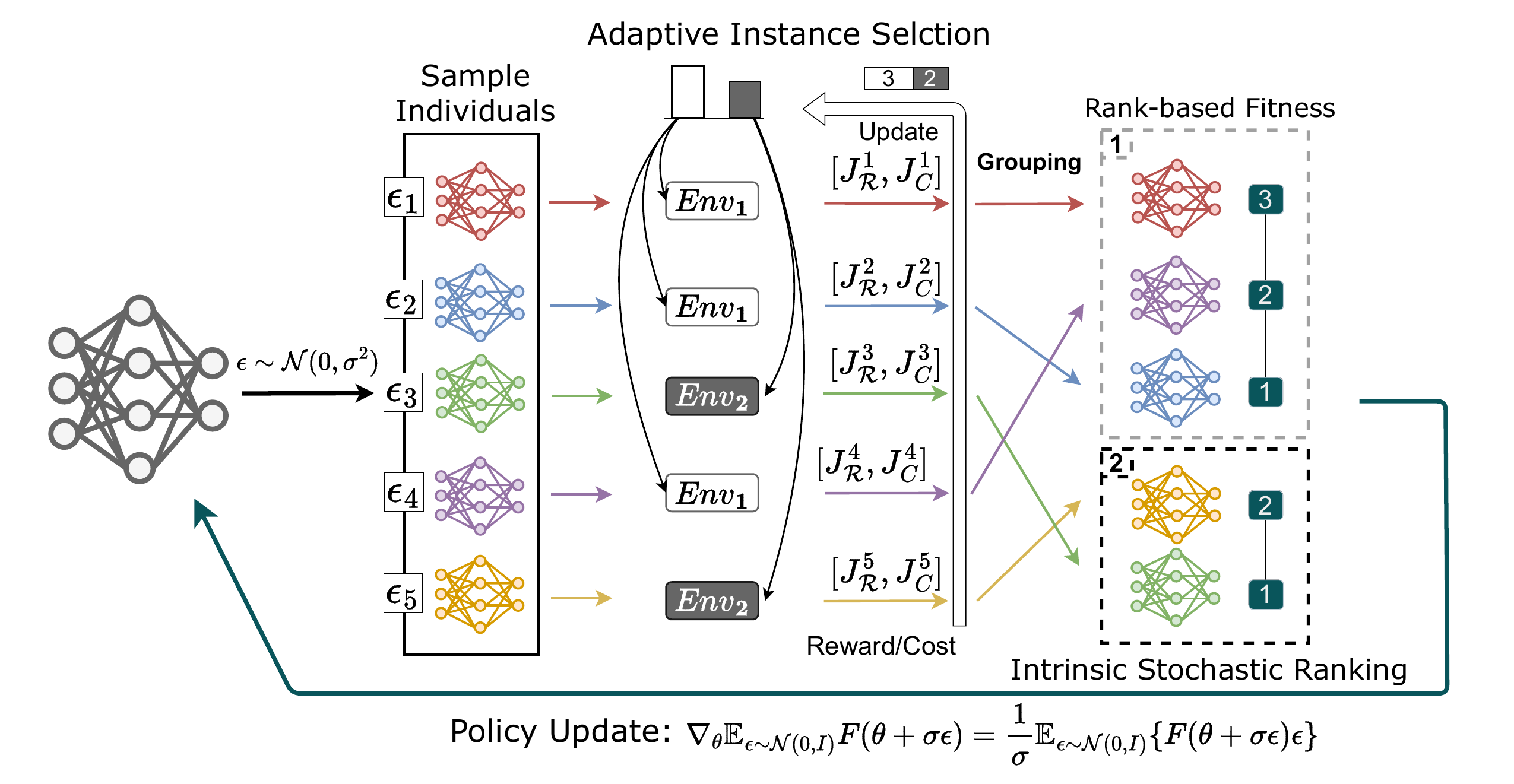}
  \caption{Illustration of ACERL framework. A population of actors is sampled with Gaussian noise. All actors then interact with the specific subset of training instances determined by the adaptive instance sampler (cf. Section \ref{sec:ais}). The collected rewards and selection number of training instances are used to update the instance sampler. The fitness of each actor is assigned according to its inner rank with the intrinsic stochastic ranking (cf. Section \ref{sec:isr}). A natural gradient ascent is applied to update the policy accordingly (cf. Section \ref{sec:pi}).}
  \label{fig:diagram_ACERL}
\end{figure*}

We propose adaptive constrained evolutionary reinforcement learning algorithm (ACERL)~\footnote{\label{ft:code}Code: \url{https://github.com/HcPlu/ACERL}} to optimise the parameters $\theta$ of the policy $\pi_{\theta}$ for scheduling. 
The framework and pseudo-code are presented in Fig.~\ref{fig:diagram_ACERL} and Algorithm~\ref{alg:acel}, respectively.

\begin{algorithm}[htbp]
\caption{ACERL.}
\label{alg:acel}
\begin{algorithmic}[1] 
\Require Generation number $G$, population size $\lambda$, number of Instances $K$, learning rate $\alpha$, noise standard deviation $\sigma$
\Ensure $\pi_\theta$
\State Initialise policy $\pi_{\theta}$
\State Initialise reward buffers $\mathcal{B}_\sR=\langle\mathcal{B}_i\rangle$, $i =\{1\,\dots, K\}$
\State Initialise number of selections $N = \langle N_i\rangle$, $i =\{1\,\dots, K\}$
\For{$n=1$ to $G$}
\State Initialise instance buffers $\mathcal{I}=\langle\mathcal{I}_i\rangle$, $i =\{1\,\dots, K\}$
    \State Sample noise $\epsilon_1 \cdots \epsilon_\lambda \in \mathcal{N}^{|\theta|} (0,I)$
    \For{$i=1$ to $\lambda$ }

        \State $\pi_{\theta_i} \leftarrow \pi_{\theta+\sigma\epsilon_i}$
        \State Sample instance $\eta \leftarrow \text{AIS}(\mathcal{B}_R,N)$  \Comment{\emph{Algorithm \ref{alg:ais}}}
        \State $J^{\pi_{\theta_i}}_\sR,J^{\pi_{\theta_i}}_\sC= \text{Evaluate}(\pi_{\theta_i}, \eta)$ \Comment{\emph{Algorithm \ref{alg:eva}}}
        \State $N_{\eta}\leftarrow N_{\eta}+1$
        \State $\zeta \leftarrow i$
        \State Store $J^{\pi_{\theta_i}}_\sR$ in $\mathcal{B}_\sR^\eta$
        \State Store $J^{\pi_{\theta_i}}_\sR,J^{\pi_{\theta_i}}_\sC,\zeta$ in $\mathcal{I}_j$
    \EndFor

    \State $f_1,\dots,f_{\lambda} = \text{ISR}(\mathcal{I})$  \Comment{\emph{Algorithm \ref{alg:isr}}}
    
    \State $\theta \leftarrow \theta+\alpha\frac{1}{\lambda\sigma}\sum^\lambda_i f_i\epsilon_i$
\EndFor

\end{algorithmic}
\end{algorithm}

ACERL models an actor (a neural network in our case) as an individual and maintains a population of those independent actors. Instead of evolving the population with genetic operators, at each generation, ACERL samples a population from a distribution based on the policy $\pi_{\theta}$.
The individuals and their corresponding fitnesses are used to update the distribution and discarded instantly later.

Note that all individuals of the population interact with the environment formed with a specific training instance. The corresponding training instances are chosen by an \textit{adaptive instance sampler} (Algorithm \ref{alg:ais}), resulted in an adaptive training process. It estimates the advantage of each \textit{candidate instance} according to its historical rewards based on the evaluation of the policy (Algorithm \ref{alg:eva}), and the number of selections. The advantage of each instance describes how good the instance is for policy improvement. The most beneficial instance for training the policy is selected according to the estimated soft probability formed with the advantages. If the policy shows inferior performance on certain instances, the likelihood of choosing those instances increases accordingly.

After obtaining the episodic rewards and penalties, we design an \textit{intrinsic stochastic ranking} method (Algorithm~\ref{alg:isr}) to group the sampled individuals that interact with the same training instance into the same buffer.
The fitness of each individual is assigned with its own rank index in the descending intrinsic ranked buffer. The ranking method balances the rewards and penalties, seeking to maximise the long-term reward with constraint satisfaction. The impact of fitness weight is further involved in breaking trade-off among multiple training instances. Finally, ACERL updates the policy with \textit{natural evolution strategies} according to sampled noises and the corresponding fitness values. 

The following subsections detail the core
ingredients of ACERL, including adaptive training, intrinsic stochastic ranking with rank-based fitness and natural evolution strategies.
.

\subsection{Efficient instance selection via adaptive training}\label{sec:ais}

To achieve a better computational resource allocation, we design an adaptive instance sampler (AIS) to train ACERL by selecting suitable training instances adaptively. The pseudo-code of AIS is shown in Algorithm~\ref{alg:ais}.
The idea behind AIS is that if the policies have performed well on a specific instance, then the probability of selecting this instance should be reduced and the likelihood of selecting the instances that were evaluated fewer times should increase.

Specifically, an inverted distance metric for evaluating the advantage of training instances is proposed:
\begin{equation}\label{eq:metric}
  u_\eta = \frac{1}{|\mathcal{B}_\sR^\eta|}\sum_{i=1}^{|\mathcal{B}_\sR^\eta|}
  \frac{\max \limits_{1\leq k \leq |\mathcal{B}_\sR^\eta|}J_{\sR}^k-J_{\sR}^i}{\max\limits_{1\leq k \leq |\mathcal{B}_\sR^\eta|}{J_{\sR}^k}-\min\limits_{1\leq k \leq |\mathcal{B}_\sR^\eta|}{J_{\sR}^k}}.
\end{equation}

The distance metric measures the performance of the policy on the training instance $\eta$ at the time horizon. Specifically, it calculates how much each episode’s reward deviates from the maximum reward, scaled by the batch’s reward range. A larger $u_\eta$ signifies poorer policy performance on $\eta$, which provides an insight that the instance should be chosen more.

The adaptive training takes the metric in Eq. \eqref{eq:metric} and the number of times of selecting each training instance into account by the widely used upper confidential bound (UCB)~\cite{agrawal1995sample,Hao2019Upper}:
\begin{equation}
UCB = u_i+\alpha_u\sqrt{\frac{\log(\sum_j^K N_j)}{N_i}},
\end{equation}
where $\alpha_u$ is the exploration factor, $K$ is the total number of training instances and $N_i$ is the number of instance $i$ being selected as the training candidate. 
With the UCB-based adaptive instance sampler, instances are chosen to train the policy more efficiently.
Thus the computational resource can be dynamically allocated in a proper way.

\begin{algorithm}[b]
\caption{Adaptive instance sampler, $\text{AIS}(\mathcal{B}_R,N)$.
}
\label{alg:ais}
\begin{algorithmic}[1] 
\Require Reward buffers $\mathcal{B}_\sR$, counts $ N$, exploration factor $\alpha_u$  
\Ensure $\eta$
\For{$\eta=1$ to $K$}
    \State $u_\eta =  \frac{1}{|\mathcal{B}_\sR^\eta|}\sum\limits_i^{|\mathcal{B}_\sR^\eta|}
    \frac{\max \limits_{1\leq k \leq |\mathcal{B}_\sR^\eta|}(J_{\sR}^k)-J_{\sR}^i}{\max\limits_{1\leq k \leq |\mathcal{B}_\sR^\eta|}{(J_{\sR}^k)}-\min\limits_{1\leq k \leq |\mathcal{B}_\sR^\eta|}{(J_{\sR}^k)}}.$

\EndFor
\State Sample $\eta = \softmax\limits_\eta(u_\eta+\alpha_u\sqrt{\frac{\log(\sum_\eta^M N_\eta)}{N_\eta}})$

\end{algorithmic}
\end{algorithm}

\subsection{Intrinsic stochastic ranking with rank-based fitness}\label{sec:isr}

Assessing a policy for DMH that involves multiple instances and constraints can be challenging beyond the unconstrained optimisation problems~\cite{salimans2017evolution,conti2018improving}. Regular methods use a weighted sum of the objective value and penalties for constraint violation as the reshaped reward function, which, however, introduces the challenge of adjusting the weights~\cite{watanabe2004evolutionary}.

Stochastic ranking (SR)~\cite{runarsson2000stochastic} is an effective constraint handling technique, which has been successfully applied to constrained optimisation~\cite{runarsson2003constrained} and combinatorial optimisation~\cite{tang2009memetic}. SR makes no assumption on the problem, while only one parameter is introduced and is easy to tune~\cite{runarsson2000stochastic}.
Inspired by the work of~\cite{runarsson2000stochastic}, we propose intrinsic stochastic ranking (ISR) with rank-based fitness, shown in Algorithm \ref{alg:isr}, to handle the constraints in the case of sparse feedback.
During the evolution process, the individuals that interact with the same training instance are collected and stored in the same buffer.
Then, all the individuals in the same buffer are ranked according to their rewards and penalties. 
Notably, if solutions provided by two individuals are either feasible or meet the probability threshold $p_f$, the solution with the higher reward is assigned a better rank, otherwise, the one with fewer constraint violations is given the higher rank.

\begin{algorithm}[!t]
\caption{$\text{Evaluate}(\pi)$ evaluates a policy $\pi_{\theta}$ by interacting with a given environment.}
\label{alg:eva}
\begin{algorithmic}[1] 
\Require Policy $\pi_{\theta}$, Instance $\eta$
\Ensure $J^{\pi_{\theta}}_\sR,J^{\pi_{\theta}}_\sC$
\State $J^{\pi_{\theta}}_\sR \leftarrow 0$
\State $J^{\pi_{\theta}}_\sC \leftarrow 0$
\State Initialise environment $env$ with the instance $\eta$

\For{$t=0,1,2, \dots, T-1$}
    \State Sample action $a_t \in \pi_{\theta}(s_t)$
    \State Obtain $r_t,c_t,s_{t+1}$ from $env$ by acting $a_t$
    \State $J^{\pi_{\theta}}_\sR \leftarrow J^{\pi_{\theta}}_\sR  +r_t$
    \State $J^{\pi_{\theta}}_\sC \leftarrow J^{\pi_{\theta}}_\sC + c_t$

\EndFor
\end{algorithmic}
\end{algorithm}

An individual's fitness value is assigned based on the ranking in its own buffer by ISR.
Following the setting of~\cite{runarsson2000stochastic}, the penalty function related to the constraint violation is:
\begin{equation}
\phi(\pi)=\left(\max\{0,J^{\pi}_\sC-\xi\}\right)^2.
  \label{eq:isr_penalty}
\end{equation}

ISR remains some infeasible solutions with high fitness values besides giving priority to feasible solutions.
Although infeasible solutions cannot be executed to solve the problem, they may help escaping from some infeasible areas. ISR not only balances the rewards and penalties, but also allows the dynamic selection of suitable training instances with preference, which makes efficient use of the computational resource. Additionally, ISR only requires the final scalar value instead of temporal information, which aligns well with the sparse case.

\begin{algorithm}[!t]
\caption{\label{alg:isr} Intrinsic stochastic ranking, $\text{ISR}(\mathcal{I})$. $p_f\in (0,1)$ is the tolerate probability. $\phi(\pi)$ denotes the penalty.}
\begin{algorithmic}[1] 
\Require Instance buffers $\mathcal{I}_1,\cdots,\mathcal{I}_K$
\Ensure $f_1,\cdots,f_{\mu}$ 

\For{$k=1$ to $K$ }
    \State $\mu'=|\mathcal{I}_k|$
    \For{$i=1$ to $\mu'$}
        \State $l_i = i$
    \EndFor
    \State Get $\langle J^{\pi_{\theta_j}}_\sR,J^{\pi_{\theta_j}}_\sC,\zeta_j^k\rangle$ from $\mathcal{I}_k, j=1,\dots, \mu'$
\For{$i=1$ to $\mu'$}
    \For{$j=1$ to $\mu'-1$}
        \State Sample $\delta$ uniformly at random in $(0,1)$ \label{line:random}
        \If{$(\phi({\pi_{\theta_j}})=\phi({\pi_{\theta_{j+1}}})=0)$ or $(\delta < P_f)$} 
           \If{$J_\sR^{\pi_{\theta_j}}<J^{\theta_{j+1}}_\sR$}
            \State swap $l_j$ and $l_{j+1}$
        \EndIf \label{line:endrandom}
        \Else
                \If{$\phi(\pi_{\theta_j})>\phi(\pi_{\theta_{j+1}})$}
            \State swap $l_j$ and $l_{j+1}$
                    \EndIf
        \EndIf
    \EndFor
\EndFor
\For{$i=1$ to $\mu'$ }
        \State $f_{\zeta^k_{l_i}}=\mu'-i+1$
    \EndFor
\EndFor
\end{algorithmic}
\end{algorithm}

\subsection{Policy improvement through searching gradients}\label{sec:pi}
ACERL applies evolution strategies to update the policy $\pi_{\theta}$, which is parameterised by a neural network $\theta$. A population is sampled from a distribution over $p_{\psi}(\theta)$ with the given policy $\pi_{\theta}$. The estimated gradient is given by Wierstra et al.~\cite{wierstra2014natural}:
\begin{equation}
  \nabla_{\psi} \mathbb{E}_{\theta\sim p_{\psi}}F(\theta) = \mathbb{E}_{\theta\sim p_{\psi}}F(\theta)\nabla_{\psi}\log(p_{\psi}(\theta)),
\end{equation}
where $F(\theta)$ is the fitness function related to the policy $\pi_{\theta}$. 
The minimal requirement of temporary information makes natural gradient intuitive to tackle sparse reward and long-term horizon, compared with value-based methods and policy gradient, without calculating the gradients for the backpropagation. The perturbations in the parameter space also facilitate the exploration for collecting more diverse experiences~\cite{salimans2017evolution}.

In the DMH, the fitness of each individual is given by the intrinsic stochastic ranking (Algorithm~\ref{alg:isr}).
More specifically, $p_{\psi}$ is a factored Gaussian distribution $\mathcal{N}(\psi,\sigma^2)$, in which $\psi$ is the mean value and $\sigma$ is the covariance. The expectation of fitness under the distribution considering $\theta$ as the mean parameter~\cite{salimans2017evolution} is written as Eq. \eqref{eq:acerl_r} :
\begin{equation}
  \mathbb{E}_{\theta\sim p_{\psi}}F(\theta) = \mathbb{E}_{\epsilon\sim \mathcal{N}(0, I)}F(\theta+\sigma\epsilon).\label{eq:acerl_r}
\end{equation}
Finally, the policy is optimised by the gradient ascent with a vanilla estimator:
\begin{equation}\nabla^{\epsilon}_{\theta}F(\theta+\sigma\epsilon) = \frac{1}{\sigma}\mathbb{E}_{\epsilon\sim \mathcal{N}(0, I)}\left[F(\theta+\sigma\epsilon)\epsilon\right],
\end{equation}
or an antithetic estimator:
\begin{equation}\nabla^{\epsilon}_{\theta}F(\theta+\sigma\epsilon) = \frac{1}{2\sigma}\mathbb{E}_{\epsilon\sim \mathcal{N}(0, I)}\left[(F(\theta+\sigma\epsilon)-F(\theta-\sigma\epsilon))\epsilon\right]. \label{eq:ane}
\end{equation}

All weights are perturbed in the case of factored Gaussian distribution, which is similar to the coupon collector problem in a continuous way.
Considering ACERL as a randomised finite difference method, the optimisation complexity scales linearly with the number of weights, i.e., $\mathcal{O}(|\theta|)$.
The expected covering time to ensure all weights dimensions are sufficiently perturbed is at least polynomial bound $\mathcal{O}(|\theta|\log{|\theta|})$.
Given the population size $\lambda$, the space complexity and the communication complexity are both $\mathcal{O}(\lambda|\theta|)$ for passing gradients since we only optimise the neural network. Notably, ACERL follows the optimisation scheme of classic evolution strategies. As such, its optimisation complexity, covering time, space and communication complexities align with existing theoretical bounds~\cite{wierstra2014natural,nesterov2017random}.

\section{Theoretical Analysis}

Although Choromansk et al.~\cite{choromanski2019complexity} and Liu et al.~\cite{liu2020self} show that the gradient estimator is close to the true gradient, they only derive the theoretical results in the unconstrained setting. We extend the analysis in \cite{choromanski2019complexity,liu2020self} and demonstrate the theoretical guarantee on the constrained optimisation with stochastic ranking.

First, we simplify the constrained problem formulation on the function landscape with one single instance for a better analysis as follows:
\begin{eqnarray}
    \max_{\theta}~F(\theta) \quad s.t.&& g(\theta) \leq \xi, 
\end{eqnarray}
where $F(\theta)$ is the objective function, $g(\theta)$ is the constraint function and $\xi$ is the constraint threshold.

Following the settings of Choromansk et al.~\cite{choromanski2019complexity}, the following assumptions of the regularities of $F(\theta)$ and $g(\theta)$ are made.

\begin{assumption}
\label{assum1}
$F$ and $g$ are $L$-Lipschitz, i.e., $\forall \theta, \theta'\in \mathbb{R}^d, |F(\theta)-F(\theta')| \leq L_f||\theta-\theta'||$, $\quad |g(\theta)-g(\theta')| \leq L_g||\theta-\theta'||$.
\end{assumption}
\begin{assumption}
\label{assum2}
    $F$ has a $\tau$-smooth third order derivative tensor with respect to $\sigma>0$, so that 
$F(\theta+\sigma \epsilon) = F(\theta)+\sigma\nabla F(\theta)^{\top}\epsilon+\frac{\sigma^2}{2}\epsilon^{\top}H(\theta)\epsilon+\frac{1}{6}\sigma^3f'''(\theta)\left[v,v,v\right]$
with $ v \in \left[0,\epsilon\right]$ satisfying $|F'''(v,v,v)|\leq\tau||v||^3\leq\tau||\epsilon||^3$, where $H(\theta)$ and $F'''(\theta)$ denote Hessian matrix and third derivative of $F$, respectively. Similarly, $g$ has a $\tau$-smooth third order derivative tensor with same regularities as $F$.
\end{assumption}
\begin{assumption}
\label{assum3}
$\sigma$ is small enough, i.e., $0<\sigma < \frac{1}{35}\sqrt{\frac{\mathcal{E}}{\tau d^3\max\{L_f,L_g,1\}}}$, where $\mathcal{E}>0$.
\end{assumption}

For a better analysis, we relax the penalty function $\phi(\theta)$ formulated in Eq.~\eqref{eq:isr_penalty} to a smooth approximation since $\phi(\theta)$ is not differentiable at $g(\theta)=\xi$:
\begin{equation}
\phi'(\theta)=\rho \ln{(1+e^{(g(\theta)-\xi)/\rho})},
\end{equation}
 where $\rho>0$. This relaxation nearly preserves the ordering of the inequality as $\rho \rightarrow 0$.
\begin{theorem}

There exists a sufficiently small $\rho>0$, such that 
\begin{eqnarray}
   \noindent &&\forall \theta, \theta'\in \mathbb{R}^d, \phi(\theta_i)\leq \phi(\theta_j) \iff \phi'(\theta_i)\leq \phi'(\theta_j),\nonumber \\
&&\quad~i.e.,~\rho\ln{(1+e^{(g(\theta_i)-\xi)/\rho})} \leq \rho \ln{(1+e^{(g(\theta_j)-\xi)/\rho})}. \nonumber
\end{eqnarray}
\end{theorem}

Intuitively, the relaxed penalty function $\phi'(\theta)$, is also L-Lipschitz. Besides, Assumption~\ref{assum2} applies to $\phi'(\theta)$.

Recalling stochastic ranking, it assigns priority to individuals based on the objective and penalty function with probability $p_f$.
Thus, we consider the objective function relaxed by stochastic ranking as follows.
\begin{eqnarray}
    f_{SR}(\theta) = p_f F(\theta)- (1-p_f)\phi'(\theta), 
\end{eqnarray}
where $p_f$ is the probability with values in $[0,1]$. 
It is easy to see that Assumptions 1 and 2 also hold for $f_{SR}$.

Assuming two candidates $\pi_{\theta_i}$ and $\pi_{\theta_j}$ with $F(\theta_i)\leq F(\theta_j)$, all three possible scenarios are described as follows:
\begin{case} If $\phi'(\theta_i)=\phi'(\theta_j)=0$,
then $f_{SR}(\theta_i)\leq f_{SR}(\theta_j)$.
\end{case}

\begin{case} If $0\leq\phi'(\theta_j)\leq\phi'(\theta_i)$ and $0<\phi'(\theta_i)$, then $f_{SR}(\theta_i)\leq f_{SR}(\theta_j)$.
\end{case}

\begin{case} If $0\leq\phi'(\theta_i)\leq\phi'(\theta_j)$ and $0<\phi'(\theta_j)$, then
the ranking $f_{SR}(\theta_i)\leq f_{SR}(\theta_j)$ holds under the following inequality:
\begin{equation}
    p_f(F(\theta_i)-F(\theta_j))\leq(1-p_f)(\phi'(\theta_i)-\phi'(\theta_j)),\nonumber
\end{equation}
which indicates that the final ranking relies on $p_f$, as well as both reward and penalty values. 
\end{case}

With the above discussions, we show that $f_{SR}$  captures the behaviours of stochastic ranking.

Under the assumptions, we can derive an antithetic form that
\begin{eqnarray}
    \frac{f_{SR}(\theta+\sigma\epsilon)-f_{SR}(\theta-\sigma\epsilon)}{2\sigma} = \epsilon^{\top}\nabla f_{SR}(\theta)+\zeta(\theta), \nonumber
\end{eqnarray}
where $\zeta(\theta)\leq\frac{\tau}{6}\sigma^2||\epsilon||^3$. Given $f_{SR}$ is smooth with constant $\tau$ on the third order derivative tensor, we have
\begin{eqnarray}
    \left|\frac{f_{SR}(\theta+\sigma\epsilon)-f_{SR}(\theta-\sigma\epsilon)}{2\sigma}-\epsilon^{\top}\nabla f_{SR}(\theta)\right| \leq\tau \sigma^2||\epsilon||^3. \nonumber
\end{eqnarray}
Recall the gradient in the antithetic case (Eq.~\eqref{eq:ane}):
\begin{eqnarray}
    \nabla^{\epsilon}_{\theta}F(\theta+\sigma\epsilon) = \frac{1}{2\sigma}\mathbb{E}_{\epsilon\sim \mathcal{N}(0, I)}[(F(\theta+\sigma\epsilon)-F(\theta-\sigma\epsilon))\epsilon]. \nonumber
\end{eqnarray}
We have $\left|\frac{f_{SR}(\theta+\sigma\epsilon)-f_{SR}(\theta-\sigma\epsilon)}{2\sigma}-\epsilon^{\top}\nabla f_{SR}(\theta)\right| \leq\tau \sigma^2||\epsilon||^3$.
Then, we can derive the following inequality:
\begin{eqnarray}\label{eq:tausigma}
            || \nabla_{\theta}\mathbb{E}_{\epsilon\sim \mathcal{N}(0, I)}f_{SR}(\theta+\sigma\epsilon)-\nabla f_{SR}(\theta)|| \leq \mathbb{E}_{\epsilon}\tau \sigma^2||\epsilon||^4.
\end{eqnarray}

\begin{theorem}
    \label{theo3}
The bias of the gradient estimator under stochastic ranking is well bounded:
\begin{eqnarray}
        ||\nabla_{\theta}\mathbb{E}_{\epsilon\sim \mathcal{N}(0, I)}f_{SR}(\theta+\sigma\epsilon)-\nabla f_{SR}(\theta)|| \leq \mathcal{E}.
\end{eqnarray}
\end{theorem}

\begin{proof}
    Considering $\sigma \sim \mathcal{N}(0,I)$, we have $\mathbb{E}[\sigma(i)^4]=3$, $\mathbb{E}[\sigma(i)^2]=1$, $\forall i\in\{1,\dots,d\}$. The following inequality holds:
    \begin{eqnarray}
        \mathbb{E}_{\epsilon\sim \mathcal{N}(0, I)}[||\epsilon||^4]=&\mathbb{E}_{\epsilon}[\sum^d_{i=1}\epsilon(i)^4+\sum_{i\neq j}\epsilon(i)^2\epsilon(j)^2] \nonumber\\ 
        \leq& 3d+d^2 \leq 3d^2.\nonumber
    \end{eqnarray}
Thus, together with \eqref{eq:tausigma}
\begin{eqnarray}
            ||\nabla_{\theta}\mathbb{E}_{\epsilon\sim \mathcal{N}(0, I)}f_{SR}(\theta+\sigma\epsilon)-\nabla f_{SR}(\theta)|| &\leq& \mathbb{E}_{\epsilon}\tau \sigma^2||\epsilon||^4 \nonumber\\ 
            &\leq&3\tau\sigma^2d^2. \nonumber
\end{eqnarray}
Recall Assumption~\ref{assum3} that $ 0<\sigma < \frac{1}{35}\sqrt{\frac{\mathcal{E}}{\tau d^3\max\{L_f,L_g,1\}}}$.
Finally, we derive the result:
\begin{eqnarray}
    ||\nabla_{\theta}\mathbb{E}_{\epsilon\sim \mathcal{N}(0, I)}f_{SR}(\theta+\sigma\epsilon)-\nabla f_{SR}(\theta)|| &\leq \mathcal{E}. \nonumber
\end{eqnarray}

\end{proof}

\section{Experimental Results and Analysis}
\label{sec:exp}
We conduct several sets of experiments and an ablation study to comprehensively evaluate ACERL. The aims of experiments and compared methods are detailed as follows.
\begin{itemize}
 \item To validate the effectiveness of ACERL, we compare it on eight training instances and eight unseen test instances with several state-of-the-art methods and baselines categorised into groups including (i) ``MAPPO"~\cite{li2025real}, ``RCPOM''\cite{hu2023dmh} and soft actor-critic (SAC)~\cite{haarnoja2018soft} with the fixed Lagrangian multiplier, ``LSAC''\cite{tamar2012policy}; (ii)  ``SAC'' and ``PPO''\cite{schulman2017proximal}; (iii) MAPPO, RCPOM, LSAC, SAC and PPO equipped with the adaptive instance sampler denoted as ``AMAPPO",``ARCPOM'', ``ALSAC'', ``ASAC'' and ``APPO'', respectively; (iv) classic dispatching rules including ``FCFS'', ``EDD'', ``NVF'' and ``STD''\cite{sabuncuoglu1998study}; and (v) two random policies that randomly choose rules or tasks, denoted as ``MIX'' and ``Random'' on both training instances and test instances, respectively.
  \item To present the robust performance of ACERL, it is tested on 40 instances with increasing extent of perturbations which simulate different degrees of dynamic events.
  \item To further evaluate the performance and limitations of ACERL in out-of-distribution cases, a leave-one-out cross-validation is conducted by dividing the training dataset into subsidiary training datasets and test datasets.
  \item To examine each ingredient of ACERL including intrinsic stochastic ranking, rank-based fitness and adaptive training, an ablation study is performed to present their unique contributions.
\end{itemize}

\subsection{Experiment setting}
Experiment settings including hyperparameters, problem instances and metrics for performance assessment are described as follows.

\subsubsection{Implementation details}
The implementations of RL and constrained RL (CRL) policies, including SAC~\cite{haarnoja2018soft} and PPO~\cite{schulman2017proximal},
MAPPO~\cite{li2025real},
RCPOM~\cite{hu2023dmh} and LSAC~\cite{tamar2012policy} are adapted based on Tianshou framework\footnote{\label{foot:tianshou}https://github.com/thu-ml/tianshou}~\cite{weng2021tianshou}. Source codes of RCPOM are provided by \cite{hu2023dmh}.
 The network structure is formed by two hidden fully connected layers $128 \times 128$.
The discounted factor $\gamma$ is 0.97. 
The initial multiplier $\lambda$ and the learning rate of RCPOM are set as 0.001 and 0.0001, respectively.
Population size $\lambda$ is 256. The number of generations $G$ is 128.
The constraint threshold $\xi$ is set as 50. 
Other common hyperparameters are set following the default setting of Tianshou\footref{foot:tianshou}. All learning policies are trained for 1e6 steps with five different seeds and each is tested 30 times independently.

\subsubsection{Problem simulator and instances}
The experiments are conducted on publicly available DMH instances and simulator, DMH-GYM\footnote{\url{https://github.com/HcPlu/DMH-GYM}}, provided in~\cite{hu2023dmh}. 
The training instances (DMH-01 to DMH-08) are drawn from different distributions using a searching-based method, while test instances (DMH-09 to DMH-16) are generated by noising the training instances~\cite{hu2023dmh}.

\subsubsection{Evaluation metrics}
Three metrics are used to evaluate the policies, including the average normalised score of makespan $M$, the average
normalised score of tardiness $C$, and the average constraint satisfaction percentage $P$, formulated as:
\begin{equation}
M = \frac{1}{K}\sum_{j=1}^K \frac{F_m^{max}-F_m^{j}}{F_m^{max}-F_m^{min}},
  \label{eq:eq_m}
\end{equation}
\begin{equation}
C = \frac{1}{K}\sum_{j=1}^K \frac{F_t^{max}-F_t^{j}}{F_t^{max}-F_t^{min}},
  \label{eq:eq_C}
\end{equation}

\begin{equation}
  P = \frac{1}{K}\sum_{j=1}^K \mathds{1}_{F_t^{j}<\xi},
  \label{eq:eq_P}
\end{equation}
where $K$ is the number of instances and $\xi$ is the constraint threshold. 
$F_m^{max}$ and $F_m^{min}$ represent the maximal and minimal makespan values $F_m$ (Eq. \eqref{eq:makespan}) among all the policies. Similarly, $F_t^{max}$ and $F_t^{min}$ represent the maximal and minimal tardiness values $F_t$ (Eq. \eqref{eq:tardiness}) among all policies.
All metrics follow the principle that a larger value indicates better performance of the policy.

\subsection{Comparison with state-of-the-arts and baselines}
\label{sec:com}
Tables \ref{tab:comparision} and \ref{tab:comparision_unseen} present the experiment results of ACERL on training instances (DMH-01 to DMH-08) and test instances (DMH-09 to DMH-16), respectively, compared with advanced RL methods, CRL methods and classic dispatching rules. It is obvious that our proposed method, ACERL, statistically outperforms other algorithms on all the training and test instances except on DMH-07, DMH-13 and DMH-15. Compared to ACERL, the dispatching rule FCFS gets competitive makespan but statistically worse tardiness on DMH-07, while other compared algorithms perform statistically worse than ACERL. 
Similar observations are also found on DMH-13 and DMH-15.
Overall, ACERL has the best values in terms of normalised makespan $M$, normalised tardiness $C$ and constraint satisfaction percentage $P$, and achieves the overall best performance on all instances.

Besides, ACERL achieves 100\% and 97\% constraint satisfaction on training and test instances, respectively, which is much better than other algorithms. For example, RCPOM~\cite{hu2023dmh}, a state-of-the-art CRL method, only gets 60\% and 63\% constraint satisfaction on training and test instances, respectively. 

As an example, the training curves of DMH-01 are shown in Fig. \ref{fig:tc_0}. ACERL shows the best performance with fully satisfying the constraint in DMH-01. The training time of ACERL is about 30 minutes, while other RL and CRL-based methods take more than 1 hour for training 1e6 steps. Moreover, the computation time of ACERL for one single decision is about 2 ms, which meets the real-time requirement~\cite{hu2020deep}.
More training curves on various instances are provided in Section A of \emph{Supplementary Material}.

\begin{figure}[!h]
  \centering
  \subfigure{
		\centering
 		\includegraphics[width=0.9\columnwidth,trim=0 27mm 0 0, clip]{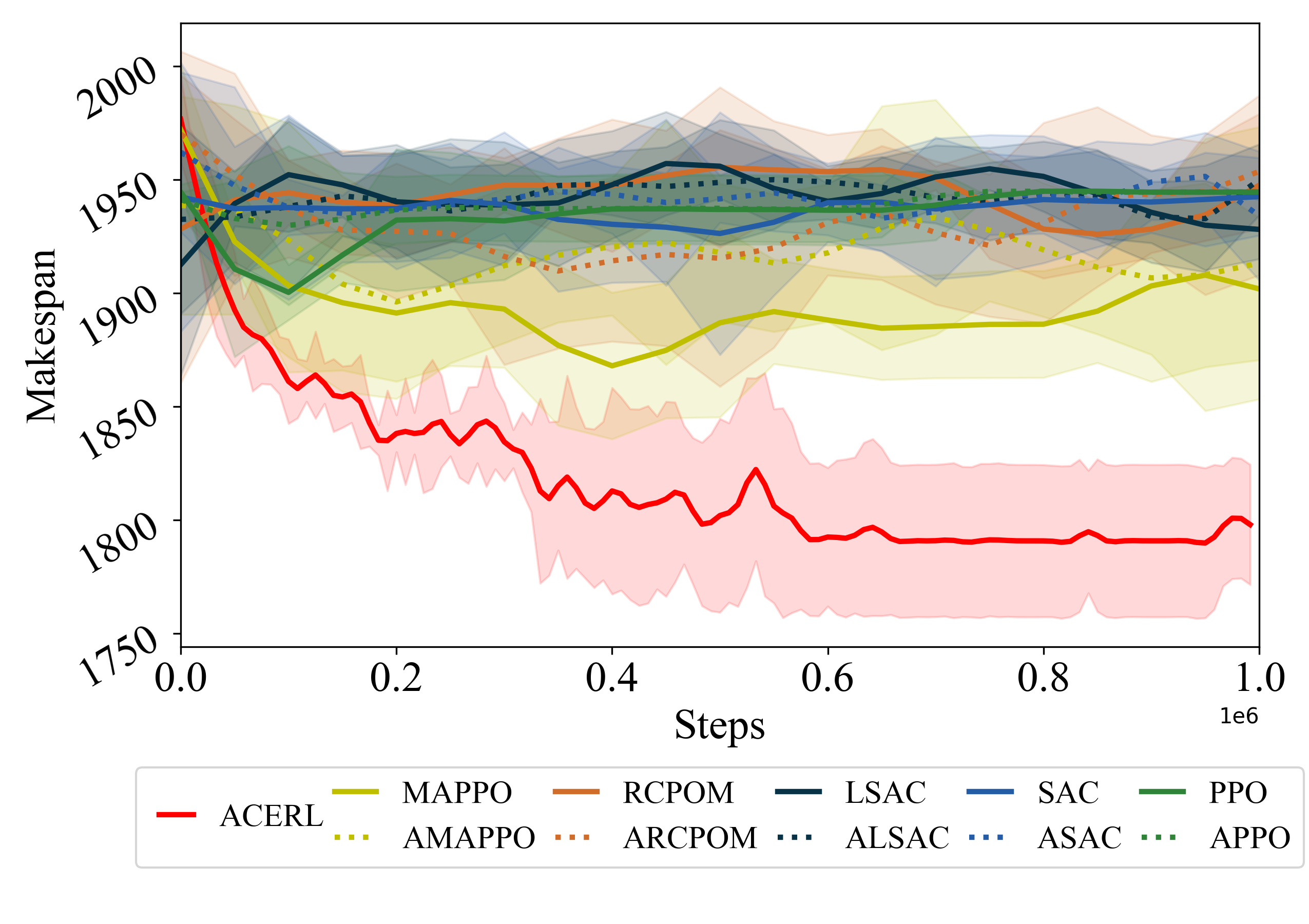}
 }	
 
   \subfigure{
		\centering
		\includegraphics[width=0.9\columnwidth]{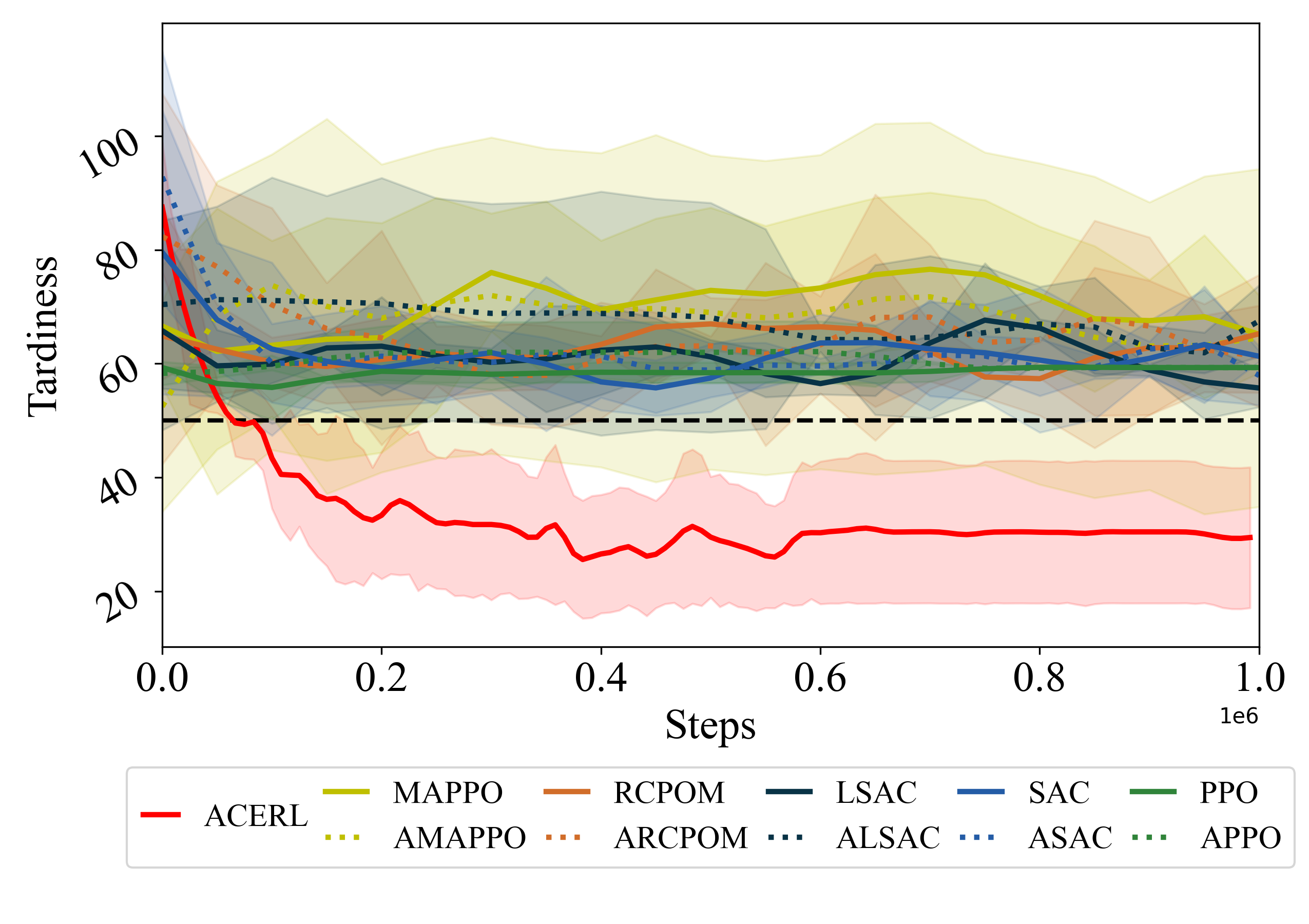}
	}

	\caption{\label{fig:tc_0}Training curves on DMH-01. ACERL (red curve) obtains the best makespan while fully satisfying the tardiness constraint.}
\end{figure}

\begin{table*}[t]
    \centering
        \caption{Average makespan and tardiness over 30 independent trials of five different seeds on training instances. Bold numbers indicate the best makespan and tardiness. ``+'',``$\approx$'' and ``-'' indicate the policy performs statistically better/similar/worse than ACERL policy. The number of policies that are ``better'', ``similar'' and ``worse'' than ACERL in terms of makespan and tardiness on each instance is summarised in the bottom row. The last column with header ``$M/C (P)$'' indicates the average normalised makespan, tardiness and percentage of constraint satisfaction. Horizontal rules in the table separate different groups of algorithms.}

    \resizebox{\textwidth}{!}{
      \setlength{\tabcolsep}{2pt}
      
    \begin{tabular}{c|c|c|c|c|c|c|c|c|c}
\toprule
\multirow{2}{*}{Algorithm}&DMH-01 & DMH-02 & DMH-03 & DMH-04 & DMH-05 & DMH-06 & DMH-07 & DMH-08 & \multirow{2}{*}{$M/C (P)$} \\
&$F_m$/$F_t$ & $F_m$/$F_t$& $F_m$/$F_t$& $F_m$/$F_t$& $F_m$/$F_t$& $F_m$/$F_t$& $F_m$/$F_t$& $F_m$/$F_t$&\\
\midrule
ACERL&\textbf{1797.8}/\textbf{29.5} & \textbf{1859.2}/30.7 & \textbf{1840.0}/28.6 & \textbf{1896.6}/35.7 & \textbf{1856.4}/\textbf{29.0} & \textbf{1864.4}/\textbf{35.5} & 1929.6/\textbf{9.5} & \textbf{1856.8}/\textbf{25.4}&\textbf{0.90}/\textbf{0.90} (\textbf{100\%})\\
\midrule
MAPPO&1869.2-/68.3- & 1926.6-/42.3- & 1925.6-/53.0- & 1908.2-/\textbf{32.2}+ & 1928.4-/42.1- & 1955.8-/79.4- & 1945.0$\approx$/20.4- & 1875.4-/30.2-&0.72/0.72 (62\%)\\
RCPOM&1875.5-/62.8- & 1932.6-/49.1- & 1912.1-/51.8- & 1947.7-/37.4$\approx$ & 1937.6-/44.4- & 1946.9-/58.5- & 1937.9-/12.0- & 1883.6-/27.9$\approx$&0.71/0.75 (60\%)\\
LSAC&1922.2-/65.7- & 1968.1-/53.4- & 1935.9-/49.2- & 2012.0-/54.6- & 1960.6-/48.1- & 1997.7-/73.1- & 1971.4-/15.6- & 1934.3-/38.1-&0.55/0.68 (56\%)\\
AMAPPO&1890.8-/73.2- & 1936.4-/43.3- & 1931.2-/42.9- & 1928.6-/42.4- & 1925.4-/47.1- & 1952.0-/90.5- & \textbf{1927.0}+/14.3- & 1897.6-/39.1-&0.70/0.69 (57\%)\\
ARCPOM&1885.8-/76.0- & 1948.3-/54.6- & 1922.1-/56.3- & 1979.2-/46.0- & 1905.0-/37.6- & 1948.5-/87.6- & 1979.0-/18.0- & 1928.1-/43.1-&0.63/0.65 (48\%)\\
ALSAC&1903.6-/61.7- & 1964.9-/51.6- & 1899.1-/45.2- & 1986.6-/47.8- & 1967.0-/54.3- & 2009.3-/73.7- & 1966.5-/16.8- & 1917.5-/38.3-&0.60/0.69 (53\%)\\
\midrule
SAC&1899.4-/59.6- & 1973.1-/52.5- & 1925.6-/43.9- & 1991.6-/51.3- & 1990.4-/52.9- & 2013.4-/80.2- & 1952.0-/14.7- & 1933.3-/37.2-&0.57/0.69 (57\%)\\
PPO&1877.4-/57.4- & 1954.1-/50.6- & 1920.2-/45.8- & 1979.5-/51.8- & 1955.4-/47.0- & 1978.9-/68.9- & 1973.3-/16.5- & 1920.3-/34.1-&0.61/0.71 (58\%)\\
ASAC&1913.2-/70.5- & 1985.9-/55.3- & 1944.8-/47.6- & 1985.5-/49.1- & 1987.7-/53.5- & 2027.5-/92.3- & 1950.2-/14.9- & 1934.1-/39.8-&0.55/0.65 (52\%)\\
APPO&1899.6-/72.9- & 1951.4-/51.3- & 1922.5-/54.3- & 1972.5-/42.0- & 1949.9-/54.3- & 1992.8-/97.7- & 1967.9-/19.3- & 1905.6-/34.8-&0.62/0.65 (50\%)\\
\midrule
MIX&1939.9-/60.3- & 2000.9-/60.4- & 1932.2-/45.0- & 2006.1-/52.8- & 2028.2-/64.2- & 2027.1-/69.5- & 1969.2-/16.1- & 1971.0-/46.9-&0.49/0.65 (54\%)\\
FCFS&2081.1-/90.2- & 2084.5-/82.3- & 2014.7-/64.2- & 2136.7-/105.0- & 2194.6-/122.5- & 2123.5-/85.2- & 1927.9$\approx$/11.2- & 1933.6-/27.4$\approx$&0.25/0.45 (37\%)\\
EDD&1903.2-/33.9$\approx$ & 1968.6-/\textbf{29.6}+ & 1977.8-/\textbf{26.5}$\approx$ & 1988.1-/32.8+ & 1950.7-/33.7$\approx$ & 2016.8-/46.8$\approx$ & 1940.5-/12.1- & 2020.8-/45.3-&0.52/0.85 (86\%)\\
NVF&1876.5-/69.6- & 1958.4-/56.4- & 1946.7-/49.6- & 2040.6-/66.7- & 1933.9-/56.3- & 1953.4-/78.5- & 1996.5-/21.8- & 1944.6-/35.5-&0.56/0.63 (51\%)\\
STD&1868.8-/66.6- & 1961.7-/49.2- & 1921.3-/51.7- & 1970.7-/35.6$\approx$ & 1917.1-/41.4- & 1955.4-/62.7- & 1983.7-/30.4- & 1883.5-/35.1-&0.65/0.70 (53\%)\\
Random&2098.7-/124.5- & 2113.8-/103.1- & 2091.2-/143.7- & 2135.1-/123.1- & 2149.3-/119.0- & 2159.1-/129.3- & 2083.0-/70.2- & 2067.8-/89.5-&0.02/0.01 (12\%)\\
\midrule
& (0/0/16) / (0/1/15)& (0/0/16) / (1/0/15)& (0/0/16) / (0/1/15)& (0/0/16) / (2/2/12)& (0/0/16) / (0/1/15)& (0/0/16) / (0/1/15)& (1/2/13) / (0/0/16)& (0/0/16) / (0/2/14)\\

\bottomrule
    \end{tabular}
    }

    \label{tab:comparision}
\end{table*}

\begin{table*}[t]
    \centering
        \caption{Average makespan and tardiness over 30 independent trials of five different seeds on test instances. Bold numbers indicate the best makespan and tardiness. ``+'',``$\approx$'' and ``-'' indicate the policy performs statistically better/similar/worse than ACERL policy. The number of policies that are ``better'', ``similar'' and ``worse'' than ACERL in terms of makespan and tardiness on each instance is summarised in the bottom row. The last column with header ``$M/C (P)$'' indicates the average normalised makespan, tardiness and percentage of constraint satisfaction. Horizontal rules in the table separate different groups of algorithms.}
    \resizebox{\textwidth}{!}{
      \setlength{\tabcolsep}{2pt}
    \begin{tabular}{c|c|c|c|c|c|c|c|c|c}
\toprule
\multirow{2}{*}{Algorithm} &DMH-09 & DMH-10 & DMH-11 & DMH-12 & DMH-13 & DMH-14 & DMH-15 & DMH-16 & \multirow{2}{*}{$M/C (P)$} \\
 &$F_m$/$F_t$ & $F_m$/$F_t$& $F_m$/$F_t$& $F_m$/$F_t$& $F_m$/$F_t$& $F_m$/$F_t$& $F_m$/$F_t$& $F_m$/$F_t$&\\
\midrule
ACERL&\textbf{1801.6}/\textbf{29.9} & \textbf{1878.4}/\textbf{30.6} & \textbf{1892.7}/34.7 & \textbf{1896.8}/35.6 & \textbf{1894.9}/\textbf{24.2} & \textbf{1865.0}/\textbf{36.1} & \textbf{1932.8}/\textbf{9.6} & \textbf{1857.0}/\textbf{25.8}&\textbf{0.89}/\textbf{0.92} (\textbf{97\%})\\
\midrule
MAPPO&1880.8-/68.7- & 1928.4-/43.4- & 1931.8-/49.7- & 1908.6-/\textbf{32.5}+ & 1923.2-/34.9- & 2034.4-/92.2- & 1972.8-/28.8- & 1875.8-/30.9-&0.70/0.73 (65\%)\\
RCPOM&1877.2-/63.2- & 1935.0-/49.2- & 1908.0-/51.7- & 1946.8-/37.6$\approx$ & 1928.2-/35.4- & 1951.9-/59.9- & 1970.4-/23.5- & 1886.0-/30.4-&0.71/0.76 (63\%)\\
LSAC&1918.0-/66.2- & 1966.5-/53.4- & 1934.4-/49.5- & 2011.9-/54.6- & 1961.6-/41.4- & 1996.1-/74.1- & 1972.1-/15.5- & 1932.2-/39.0-&0.57/0.70 (56\%)\\
AMAPPO&1879.8-/73.6- & 1937.6-/43.3- & 1937.6-/43.2- & 1930.0-/41.2- & 1924.8-/33.6- & 2056.4-/105.8- & 1954.8$\approx$/24.4- & 1894.0-/40.4-&0.67/0.69 (57\%)\\
ARCPOM&1902.3-/78.3- & 1949.0-/54.4- & 1920.3-/55.4- & 1974.0-/44.1- & 1908.3-/28.0- & 1963.4-/92.9- & 1975.1-/16.2- & 1931.5-/43.9-&0.65/0.68 (52\%)\\
ALSAC&1893.7-/62.1- & 1955.6-/49.8- & 1929.9-/48.3- & 2001.2-/47.5- & 1955.9-/45.2- & 2005.5-/73.1- & 1974.5-/20.4- & 1913.3-/38.7-&0.60/0.71 (55\%)\\
\midrule
SAC&1906.3-/62.0- & 1972.8-/52.5- & 1931.8-/44.9- & 1993.4-/51.2- & 1975.8-/49.3- & 2012.3-/81.3- & 1952.2-/14.7- & 1924.9-/38.0-&0.59/0.70 (58\%)\\
PPO&1875.4-/58.5- & 1952.8-/50.4- & 1916.3-/46.7- & 1981.7-/52.0- & 1956.3-/44.2- & 1976.6-/70.6- & 1975.2-/16.8- & 1918.9-/34.9-&0.64/0.73 (59\%)\\
APPO&1896.3-/73.2- & 1950.9-/51.1- & 1915.4-/55.1- & 1975.2-/41.6- & 1934.5-/38.5- & 1996.4-/98.8- & 1972.6-/20.2- & 1890.4-/34.5-&0.65/0.68 (56\%)\\
ASAC&1920.4-/72.3- & 1987.4-/55.2- & 1944.6-/48.4- & 1988.1-/49.0- & 1977.7-/47.8- & 2031.9-/95.3- & 1950.6-/14.7- & 1934.6-/41.0-&0.57/0.66 (53\%)\\
\midrule
MIX&1941.7-/63.9- & 2004.3-/60.2- & 1932.3-/46.3- & 2008.3-/52.7- & 2030.6-/62.4- & 2032.3-/72.5- & 1971.5-/16.2- & 1976.9-/49.4-&0.49/0.66 (54\%)\\
FCFS&2089.3-/92.5- & 2045.4-/75.7- & 1996.9-/67.6- & 2107.1-/95.2- & 2191.9-/121.7- & 2130.1-/89.9- & 1934.1$\approx$/11.1- & 1946.8-/35.8-&0.27/0.45 (35\%)\\
EDD&1996.9-/107.7- & 1976.3-/33.4$\approx$ & 1978.7-/\textbf{22.0}+ & 1997.6-/36.9$\approx$ & 1962.1-/37.0- & 1993.5-/44.6- & 1934.8$\approx$/11.0- & 2052.3-/69.4-&0.48/0.74 (73\%)\\
NVF&1847.5-/57.2- & 1958.8-/54.3- & 1926.6-/51.8- & 2021.7-/65.9- & 1939.2-/41.2- & 1975.5-/89.8- & 2003.5-/19.3- & 1933.8-/40.9-&0.59/0.66 (58\%)\\
STD&1894.1-/72.5- & 1958.2-/49.7- & 1923.1-/52.4- & 1985.2-/38.1$\approx$ & 1905.1$\approx$/31.4- & 1974.3-/71.9- & 1990.3-/33.0- & 1885.2-/38.6-&0.65/0.71 (55\%)\\
Random&2132.7-/135.3- & 2100.9-/99.5- & 2076.2-/144.4- & 2097.5-/106.3- & 2144.1-/102.6- & 2142.9-/123.6- & 2101.3-/81.1- & 2055.1-/94.1-&0.02/0.03 (11\%)\\
\midrule
& (0/0/16) / (0/0/16)& (0/0/16) / (0/1/15)& (0/0/16) / (1/0/15)& (0/0/16) / (1/3/12)& (0/1/15) / (0/0/16)& (0/0/16) / (0/0/16)& (0/3/13) / (0/0/16)& (0/0/16) / (0/0/16)\\

\bottomrule
    \end{tabular}
    }

    \label{tab:comparision_unseen}
\end{table*}

\subsubsection{Tackling sparse feedback}

Tables \ref{tab:comparision} and \ref{tab:comparision_unseen} show that none of the RL or CRL methods including MAPPO, SAC, PPO, RCPOM and LSAC, performs better than ACERL on all the 16 training and test instances.
SAC and PPO are even worse than the dispatching rule STD, which achieves the best $M=0.65$ among all dispatching rules. We attribute the phenomenon to the lack of temporary information. 
Regular environments and benchmarks considered in RL studies usually are well defined with suitable reward functions~\cite{OpenAIGym}. It is hard to guarantee that the same conditions are provided in real-world scenarios.
Objectives including makespan and tardiness are only obtained after all tasks are completed in DMH. Given an RL agent, it only receives one meaningful reward at the end of an episode, while getting zero for other steps. This also applies to constraint satisfaction part, as the cost function related to tardiness is also sparse. 
MAPPO~\cite{li2025real} and RCPOM~\cite{hu2023dmh} achieve similar results. RCPOM applies invariant reward shaping to address sparse feedback but still fails to achieve superior results. On the other hand, MAPPO suffers from the high requirement of training policies with a globally centralized critic.

ACERL is validated to handle sparse feedback by the experiment results.
Instead of stochastic gradient descent with value function or policy gradient, ACERL applies the natural gradient ascent to update a policy. The usage of rank-based fitness makes it possible to get rid of the temporary information that regular RL methods require since only the last episodic value is needed.
Thus, ACERL can tackle sparse feedback without reward shaping relying on domain knowledge.

\begin{table*}[h]
    \centering
        \caption{Performance on noised instances. ``$M/C (P)$'' indicates the average normalised makespan, tardiness and percentage of constraint satisfaction over 30 independent trials of five different seeds. Bold number indicates the best value in the corresponding column. Horizontal rules in the table separate different groups of algorithms.}
      \setlength{\tabcolsep}{5pt}
    \begin{tabular}{c|c|c|c|c|c|c|c}
\toprule
\multicolumn{1}{c|}{\multirow{2}{*}{Algorithm} }&DMH$\pm 0$ & DMH$\pm 5$ & DMH$\pm 10$ & DMH$\pm 15$ & DMH$\pm 20$ & DMH$\pm 25$ & DMH$\pm 30$  \\
 &$M/C (P)$ & $M/C (P)$& $M/C (P)$& $M/C (P)$& $M/C (P)$& $M/C (P)$  & $M/C (P)$\\
\midrule
ACERL& \textbf{0.90}/\textbf{0.90} (\textbf{100\%})& \textbf{0.89}/\textbf{0.92} (\textbf{97\%})& \textbf{0.90}/\textbf{0.94} (\textbf{80\%})& \textbf{0.94}/\textbf{0.94} (\textbf{95\%})& \textbf{0.87}/\textbf{0.92} (75\%)& \textbf{0.89}/\textbf{0.96} (\textbf{84\%})& \textbf{0.93}/\textbf{0.96} (\textbf{93\%})\\
\midrule
MAPPO& 0.72/0.72 (62\%)& 0.70/0.73 (65\%)& 0.79/0.83 (65\%)& 0.87/0.86 (85\%)& 0.83/0.89 (72\%)& 0.70/0.77 (55\%)& 0.78/0.82 (70\%)\\
RCPOM& 0.71/0.75 (60\%)& 0.71/0.76 (63\%)& 0.85/0.88 (64\%)& 0.85/0.88 (82\%)& \textbf{0.87}/0.91 (66\%)& 0.81/0.84 (60\%)& 0.87/0.87 (70\%)\\
LSAC& 0.55/0.68 (56\%)& 0.57/0.70 (56\%)& 0.71/0.81 (57\%)& 0.67/0.74 (56\%)& 0.71/0.81 (56\%)& 0.70/0.81 (60\%)& 0.74/0.79 (60\%)\\
AMAPPO& 0.70/0.69 (57\%)& 0.67/0.69 (57\%)& 0.78/0.82 (67\%)& 0.79/0.79 (70\%)& 0.86/0.85 (70\%)& 0.77/0.77 (60\%)& 0.83/0.84 (70\%)\\
ARCPOM& 0.63/0.65 (48\%)& 0.65/0.68 (52\%)& 0.83/0.81 (57\%)& 0.76/0.76 (61\%)& 0.82/0.83 (57\%)& 0.85/0.84 (63\%)& 0.84/0.82 (69\%)\\
ALSAC& 0.60/0.69 (53\%)& 0.60/0.71 (55\%)& 0.77/0.84 (63\%)& 0.76/0.78 (66\%)& 0.78/0.85 (60\%)& 0.72/0.80 (60\%)& 0.79/0.79 (68\%)\\
\midrule
SAC& 0.57/0.69 (57\%)& 0.59/0.70 (58\%)& 0.73/0.82 (59\%)& 0.70/0.76 (66\%)& 0.72/0.81 (57\%)& 0.71/0.80 (57\%)& 0.75/0.80 (62\%)\\
PPO& 0.61/0.71 (58\%)& 0.64/0.73 (59\%)& 0.77/0.83 (60\%)& 0.76/0.79 (63\%)& 0.80/0.85 (60\%)& 0.81/0.84 (60\%)& 0.83/0.83 (66\%)\\
ASAC& 0.55/0.65 (52\%)& 0.65/0.68 (56\%)& 0.70/0.78 (55\%)& 0.68/0.74 (63\%)& 0.68/0.77 (53\%)& 0.67/0.76 (55\%)& 0.69/0.76 (60\%)\\
APPO& 0.62/0.65 (50\%)& 0.57/0.66 (53\%)& 0.82/0.79 (58\%)& 0.78/0.75 (63\%)& 0.83/0.83 (61\%)& 0.84/0.82 (61\%)& 0.83/0.80 (64\%)\\
\midrule
MIX& 0.49/0.65 (54\%)& 0.49/0.66 (54\%)& 0.60/0.76 (54\%)& 0.60/0.73 (59\%)& 0.60/0.75 (47\%)& 0.63/0.77 (57\%)& 0.62/0.75 (53\%)\\
FCFS& 0.25/0.45 (37\%)& 0.27/0.45 (35\%)& 0.24/0.54 (37\%)& 0.18/0.37 (35\%)& 0.26/0.42 (30\%)& 0.26/0.45 (34\%)& 0.12/0.39 (32\%)\\
EDD& 0.52/0.85 (86\%)& 0.48/0.74 (73\%)& 0.55/0.84 (65\%)& 0.52/0.87 (80\%)& 0.43/0.73 (57\%)& 0.46/0.75 (62\%)& 0.46/0.76 (53\%)\\
NVF& 0.56/0.63 (51\%)& 0.59/0.66 (58\%)& 0.77/0.82 (61\%)& 0.76/0.77 (65\%)& 0.86/0.89 (\textbf{77\%})& 0.84/0.86 (62\%)& 0.92/0.89 (81\%)\\
STD& 0.65/0.70 (53\%)& 0.65/0.71 (55\%)& 0.86/0.83 (62\%)& 0.82/0.82 (76\%)& 0.82/0.85 (59\%)& 0.72/0.80 (58\%)& 0.82/0.84 (72\%)\\
Random& 0.02/0.00 (12\%)& 0.02/0.03 (11\%)& 0.00/0.00 (10\%)& 0.04/0.04 (13\%)& 0.04/0.01 (9\%)& 0.03/0.02 (7\%)& 0.01/0.00 (9\%)\\
\bottomrule
    \end{tabular}
    \label{tab:noised_sum}
\end{table*}

\subsubsection{Promising constraint satisfaction}
 Delay of task finishing time, i.e., tardiness, should be restricted within the given threshold in DMH. It also needs to guarantee safety while assigning tasks since only AGV in the idle state are available to serve tasks. They are considered as cumulative and instantaneous constraints and are handled as long-term behaviours and short-term behaviours, respectively. However, it is an intractable problem for RL methods to handle constraints~\cite{tessler2018reward}. The complexity arises from the need of maximising reward, i.e., minimising makespan while ensuring strict adherence to constraints. The balance is crucial that agents achieving high reward can still be outperformed by seemingly ``poor" agents that prioritise constraint satisfaction. This phenomenon shows the critical importance of constraint management in real-world applications, where violating operational limits (such as excessive tardiness) can have severe consequences like human safety.
 RCPOM~\cite{hu2023dmh} combines RCPO~\cite{tessler2018reward} and invalid action masking to deal with the tardiness and available vehicle constraints, respectively. Although the masking technique guarantees the satisfaction of the available vehicle constraint, it actually does not perform well with an inferior constraint satisfaction percentage in terms of tardiness than EDD, a time-prefer dispatching rule.

 Inherent nature of DMH systems introduces an additional layer of complexity through sparse feedback mechanisms. In DMH, meaningful rewards or penalty signals are  only available once all tasks are served. This sparsity in the feedback loop exacerbates the difficulty of policy optimisation, as the RL agent must learn to make decisions with limited immediate guidance on the long-term consequence of actions.
Tables~\ref{tab:comparision} and \ref{tab:comparision_unseen} show that regular CRL methods like LSAC fail to achieve high constraint satisfaction.
CRL methods including RCPOM and LSAC require the temporary penalty values to restrict the behaviour of the policy, which, however, are usually missing in DMH. This is why, although RCPOM shows promising results compared to RL and CRL methods, it has a worse constraint satisfaction in terms of tardiness than EDD, even with a higher makespan.

ACERL benefits from the balance of maximising rewards and satisfying constraints by the intrinsic stochastic ranking with rank-based fitness. Individuals are grouped into distinct buffers based on the selected problem instances. This grouping allows for more meaningful comparisons between solutions within similar problem contexts as performances among different problems is incomparable. Instead of using the weighted sum of rewards and penalties, ACERL estimates the fitness of each individual based on its rank. The fitness of each individual is determined by its rank within its instance-specific group, serving as an unified metric, which enables fair comparisons across potentially diverse problem instances. The natural gradient induced by these fitness metrics efficiently guides the policy update process with the most promising direction for improving both reward and constraint satisfaction simultaneously. The population evolves towards feasible regions of the policy space that offer superior performance across multiple problem instances, promoting the development of policies that can generalise well to unseen DMH scenarios.

\subsubsection{Stable performance across multiple instances}
DMH considers training an agent on multiple instances and generalising to multiple unseen instances. Each instance has its own context, like different objective value ranges and representations. This means the performance of the same agent on different instances is incomparable. The challenge lies not only in achieving high performance on individual instances but in developing a robust policy that can adapt to the underlying patterns and principles of the dynamic system.

As shown in Tables \ref{tab:comparision} and \ref{tab:comparision_unseen}, ACERL shows the best performance on all instances. The diverse problem instances present a challenge in achieving an overall good performance for a single policy.
There is no single dispatching rule that performs the best on all instances. EDD, a time-prefer rule has the minimal tardiness $F_t$ in DMH-02, DMH-03 and DMH-04 with 29.6, 26.5 and 32.8, respectively. But it rarely gets good results on makespan. A similar case happens on NVF, a distance-prefer rule, which performs fairly on makespan $F_m=1876.5$ but violates the tardiness constraint on DMH-01 with $69.6<50$. 

ACERL tackles the trade-off of multiple training instances. Historical metrics such as obtained episodic reward and the number of selecting an instance are collected. An inverted distance metric is used
to estimate how good the individual is in the view of the time horizon.
 We apply an UCB-based sampler to implement the adaptive training with the inverted distance metric and the number of selections. If an instance has been optimised well, i.e., the distance metric is small enough, then the proportion of the training instance decreases accordingly.
The sampled population then interacts with environments instantiated by the adaptively selected training instances, for which the policy benefits the most.
By adaptive training, the corresponding computational resource is allocated suitably. 
ACERL can explore and exploit more instances on which it performs badly to achieve an overall good performance. 
Besides, bias estimated by the intrinsic stochastic ranking provokes some optimisation pressure which guides the gradient directions and helps with the next round's instance selection.

\begin{table*}[ht]
    \centering
            \caption{Cross validation using leave-one-out on DMH-01 (highlighted with grey blocks): average makespan and tardiness over 30 independent trials of five different seeds on each instance. Bold numbers indicate the best makespan and tardiness. ``+'',``$\approx$'' and ``-'' indicate the policy performs statistically better/similar/worse than ACERL policy. The number of policies that are ``better'', ``similar'' and ``worse'' than ACERL in terms of makespan and tardiness on each instance is summarised in the bottom row. The last column with header ``$M/C (P)$'' indicates the average normalised makespan, tardiness and percentage of constraint satisfaction. Horizontal rules in the table separate different groups of algorithms.}
    \resizebox{\textwidth}{!}{
      \setlength{\tabcolsep}{1pt}
    \begin{tabular}{c|c|c|c|c|c|c|c|c|c}
\toprule
\multirow{2}{*}{Algorithm} & \cellcolor[rgb]{.9,.9,.9}DMH-01 & DMH-02 & DMH-03 & DMH-04 & DMH-05 & DMH-06 & DMH-07 & DMH-08 & \multirow{2}{*}{$M/C (P)$} \\
 &\cellcolor[rgb]{.9,.9,.9} $F_m$/$F_t$ & $F_m$/$F_t$& $F_m$/$F_t$& $F_m$/$F_t$& $F_m$/$F_t$& $F_m$/$F_t$& $F_m$/$F_t$& $F_m$/$F_t$&\\
\midrule
ACERL&\cellcolor[rgb]{.9,.9,.9}1914.1/54.0 & \textbf{1861.8}/\textbf{28.6} & \textbf{1834.8}/31.0 & \textbf{1919.8}/37.8 & \textbf{1870.0}/35.3 & \textbf{1878.8}/\textbf{38.6} & \textbf{1927.0}/\textbf{9.1} & \textbf{1864.4}/34.1&0.97/0.94 (87\%)\\
\midrule
MAPPO&\cellcolor[rgb]{.9,.9,.9}\textbf{1867.4}+/62.9- & 1908.8-/36.7- & 1907.0-/47.0- & 1945.2-/41.1- & 1962.2-/55.4- & 2023.8-/102.3- & 1927.0$\approx$/10.8- & 1881.2-/39.2$\approx$&0.82/0.76 (55\%)\\
RCPOM&\cellcolor[rgb]{.9,.9,.9}1900.7+/72.5- & 1957.5-/54.1- & 1925.1-/57.5- & 1970.1-/46.0- & 1919.0-/39.5$\approx$ & 1976.1-/85.9- & 1985.0-/27.2- & 1929.3-/42.1-&0.71/0.71 (49\%)\\
LSAC&\cellcolor[rgb]{.9,.9,.9}1926.8$\approx$/69.7- & 1970.5-/57.4- & 1924.5-/49.9- & 1963.3-/44.4- & 1940.4-/42.6- & 1996.9-/78.2- & 1952.9-/15.6- & 1901.6-/30.9+&0.72/0.77 (52\%)\\
AMAPPO&\cellcolor[rgb]{.9,.9,.9}1904.2+/66.1$\approx$ & 1949.2-/46.3- & 1889.8-/42.6- & 1950.6-/39.9- & 1892.4-/\textbf{33.0}+ & 1989.8-/79.2- & 1932.8-/13.4- & 1889.0-/\textbf{23.4}+&0.81/0.83 (62\%)\\
ARCPOM&\cellcolor[rgb]{.9,.9,.9}1901.3$\approx$/59.1$\approx$ & 1988.2-/53.1- & 1927.1-/46.4- & 1986.8-/51.6- & 1996.5-/50.6- & 1991.6-/67.8- & 1955.7-/17.1- & 1932.7-/34.2-&0.67/0.78 (54\%)\\
ALSAC&\cellcolor[rgb]{.9,.9,.9}1901.3$\approx$/56.4- & 1968.8-/53.1- & 1936.9-/42.7- & 1985.3-/49.1- & 1968.0-/51.2- & 1990.3-/72.1- & 1959.6-/13.3- & 1933.5-/34.5$\approx$&0.68/0.79 (58\%)\\
\midrule
SAC&\cellcolor[rgb]{.9,.9,.9}1893.1+/64.5- & 1972.1-/58.3- & 1925.5-/43.6- & 1971.6-/44.3- & 1936.2-/41.8- & 1978.5-/72.2- & 1955.2-/16.4- & 1925.9-/33.1+&0.73/0.78 (53\%)\\
PPO&\cellcolor[rgb]{.9,.9,.9}1901.9+/60.2$\approx$ & 1994.2-/56.6- & 1931.4-/43.9- & 1989.7-/50.9- & 1979.5-/54.4- & 2008.2-/84.8- & 1957.8-/17.3- & 1922.9-/37.1$\approx$&0.67/0.74 (53\%)\\
ASAC&\cellcolor[rgb]{.9,.9,.9}1901.6$\approx$/55.4$\approx$ & 1985.6-/54.6- & 1924.2-/44.2- & 2010.8-/53.0- & 1974.1-/47.7- & 1997.7-/70.1- & 1964.6-/14.9- & 1925.1-/36.8$\approx$&0.66/0.78 (58\%)\\
APPO&\cellcolor[rgb]{.9,.9,.9}1921.4$\approx$/62.0$\approx$ & 2006.3-/59.7- & 1942.5-/45.7- & 2001.9-/53.4- & 2005.3-/57.2- & 2057.6-/95.2- & 1948.0-/15.7- & 1940.7-/37.4$\approx$&0.60/0.71 (52\%)\\
\midrule
MIX&\cellcolor[rgb]{.9,.9,.9}1939.9-/60.3$\approx$ & 2000.9-/60.4- & 1932.2-/45.0- & 2006.1-/52.8- & 2028.2-/64.2- & 2027.1-/69.5- & 1969.2-/16.1- & 1971.0-/46.9-&0.57/0.72 (54\%)\\
FCFS&\cellcolor[rgb]{.9,.9,.9}2081.1-/90.2- & 2084.5-/82.3- & 2014.7-/64.2- & 2136.7-/105.0- & 2194.6-/122.5- & 2123.5-/85.2- & 1927.9$\approx$/11.2- & 1933.6-/27.4+&0.28/0.49 (37\%)\\
EDD&\cellcolor[rgb]{.9,.9,.9}1903.2$\approx$/\textbf{33.9}+ & 1968.6-/29.6$\approx$ & 1977.8-/\textbf{26.5}$\approx$ & 1988.1-/\textbf{32.8}+ & 1950.7-/33.7$\approx$ & 2016.8-/46.8$\approx$ & 1940.5-/12.1- & 2020.8-/45.3$\approx$&0.62/0.94 (86\%)\\
NVF&\cellcolor[rgb]{.9,.9,.9}1876.5+/69.6- & 1958.4-/56.4- & 1946.7-/49.6- & 2040.6-/66.7- & 1933.9-/56.3- & 1953.4-/78.5- & 1996.5-/21.8- & 1944.6-/35.5$\approx$&0.66/0.70 (51\%)\\
STD&\cellcolor[rgb]{.9,.9,.9}1868.8+/66.6$\approx$ & 1961.7-/49.2- & 1921.3-/51.7- & 1970.7-/35.6+ & 1917.1-/41.4- & 1955.4-/62.7- & 1983.7-/30.4- & 1883.5-/35.1$\approx$&0.77/0.78 (53\%)\\
Random&\cellcolor[rgb]{.9,.9,.9}2098.7-/124.5- & 2113.8-/103.1- & 2091.2-/143.7- & 2135.1-/123.1- & 2149.3-/119.0- & 2159.1-/129.3- & 2083.0-/70.2- & 2067.8-/89.5-&0.02/0.00 (12\%)\\
\midrule
& \cellcolor[rgb]{.9,.9,.9}(7/6/3) / (1/7/8)& (0/0/16) / (0/1/15)& (0/0/16) / (0/1/15)& (0/0/16) / (2/0/14)& (0/0/16) / (1/2/13)& (0/0/16) / (0/1/15)& (0/2/14) / (0/0/16)& (0/0/16) / (4/8/4)\\

\bottomrule
    \end{tabular}
    }

    \label{tab:l1o1}
\end{table*}

\subsection{Robust performance on noised instances}
\label{sec:noise}

 To further validate the effectiveness of ACERL, we validate the algorithm on some noised datasets. We extend the problem dataset by introducing random perturbations to the original data. The arrival time of each task is perturbed by this random noise, which is defined using the noise operator denoted as $\pm$. The number following $\pm$ specifies the magnitude of the perturbations.
The initial training dataset and test dataset are labeled as DMH$\pm0$ and DMH$\pm5$, respectively. Subsequently, we created extended instance datasets with increasing levels of perturbation strength, labelled as DMH$\pm10$, DMH$\pm15$, DMH$\pm20$, DMH$\pm25$, and DMH$\pm30$.
Table \ref{tab:noised_sum} summarises the experiment results on noised datasets. More results are detailed in Section B of \emph{Supplementary Material}.

It is observed that the performance rank of dispatching rules changes across different problem datasets. In DMH$\pm 0$, the original training dataset, EDD has the highest constraint satisfaction. However, its constraint satisfaction descends and is even worse than NVF in DMH$\pm 20$, DMH$\pm 25$ and DMH$\pm 30$. A similar case happens in makespan, NVF gradually achieves a better $M$ over STD. 
 ACERL outperforms other algorithms on several noised datasets, on which it has the best performance on metrics $M$, $C$ and $P$. The constraint satisfaction percentage of ACERL is much higher than the other algorithms, e.g., 100\% on DMH$\pm0$, 97\% on DMH$\pm5$, 95\% on DMH$\pm15$ and 93\% on DMH$\pm 30$. ACERL still holds the first place in makespan with respect to metric $M$.
Although RCPOM achieves a slightly higher $M$ with 0.88 over ACERL with 0.87, its constraint satisfaction percentage $P$ is rarely higher than ACERL's, i.e., 66\%$<$75\% in DMH$\pm20$.

MIX selects dispatching rules uniformly at random and guarantees a lower performance bound of assigning tasks. From Table \ref{tab:noised_sum}, MIX has better performance than FCFS in both makespan and tardiness and is better than EDD in makespan. Learning policies including ACERL, MAPPO, RCPOM as well as other RL and CRL policies aim at learning a suitable policy to choose a proper rule within a given state, i.e., improving the lower bound. But they can hardly make it due to the trade-off among multiple instances, as well as sparse feedback.

ACERL learns from the diverse experience sampled by the population. The adaptive selection of training instances prevents ACERL from focusing solely on specific instances during training, enabling a strong overall performance across multiple instances. Intrinsic stochastic ranking balances the conflict in terms of maximising rewards and satisfying constraints according to the episodic values including makespan and tardiness. The obtained rank-based fitness helps with the estimation of the natural gradient for updating the policy. ACERL is particularly well-suited to tackling sparse reward and penalty in DMH, as it requires only episode-level values to guide the learning process.

A notable strength of ACERL lies in maintaining robust performance even when the noised dataset gradually differs from the original training dataset. This robustness suggests that the learned policy captures fundamental principles of DMH rather than merely memorising specific instance characteristics. It indicates potential applications of ACERL in complex real-world scenarios with dynamics and constraints.

\subsection{Leave-one-out cross-validation}
\label{sec:cv}
To further explore ACERL's effectiveness and limitations, we divide the training set into subsidiary training and test sets by the leave-one-out method. For example, DMH-01 is stripped out as the test instance, and then the remaining instances (DMH-02 to 08) are used to train the policy. 
Table \ref{tab:l1o1} shows the experiment results by leaving DMH-01 out. More results are attached in Section C of \emph{Supplementary Material}. 

Still, our proposed method, ACERL achieves the highest $M$ value of 0.98 and 87\% constrained satisfaction. 
Across DMH-02 to DMH-08, ACERL achieves the best makespans ($F_m$).
However, we also find that ACERL does not perform well on the left instance DMH-01 as expected, with $F_m=1914.1$ and $F_t=54.0$. At the same time, STD, a dispatching rule achieves the best makespan value of 1868.8 and EDD achieves the lowest tardiness $F_t=33.9$. 
The phenomenon is attributed to the out-of-distribution. The diverse training instances are intentionally designed to facilitate a comprehensive evaluation, and thus instances may differ from each other significantly.
 It is not surprising to observe the limited performance of learning policies on DMH-01, which is separated from the training set as they have never encountered the instance before. At the same time, policies can truly learn something from the remaining training instances. Learning policies including ACERL, RCPOM and SAC show superior performance than the MIX policy, which selects dispatching rules randomly. Some special knowledge is able to transfer from some ``similar'' instances to the split instance. 
Although ACERL does not demonstrate superior performance in terms of makespan on the split instance, it still achieves the overall best performance.

\subsection{Ablation study}

\begin{table*}[!ht]
    \centering
        \caption{Ablation study on training instances (DMH-01 to DMH-08). Bold number indicates the best makespan and tardiness. ``$M/C (P)$'' indicates the average normalised makespan, tardiness and percentage of constraint satisfaction. Strategy denotes the ranking strategies including ISR, CF, R and F. Mode denotes training modes including AIS, Uniform and Random.
        Different groups of algorithms are divided by the row lines according to the ranking strategies and training modes. 
        }
    \resizebox{\textwidth}{!}{
      \setlength{\tabcolsep}{2pt}
    \begin{tabular}{lc|c|c|c|c|c|c|c|c|c}
\toprule
\multirow{2}{*}{Strategy}& \multirow{2}{*}{Mode}&DMH-01 & DMH-02 & DMH-03 & DMH-04 & DMH-05 & DMH-06 & DMH-07 & DMH-08 & \multirow{2}{*}{$M/C (P)$} \\
& &$F_m$/$F_t$ & $F_m$/$F_t$& $F_m$/$F_t$& $F_m$/$F_t$& $F_m$/$F_t$& $F_m$/$F_t$& $F_m$/$F_t$& $F_m$/$F_t$&\\
\midrule

ISR&AIS&1797.8/\textbf{29.5} & 1859.2/30.7 & 1840.0/28.6 & \textbf{1896.6}/35.7 & 1856.4/29.0 & 1864.4/35.5 & 1929.6/9.5 & 1856.8/25.4&0.90/0.90 (100\%)\\
ISR&Uniform&1817.1/47.5 & 1869.8/33.3 & 1839.4/25.8 & 1906.6/37.1 & \textbf{1842.4}/23.7 & 1873.8/\textbf{31.9} & 1932.8/10.0 & 1863.8/24.5&0.88/0.89 (91\%)\\
ISR&Random&1800.0/42.0 & 1881.9/36.9 & 1860.4/30.2 & 1902.2/34.3 & 1881.2/35.8 & 1865.6/37.7 & 1929.6/9.1 & \textbf{1852.2}/32.0&0.87/0.86 (94\%)\\
\midrule
CF&AIS&1802.0/35.2 & 1858.2/30.5 & 1848.4/20.2 & 1901.9/32.7 & 1855.8/21.8 & 1872.2/38.6 & 1932.8/6.0 & 1874.6/20.6&0.88/0.93 (95\%)\\
CF&Uniform&1802.5/41.8 & 1874.8/33.7 & 1840.0/24.6 & 1930.0/\textbf{32.0} & 1845.8/23.4 & 1875.0/38.2 & \textbf{1927.0}/\textbf{4.8} & 1864.0/25.7&0.88/0.90 (92\%)\\
CF&Random&1815.4/42.6 & 1883.2/39.6 & \textbf{1837.0}/\textbf{19.9} & 1930.2/35.3 & 1844.6/\textbf{20.2} & 1891.6/39.4 & 1932.8/7.0 & 1862.6/24.4&0.86/0.89 (92\%)\\
\midrule
R&AIS&1794.8/34.9 & 1858.2/\textbf{28.6} & 1840.0/22.5 & 1928.7/37.6 & 1863.0/33.7 & \textbf{1854.8}/33.2 & 1927.0/7.9 & 1860.8/\textbf{19.7}&0.89/0.91 (93\%)\\
R&Uniform&1824.2/43.9 & 1860.9/31.5 & 1849.0/27.2 & 1926.3/33.8 & 1861.8/30.7 & 1865.2/37.1 & 1927.0/7.2 & 1859.0/27.7&0.87/0.88 (91\%)\\
R&Random&1815.6/42.0 & \textbf{1849.4}/28.8 & 1839.6/36.2 & 1937.2/34.7 & 1857.0/26.2 & 1883.1/41.7 & 1927.0/7.6 & 1882.6/25.6&0.86/0.88 (87\%)\\
\midrule
F&AIS&\textbf{1794.0}/46.9 & 1866.8/32.4 & 1843.0/33.2 & 1928.2/39.0 & 1865.4/41.4 & 1856.0/35.0 & 1932.8/10.0 & 1854.8/24.6&0.88/0.86 (87\%)\\
F&Uniform&1807.4/52.7 & 1863.4/29.4 & 1844.5/29.5 & 1933.8/36.6 & 1874.2/36.9 & 1876.8/38.1 & 1927.0/10.3 & 1859.2/28.4&0.87/0.85 (88\%)\\
F&Random&1800.2/46.4 & 1881.8/38.4 & 1837.0/24.7 & 1924.6/38.2 & 1850.4/24.9 & 1872.0/38.1 & 1935.4/8.5 & 1857.4/25.7&0.88/0.87 (87\%)\\

\bottomrule
 \multicolumn{5}{l}{``ISR": Intrinsic stochastic ranking with rank-based fitness} & \multicolumn{5}{l}{``CF": Weighted sum of raw rewards and penalties} \\
  \multicolumn{5}{l}{``R": Raw reward with rank-based fitness} & \multicolumn{3}{l}{``F": Raw reward}\\
\multicolumn{5}{l}{``AIS": Adaptive instance sampler} & \multicolumn{3}{l}{``Uniform": Fixed instances}& \multicolumn{3}{l}{``Random": Randomly selecting instances}\\
    \end{tabular}
    }

    \label{tab:ablation}
\end{table*}

\begin{table*}[!t]
    \centering
        \caption{Ablation study on test instances (DMH-09 to DMH-16). Bold number indicates the best makespan and tardiness. ``$M/C (P)$'' indicates the average normalised makespan, tardiness and percentage of constraint satisfaction.  Strategy denotes the ranking strategies including ISR, CF, R and F. Mode denotes training modes including AIS, Uniform and Random.
        Different groups of algorithms are divided by the row lines according to the ranking strategies and training modes. 
        }
    \resizebox{\textwidth}{!}{
      \setlength{\tabcolsep}{2pt}
    \begin{tabular}{lc|c|c|c|c|c|c|c|c|c}
\toprule
\multirow{2}{*}{Strategy}& \multirow{2}{*}{Mode}&DMH-09 & DMH-10 & DMH-11 & DMH-12 & DMH-13 & DMH-14 & DMH-15 & DMH-16 & \multirow{2}{*}{$M/C (P)$} \\
& &$F_m$/$F_t$ & $F_m$/$F_t$& $F_m$/$F_t$& $F_m$/$F_t$& $F_m$/$F_t$& $F_m$/$F_t$& $F_m$/$F_t$& $F_m$/$F_t$&\\
\midrule

ISR& AIS&1801.6/\textbf{29.9} & 1878.4/30.6 & 1892.7/34.7 & \textbf{1896.8}/35.6 & 1894.9/24.2 & 1865.0/36.1 & 1932.8/9.6 & 1857.0/25.8&0.89/0.92 (97\%)\\
ISR&Uniform&1823.0/47.3 & 1891.2/34.2 & 1868.6/28.4 & 1913.7/37.2 & 1878.7/26.4 & 1867.6/\textbf{32.6} & 1934.8/10.0 & 1888.2/26.8&0.87/0.90 (92\%)\\
ISR&Random&1804.6/38.7 & 1879.6/36.5 & 1877.0/29.3 & 1919.6/34.2 & 1923.4/35.0 & 1871.8/37.2 & 1932.9/9.3 & \textbf{1852.0}/32.3&0.88/0.89 (93\%)\\
\midrule
CF&AIS&1802.0/35.5 & 1877.6/30.1 & 1876.0/23.8 & 1925.0/34.2 & 1887.1/\textbf{23.0} & 1888.2/44.1 & 1934.8/7.7 & 1874.4/\textbf{20.8}&0.87/0.93 (92\%)\\
CF&Uniform&1801.7/42.2 & 1884.5/33.7 & 1870.8/24.3 & 1931.4/\textbf{31.8} & \textbf{1853.8}/23.6 & 1874.0/38.6 & 1941.2/\textbf{6.5} & 1870.0/26.7&0.88/0.92 (92\%)\\
CF&Random&1816.0/42.9 & 1882.4/39.3 & 1870.8/\textbf{23.6} & 1940.8/35.2 & 1886.8/24.2 & 1906.0/43.0 & 1934.8/8.6 & 1863.6/24.7&0.85/0.90 (92\%)\\
\midrule
R&AIS&1807.2/34.6 & 1869.6/\textbf{29.6} & 1886.8/26.8 & 1932.2/37.3 & 1884.3/31.0 & \textbf{1854.2}/33.6 & \textbf{1930.0}/8.8 & 1874.6/21.4&0.88/0.93 (93\%)\\
R&Uniform&1824.8/44.4 & \textbf{1863.5}/31.6 & 1874.1/29.5 & 1955.6/34.6 & 1886.7/31.3 & 1886.0/35.3 & 1930.0/8.2 & 1858.8/28.0&0.86/0.90 (91\%)\\
R&Random&1813.5/42.4 & 1873.5/32.1 & 1916.9/46.3 & 1936.8/34.7 & 1885.0/29.2 & 1883.7/42.9 & 1930.0/7.7 & 1880.4/25.6&0.84/0.88 (82\%)\\
\midrule
F&AIS&1801.7/47.0 & 1876.2/34.7 & 1864.8/35.4 & 1927.8/38.7 & 1924.6/39.7 & 1856.4/35.6 & 1965.6/21.2 & 1862.6/25.9&0.86/0.85 (82\%)\\
F&Uniform&1802.8/45.4 & 1906.4/32.9 & 1876.1/37.4 & 1946.8/37.6 & 1915.5/37.3 & 1876.9/38.8 & 1995.4/31.2 & 1860.0/28.9&0.80/0.84 (80\%)\\
F&Random&\textbf{1800.2}/46.8 & 1881.0/38.0 & \textbf{1864.2}/27.5 & 1932.8/38.5 & 1935.0/36.4 & 1888.0/43.9 & 1937.6/8.7 & 1866.2/27.1&0.86/0.87 (82\%)\\
\bottomrule
 \multicolumn{5}{l}{``ISR": Intrinsic stochastic ranking with rank-based fitness} & \multicolumn{5}{l}{``CF": Weighted sum of raw rewards and penalties} \\
  \multicolumn{5}{l}{``R": Raw reward with rank-based fitness} & \multicolumn{3}{l}{``F": Raw reward}\\
\multicolumn{5}{l}{``AIS": Adaptive instance sampler} & \multicolumn{3}{l}{``Uniform": Fixed instances}& \multicolumn{3}{l}{``Random": Randomly selecting instances}\\

    \end{tabular}
    }

    \label{tab:ablation_unseen}
\end{table*}

To fully evaluate ACERL and analyse each ingredient, a rigorous ablation study is conducted. 
Compared strategies are organised into the following two groups.
\begin{itemize}
  \item Four ranking strategies are compared. 
  Intrinsic stochastic ranking with rank-based fitness is denoted as ``ISR''. Weighted sum of raw rewards and penalties is denoted as ``CF''. Rank-based fitness is denoted as ``R''. Raw reward is denoted as ``F''. 
  \item ``AIS'', ``Uniform'' and ``Random'' are three different training modes. ``AIS'' denotes the policy trained with the adaptive instance sampler. ``Uniform'' denotes each instance is equally selected and fixed during training. ``Random'' means instances are selected randomly.
\end{itemize}

\subsubsection{Effect of intrinsic stochastic ranking with rank-based fitness}

Experiment results are presented in Tables \ref{tab:ablation} and \ref{tab:ablation_unseen}. Ranking strategy ``F" can obtain good results on makespan, but it hardly satisfies the constraints with a high rate. Similarly, although ``R", a rank-based fitness strategy, has a slightly better makspan than ``F", it still does not perform well in satisfying the problem constraints, i.e. making the tardiness of tasks descend to a promising threshold. ``CF" that applies the regular weighted sum of makespan and tardiness, tries to improve the likelihood of satisfying the constraints, while only obtaining comparable result with ``F" and ``R".
The results from the ablation study verify the effectiveness of the ``ISR'' strategy, in which we obtain excellent performance on makespan in various training modes with better constraint satisfaction, compared with ``F", ``R" and ``CF". 
ACERL that uses ``ISR'' strategy can achieve $100\%$ constraint satisfaction on the training set in AIS mode.
Our proposed ``ISR'' strategy not only can eliminate the effects of scale between rewards and constraint values, but also can well consider reward maximisation and constraint satisfaction simultaneously.

\subsubsection{Necessity of adaptive training}
When a policy is trained on one single instance, i.e., without assuming multiple instances available as historical record, the policy performs well on the training instance but fails to generalise to new instances. Specifically, it results in a high makespan and struggles to meet constraints on different instances, which underscores the limitations of a model trained on just one instance.
Detailed experiment results are presented in Section A of \emph{Supplementary Material}.
Subsequently, we trained the policies on three training modes, including ``AIS'', ``Uniform'' and ``Random'' to verify the effectiveness of our adaptive approach. 
From Tables \ref{tab:ablation} and \ref{tab:ablation_unseen}, we can observe that using our proposed ``AIS'' as the training model results in a better makespan and a higher constraint satisfaction percentage.

 Compared with the training modes of ``Uniform'' and ``Random'', our proposed ``AIS'' training mode is able to dynamically and adaptively adjust its proposition of training instances during the training process. The suitable choice of training instances enables a more reasonable and efficient allocation of computational resources, hence helps ACERL to achieve an overall good performance across all instances.

\subsubsection{Coordination of ACERL}

 We demonstrated the effectiveness of our proposed ACERL and the unique contributions of its ingredients through rigorous and comprehensive ablation experiments, respectively. Overall, as observed in Tables \ref{tab:ablation} and \ref{tab:ablation_unseen}, our ACERL achieves the best makespan and the highest constraint satisfaction percentage compared to all other combinations of ranking strategies and training modes. The AIS training mode considers multiple instances and adaptively allocates resources based on performance. The ISR strategy effectively mitigates the effect of value scales and achieves a well-balanced trade-off between reward and constraint, resulting in excellent performance across all instances. All ingredients of ACERL including intrinsic stochastic ranking with the rank-based fitness and adaptive training cooperate well and are indispensable for solving the problem.

\section{Conclusion}
\label{sec:conclusion}
In this paper, we consider dynamic material handling problem with uncertainties and sparse feedback, where unexpected dynamic events and constraints present. Due to the dynamics, it is hard to determine the unique contribution of each task assignment, which leads to an issue of sparse feedback in both rewards and penalties. Although the existence of historical task records enables training a policy with multiple instances, the trade-off across training instances poses a challenge with limited computational resources.
To address the challenges, we proposed a robust adaptive constrained evolutionary reinforcement learning approach, called ACERL. 
ACERL leverages natural gradient ascent to update its parameters. The population-based exploration provides diverse experiences by sampling noised actors at each generation. Adaptive instance sampler keeps choosing the training instance from which the policy benefits the most for policy improvement. Reasonable computational resource allocation is allowed.
To balance the rewards and penalties, we present intrinsic stochastic ranking with rank-based fitness. Instead of common reward-based, real-valued fitness metrics, we incorporate rank-based fitness to estimate the natural gradient, which implies the optimisation direction of both maximising rewards and constraint satisfaction at the same time. Intrinsic stochastic ranking provides an independent ranking for different instances based on rewards and penalties, which enables a bias for adaptive instance selection.
Extensive experiments show that ACERL outperforms advanced reinforcement learning methods, constrained reinforcement learning methods and classic dispatching rules on eight training instances and eight test instances in terms of maximising rewards with constraint satisfaction. ACERL not only achieves the best makespan, but also fully satisfies tardiness constraints on nearly all instances.
Besides, ACERL is evaluated on five datasets with a total of 40 noised instances to simulate real-world scenarios. ACERL presents a robust and outstanding performance on all datasets.
Leave-one-out cross-validation is conducted on ACERL by splitting the training dataset into subsidiary training and test datasets, and presents the overall effectiveness of ACERL.
Furthermore, a comprehensive ablation study demonstrates the unique contribution of ACERL's core ingredients.

In future work, it is worth applying ACERL to more real-world problems such as vehicle routing and grid optimisation. It is also interesting to figure out delicate mechanisms for a better generalisation. Besides, theoretical aspect presents an intriguing avenue for future research, potentially bridging the gap between empirical success and theoretical understanding in constrained evolutionary reinforcement learning. Theoretical analysis will be a valuable yet challenging future work.

\bibliographystyle{IEEEtran}
\bibliography{main}

\begin{thebibliography}{10}
\providecommand{\url}[1]{#1}
\csname url@samestyle\endcsname
\providecommand{\newblock}{\relax}
\providecommand{\bibinfo}[2]{#2}
\providecommand{\BIBentrySTDinterwordspacing}{\spaceskip=0pt\relax}
\providecommand{\BIBentryALTinterwordstretchfactor}{4}
\providecommand{\BIBentryALTinterwordspacing}{\spaceskip=\fontdimen2\font plus
\BIBentryALTinterwordstretchfactor\fontdimen3\font minus \fontdimen4\font\relax}
\providecommand{\BIBforeignlanguage}[2]{{%
\expandafter\ifx\csname l@#1\endcsname\relax
\typeout{** WARNING: IEEEtran.bst: No hyphenation pattern has been}%
\typeout{** loaded for the language `#1'. Using the pattern for}%
\typeout{** the default language instead.}%
\else
\language=\csname l@#1\endcsname
\fi
#2}}
\providecommand{\BIBdecl}{\relax}
\BIBdecl

\bibitem{zou2021effective}
W.-Q. Zou, Q.-K. Pan, and L.~Wang, ``An effective multi-objective evolutionary algorithm for solving the {AGV} scheduling problem with pickup and delivery,'' \emph{Knowledge-Based Systems}, vol. 218, p. 106881, 2021.

\bibitem{singh2022matheuristic}
N.~Singh, Q.-V. Dang, A.~Akcay, I.~Adan, and T.~Martagan, ``A matheuristic for {AGV} scheduling with battery constraints,'' \emph{European Journal of Operational Research}, vol. 298, no.~3, pp. 855--873, 2022.

\bibitem{ouelhadj2009survey}
D.~Ouelhadj and S.~Petrovic, ``A survey of dynamic scheduling in manufacturing systems,'' \emph{Journal of scheduling}, vol.~12, no.~4, pp. 417--431, 2009.

\bibitem{kaplanouglu2015multi}
V.~Kaplano{\u{g}}lu, C.~{\c{S}}ahin, A.~Baykaso{\u{g}}lu, R.~Erol, A.~Ekinci, and M.~Demirta{\c{s}}, ``A multi-agent based approach to dynamic scheduling of machines and automated guided vehicles ({AGV}) in manufacturing systems by considering {AGV} breakdowns,'' \emph{International Journal of Engineering Research \& Innovation}, vol.~7, no.~2, pp. 32--38, 2015.

\bibitem{hu2023dmh}
C.~Hu, Z.~Wang, J.~Liu, J.~Wen, B.~Mao, and X.~Yao, ``Constrained reinforcement learning for dynamic material handling,'' in \emph{International Joint Conference on Neural Networks}, 2023, pp. 1--9.

\bibitem{hu2020deep}
H.~Hu, X.~Jia, Q.~He, S.~Fu, and K.~Liu, ``Deep reinforcement learning based {AGV}s real-time scheduling with mixed rule for flexible shop floor in industry 4.0,'' \emph{Computers \& Industrial Engineering}, vol. 149, p. 106749, 2020.

\bibitem{jeong2021reinforcement}
Y.~Jeong, T.~K. Agrawal, E.~Flores-Garc{\'\i}a, and M.~Wiktorsson, ``A reinforcement learning model for material handling task assignment and route planning in dynamic production logistics environment,'' \emph{Procedia CIRP}, vol. 104, pp. 1807--1812, 2021.

\bibitem{dennis2020emergent}
M.~Dennis, N.~Jaques, E.~Vinitsky, A.~Bayen, S.~Russell, A.~Critch, and S.~Levine, ``Emergent complexity and zero-shot transfer via unsupervised environment design,'' \emph{Advances in Neural Information Processing Systems}, vol.~33, pp. 13\,049--13\,061, 2020.

\bibitem{blackstone1982state}
J.~H. Blackstone, D.~T. Phillips, and G.~L. Hogg, ``A state-of-the-art survey of dispatching rules for manufacturing job shop operations,'' \emph{The International Journal of Production Research}, vol.~20, no.~1, pp. 27--45, 1982.

\bibitem{sabuncuoglu1998study}
I.~Sabuncuoglu, ``A study of scheduling rules of flexible manufacturing systems: A simulation approach,'' \emph{International Journal of Production Research}, vol.~36, no.~2, pp. 527--546, 1998.

\bibitem{chryssolouris2001dynamic}
G.~Chryssolouris and V.~Subramaniam, ``Dynamic scheduling of manufacturing job shops using genetic algorithms,'' \emph{Journal of Intelligent Manufacturing}, vol.~12, no.~3, pp. 281--293, 2001.

\bibitem{li2018simulation}
M.~P. Li, P.~Sankaran, M.~E. Kuhl, A.~Ganguly, A.~Kwasinski, and R.~Ptucha, ``Simulation analysis of a deep reinforcement learning approach for task selection by autonomous material handling vehicles,'' in \emph{Winter Simulation Conference}.\hskip 1em plus 0.5em minus 0.4em\relax IEEE, 2018, pp. 1073--1083.

\bibitem{dewey2014reinforcement}
D.~Dewey, ``Reinforcement learning and the reward engineering principle,'' in \emph{2014 AAAI Spring Symposium Series}, 2014, pp. 1--4.

\bibitem{ng1999policy}
A.~Y. Ng, D.~Harada, and S.~Russell, ``Policy invariance under reward transformations: Theory and application to reward shaping,'' in \emph{International Conference on Machine Learning}, vol.~99, 1999, pp. 278--287.

\bibitem{achiam2017constrained}
J.~Achiam, D.~Held, A.~Tamar, and P.~Abbeel, ``Constrained policy optimization,'' in \emph{International Conference on Machine Learning}.\hskip 1em plus 0.5em minus 0.4em\relax PMLR, 2017, pp. 22--31.

\bibitem{altman1999constrained}
E.~Altman, \emph{Constrained Markov decision processes: Stochastic modeling}.\hskip 1em plus 0.5em minus 0.4em\relax Routledge, 1999.

\bibitem{salimans2017evolution}
T.~Salimans, J.~Ho, X.~Chen, S.~Sidor, and I.~Sutskever, ``Evolution strategies as a scalable alternative to reinforcement learning,'' \emph{arXiv preprint arXiv:1703.03864}, 2017.

\bibitem{hu2022constrained}
Z.~Hu and W.~Gong, ``Constrained evolutionary optimization based on reinforcement learning using the objective function and constraints,'' \emph{Knowledge-Based Systems}, vol. 237, p. 107731, 2022.

\bibitem{tay2008evolving}
J.~C. Tay and N.~B. Ho, ``Evolving dispatching rules using genetic programming for solving multi-objective flexible job-shop problems,'' \emph{Computers \& Industrial Engineering}, vol.~54, no.~3, pp. 453--473, 2008.

\bibitem{chen2011multiple}
C.~Chen, L.-f. Xi, B.-h. Zhou, and S.-s. Zhou, ``A multiple-criteria real-time scheduling approach for multiple-load carriers subject to lifo loading constraints,'' \emph{International Journal of Production Research}, vol.~49, no.~16, pp. 4787--4806, 2011.

\bibitem{liu2018dynamic}
S.~Liu, P.~H. Tan, E.~Kurniawan, P.~Zhang, and S.~Sun, ``Dynamic scheduling for pickup and delivery with time windows,'' in \emph{IEEE 4th World Forum on Internet of Things}.\hskip 1em plus 0.5em minus 0.4em\relax IEEE, 2018, pp. 767--770.

\bibitem{yan2018dynamic}
P.~Yan, S.~Q. Liu, T.~Sun, and K.~Ma, ``A dynamic scheduling approach for optimizing the material handling operations in a robotic cell,'' \emph{Computers \& Operations Research}, vol.~99, pp. 166--177, 2018.

\bibitem{umar2015hybrid}
U.~A. Umar, M.~Ariffin, N.~Ismail, and S.~Tang, ``Hybrid multiobjective genetic algorithms for integrated dynamic scheduling and routing of jobs and automated-guided vehicle ({AGV}) in flexible manufacturing systems ({FMS}) environment,'' \emph{The International Journal of Advanced Manufacturing Technology}, vol.~81, pp. 2123--2141, 2015.

\bibitem{wang2020proactive}
W.~Wang, Y.~Zhang, and R.~Y. Zhong, ``A proactive material handling method for {CPS} enabled shop-floor,'' \emph{Robotics and Computer-integrated Manufacturing}, vol.~61, p. 101849, 2020.

\bibitem{li2022dynamic}
Z.~Li, H.~Sang, Q.~Pan, K.~Gao, Y.~Han, and J.~Li, ``Dynamic {AGV} scheduling model with special cases in matrix production workshop,'' \emph{IEEE Transactions on Industrial Informatics}, vol.~19, no.~6, pp. 7762--7770, 2023.

\bibitem{mnih2015human}
V.~Mnih, K.~Kavukcuoglu, D.~Silver, A.~A. Rusu, J.~Veness, M.~G. Bellemare, A.~Graves, M.~Riedmiller, A.~K. Fidjeland, G.~Ostrovski, S.~Petersen, C.~Beattie, A.~Sadik, I.~Antonoglou, H.~King, D.~Kumaran, D.~Wierstra, S.~Legg, and D.~Hassabis, ``Human-level control through deep reinforcement learning,'' \emph{Nature}, vol. 518, no. 7540, pp. 529--533, 2015.

\bibitem{hu2023dorl}
C.~Hu, Z.~Wang, T.~Shu, H.~Tong, J.~Togelius, X.~Yao, and J.~Liu, ``Reinforcement learning with dual-observation for general video game playing,'' \emph{IEEE Transactions on Games}, vol.~15, no.~2, pp. 202--216, 2023.

\bibitem{silver2016mastering}
D.~Silver, J.~Schrittwieser, K.~Simonyan, I.~Antonoglou, A.~Huang, A.~Guez, T.~Hubert, L.~Baker, M.~Lai, A.~Bolton, Y.~Chen, T.~Lillicrap, F.~Hui, L.~Sifre, G.~van~den Driessche, T.~Graepel, and D.~Hassabis, ``Mastering the game of {Go} without human knowledge,'' \emph{Nature}, vol. 550, pp. 354--359, Oct. 2017.

\bibitem{haarnoja2018soft}
T.~Haarnoja, A.~Zhou, P.~Abbeel, and S.~Levine, ``Soft actor-critic: Off-policy maximum entropy deep reinforcement learning with a stochastic actor,'' in \emph{International Conference on Machine Learning}.\hskip 1em plus 0.5em minus 0.4em\relax PMLR, 2018, pp. 1861--1870.

\bibitem{chen2015reinforcement}
C.~Chen, B.~Xia, B.-h. Zhou, and L.~Xi, ``A reinforcement learning based approach for a multiple-load carrier scheduling problem,'' \emph{Journal of Intelligent Manufacturing}, vol.~26, no.~6, pp. 1233--1245, 2015.

\bibitem{xue2018reinforcement}
T.~Xue, P.~Zeng, and H.~Yu, ``A reinforcement learning method for multi-{AGV} scheduling in manufacturing,'' in \emph{IEEE International Conference on Industrial Technology}.\hskip 1em plus 0.5em minus 0.4em\relax IEEE, 2018, pp. 1557--1561.

\bibitem{kardos2021dynamic}
C.~Kardos, C.~Laflamme, V.~Gallina, and W.~Sihn, ``Dynamic scheduling in a job-shop production system with reinforcement learning,'' \emph{Procedia CIRP}, vol.~97, pp. 104--109, 2021.

\bibitem{li2025real}
Y.~Li, Q.~Wang, X.~Li, L.~Gao, L.~Fu, Y.~Yu, and W.~Zhou, ``Real-time scheduling for flexible job shop with {AGVs} using multiagent reinforcement learning and efficient action decoding,'' \emph{IEEE Transactions on Systems, Man, and Cybernetics: Systems}, vol.~55, no.~3, pp. 2120--2132, 2025.

\bibitem{bellemare2016unifying}
M.~Bellemare, S.~Srinivasan, G.~Ostrovski, T.~Schaul, D.~Saxton, and R.~Munos, ``Unifying count-based exploration and intrinsic motivation,'' \emph{Advances in Neural Information Processing Systems}, vol.~29, p. 1479–1487, 2016.

\bibitem{houthooft2016vime}
R.~Houthooft, X.~Chen, Y.~Duan, J.~Schulman, F.~De~Turck, and P.~Abbeel, ``Vime: Variational information maximizing exploration,'' \emph{Advances in Neural Information Processing Systems}, vol.~29, p. 1117–1125, 2016.

\bibitem{tessler2018reward}
\BIBentryALTinterwordspacing
C.~Tessler, D.~J. Mankowitz, and S.~Mannor, ``Reward constrained policy optimization,'' in \emph{International Conference on Learning Representations}, 2019. [Online]. Available: \url{https://openreview.net/forum?id=SkfrvsA9FX}
\BIBentrySTDinterwordspacing

\bibitem{jiang2021prioritized}
M.~Jiang, E.~Grefenstette, and T.~Rockt{\"a}schel, ``Prioritized level replay,'' in \emph{International Conference on Machine Learning}.\hskip 1em plus 0.5em minus 0.4em\relax PMLR, 2021, pp. 4940--4950.

\bibitem{yao1999evolving}
X.~Yao, ``Evolving artificial neural networks,'' \emph{Proceedings of the IEEE}, vol.~87, no.~9, pp. 1423--1447, 1999.

\bibitem{sigaud2022combining}
O.~Sigaud, ``Combining evolution and deep reinforcement learning for policy search: a survey,'' \emph{ACM Transactions on Evolutionary Learning}, pp. 1--20, 2022.

\bibitem{bai2023evolutionary}
H.~Bai, R.~Cheng, and Y.~Jin, ``Evolutionary reinforcement learning: A survey,'' \emph{Intelligent Computing}, vol.~2, p. 0025, 2023.

\bibitem{such2017deep}
F.~P. Such, V.~Madhavan, E.~Conti, J.~Lehman, K.~O. Stanley, and J.~Clune, ``Deep neuroevolution: Genetic algorithms are a competitive alternative for training deep neural networks for reinforcement learning,'' \emph{arXiv preprint arXiv:1712.06567}, 2017.

\bibitem{conti2018improving}
E.~Conti, V.~Madhavan, F.~Petroski~Such, J.~Lehman, K.~Stanley, and J.~Clune, ``Improving exploration in evolution strategies for deep reinforcement learning via a population of novelty-seeking agents,'' \emph{Advances in Neural Information Processing Systems}, vol.~31, 2018.

\bibitem{yang2022evolutionary}
P.~Yang, H.~Zhang, Y.~Yu, M.~Li, and K.~Tang, ``Evolutionary reinforcement learning via cooperative coevolutionary negatively correlated search,'' \emph{Swarm and Evolutionary Computation}, vol.~68, p. 100974, 2022.

\bibitem{khadka2018evolution}
S.~Khadka and K.~Tumer, ``Evolution-guided policy gradient in reinforcement learning,'' in \emph{International Conference on Neural Information Processing Systems}.\hskip 1em plus 0.5em minus 0.4em\relax Curran Associates Inc., 2018, p. 1196–1208.

\bibitem{hu2023ecrl}
C.~Hu, J.~Pei, J.~Liu, and X.~Yao, ``Evolving constrained reinforcement learning policy,'' in \emph{International Joint Conference on Neural Networks}, 2023, pp. 1--8.

\bibitem{agrawal1995sample}
R.~Agrawal, ``Sample mean based index policies by o(log n) regret for the multi-armed bandit problem,'' \emph{Advances in Applied Probability}, vol.~27, no.~4, pp. 1054--1078, 1995.

\bibitem{Hao2019Upper}
B.~Hao, Y.~Abbasi-Yadkori, Z.~Wen, and G.~Cheng, ``Bootstrapping upper confidence bound,'' in \emph{Proceedings of the 33rd International Conference on Neural Information Processing Systems}, 2019.

\bibitem{watanabe2004evolutionary}
K.~Watanabe, M.~Hashem, K.~Watanabe, and M.~Hashem, ``Evolutionary optimization of constrained problems,'' \emph{Evolutionary Computations: New Algorithms and their Applications to Evolutionary Robots}, pp. 53--64, 2004.

\bibitem{runarsson2000stochastic}
T.~P. Runarsson and X.~Yao, ``Stochastic ranking for constrained evolutionary optimization,'' \emph{IEEE Transactions on Evolutionary Computation}, vol.~4, no.~3, pp. 284--294, 2000.

\bibitem{runarsson2003constrained}
T.~Runarsson and X.~Yao, ``Constrained evolutionary optimization,'' in \emph{Evolutionary Optimization}.\hskip 1em plus 0.5em minus 0.4em\relax Springer, 2003, pp. 87--113.

\bibitem{tang2009memetic}
K.~Tang, Y.~Mei, and X.~Yao, ``Memetic algorithm with extended neighborhood search for capacitated arc routing problems,'' \emph{IEEE Transactions on Evolutionary Computation}, vol.~13, no.~5, pp. 1151--1166, 2009.

\bibitem{wierstra2014natural}
D.~Wierstra, T.~Schaul, T.~Glasmachers, Y.~Sun, J.~Peters, and J.~Schmidhuber, ``Natural evolution strategies,'' \emph{The Journal of Machine Learning Research}, vol.~15, no.~1, pp. 949--980, 2014.

\bibitem{nesterov2017random}
Y.~Nesterov and V.~Spokoiny, ``Random gradient-free minimization of convex functions,'' \emph{Foundations of Computational Mathematics}, vol.~17, no.~2, pp. 527--566, 2017.

\bibitem{choromanski2019complexity}
K.~Choromanski, A.~Pacchiano, J.~Parker-Holder, Y.~Tang, and V.~Sindhwani, ``From complexity to simplicity: Adaptive {ES}-active subspaces for blackbox optimization,'' in \emph{Proceedings of the 33rd International Conference on Neural Information Processing Systems}, 2019, pp. 10\,299--10\,309.

\bibitem{liu2020self}
F.-Y. Liu, Z.-N. Li, and C.~Qian, ``Self-guided evolution strategies with historical estimated gradients.'' in \emph{the International Joint Conference on Artificial Intelligence}, 2020, pp. 1474--1480.

\bibitem{tamar2012policy}
A.~Tamar, D.~Di~Castro, and S.~Mannor, ``Policy gradients with variance related risk criteria,'' in \emph{International Conference on Machine Learning}, 2012, pp. 387--396.

\bibitem{schulman2017proximal}
J.~Schulman, F.~Wolski, P.~Dhariwal, A.~Radford, and O.~Klimov, ``Proximal policy optimization algorithms,'' \emph{arXiv preprint arXiv:1707.06347}, 2017.

\bibitem{weng2021tianshou}
J.~Weng, H.~Chen, D.~Yan, K.~You, A.~Duburcq, M.~Zhang, Y.~Su, H.~Su, and J.~Zhu, ``Tianshou: A highly modularized deep reinforcement learning library,'' \emph{Journal of Machine Learning Research}, vol.~23, no. 267, pp. 1--6, 2022.

\bibitem{OpenAIGym}
G.~Brockman, V.~Cheung, L.~Pettersson, J.~Schneider, J.~Schulman, J.~Tang, and W.~Zaremba, ``Open{AI} gym,'' \emph{arXiv preprint arXiv:1606.01540}, 2016.

\end{thebibliography}

\onecolumn
\clearpage
\begin{center}
    \Huge{Robust Dynamic Material Handling via Adaptive Constrained Evolutionary Reinforcement Learning}
\end{center}
\appendix

\setcounter{figure}{0}
\setcounter{table}{0}

\subsection{Supplementary experimental results of comparisons between ACERL and other algorithms}\label{sec:suptrain}
This section provides training curves, supplemented to Section V-D of the main manuscript.~Fig. \ref{fig:tcs} shows training curves of ACERL, MAPPO, RCPOM, LSAC, AMAPPO, ARCPOM, ALSAC, SAC, PPO, ASAC and APPO on training instances (DMH-01 to DMH-08). Each algorithm is trained with five different seeds and tested for 30 times independently. Fig. \ref{fig:tcs} demonstrate that ACERL (red) achieves faster and more stable convergence. During training, ACERL promptly satisfies the constraint at an early stage, whereas other algorithms 
struggle to confine the behaviours within the threshold.

\begin{figure}[ht]
  \centering
  \subfigure[Training curves on DMH-01 (Makespan).]{
		\centering

 		\includegraphics[width=0.4\columnwidth]{figures/training_curve/test-makespan-0.png}
		\label{DMH-01-m}\label{fig:makespan_0}
 }	
   \subfigure[Training curves on DMH-01 (Tardiness).]{

		\centering
		\includegraphics[width=0.4\columnwidth]{figures/training_curve/test-tardiness-0.png}
		\label{DMH-01-t.}\label{fig:tardiness_0}
	}

  \subfigure[Training curves on DMH-02 (Makespan).]{
		\centering

 		\includegraphics[width=0.4\columnwidth]{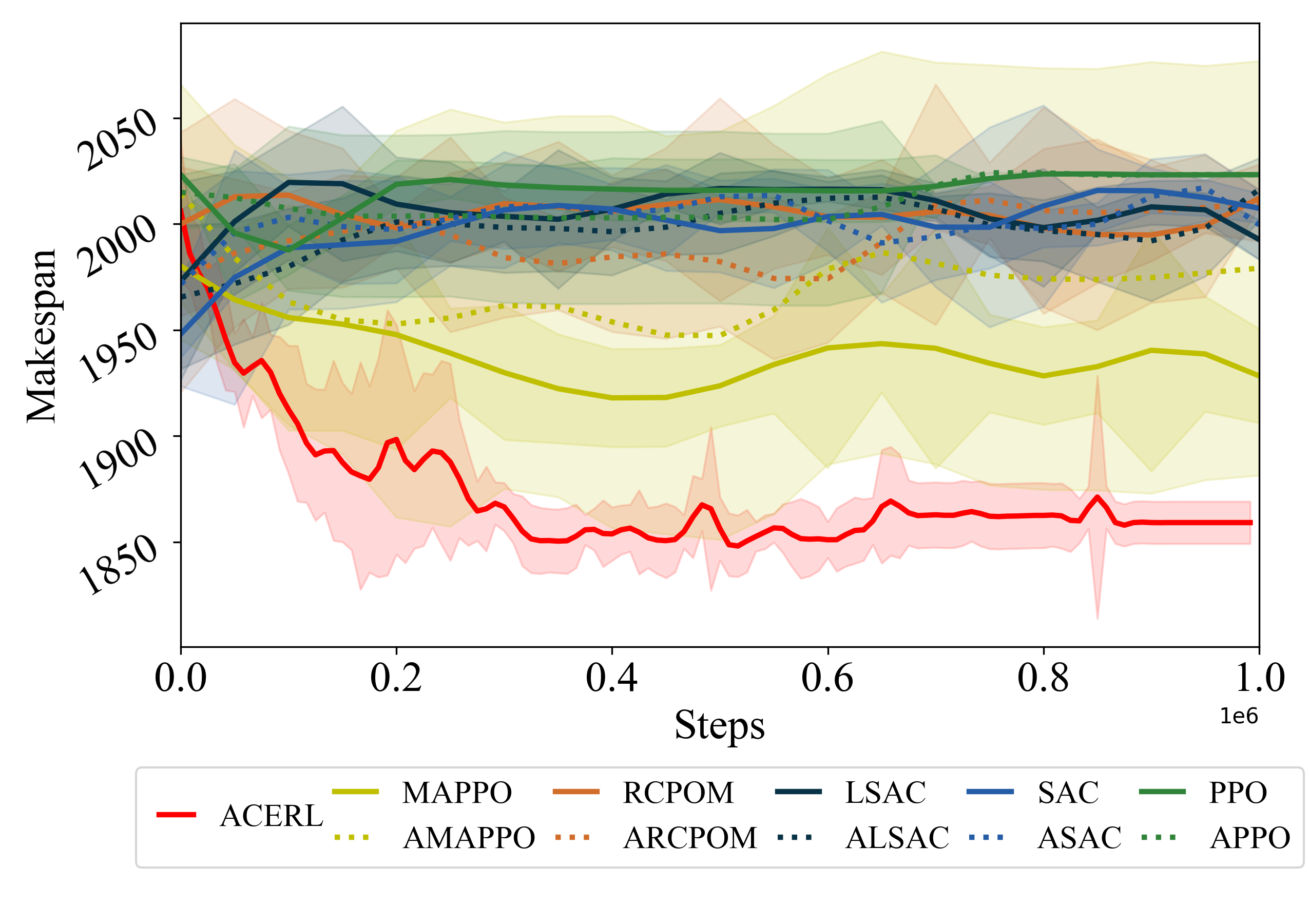}
		\label{DMH-02-m}\label{fig:makespan_1}
 }	
   \subfigure[Training curves on DMH-02 (Tardiness).]{

		\centering
		\includegraphics[width=0.4\columnwidth]{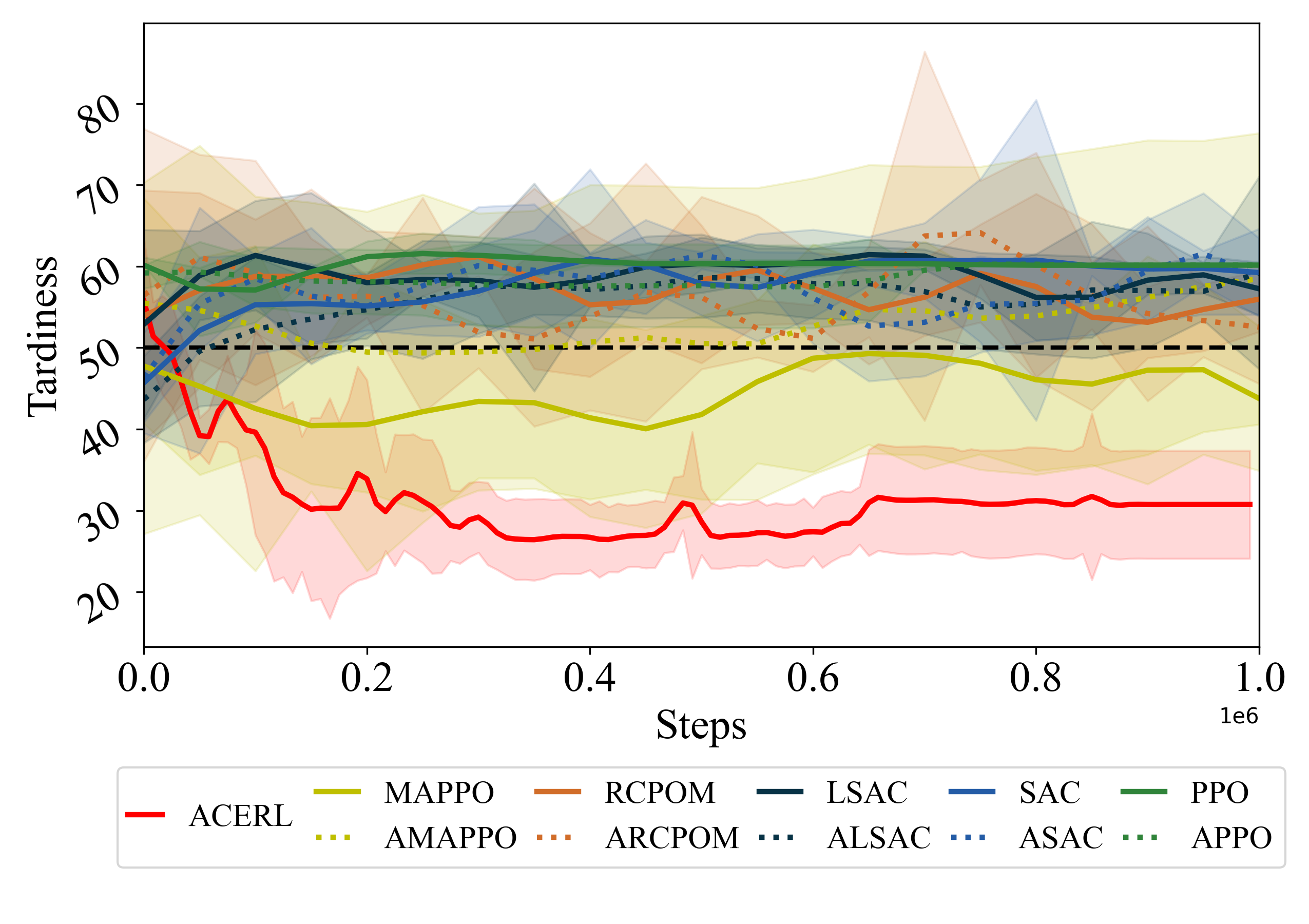}
		\label{DMH-02-t.}\label{fig:tardiness_1}
	}

\end{figure}

\begin{figure}[ht]
  \centering

  \subfigure[Training curves on DMH-03 (Makespan).]{
		\centering

 		\includegraphics[width=0.4\columnwidth]{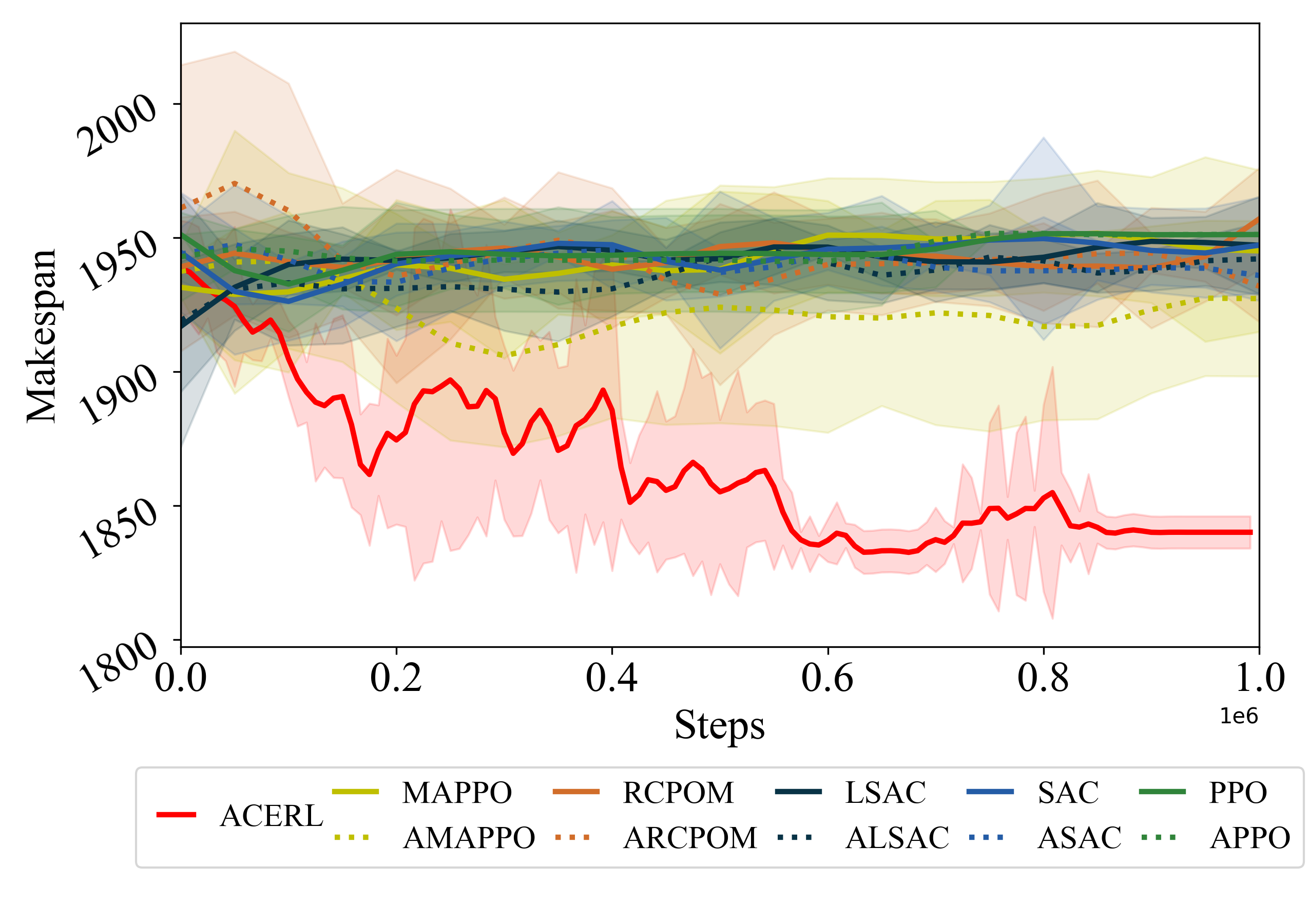}
		\label{DMH-03-m}\label{fig:makespan_2}
 }	
   \subfigure[Training curves on DMH-03 (Tardiness).]{
		\centering
		\includegraphics[width=0.4\columnwidth]{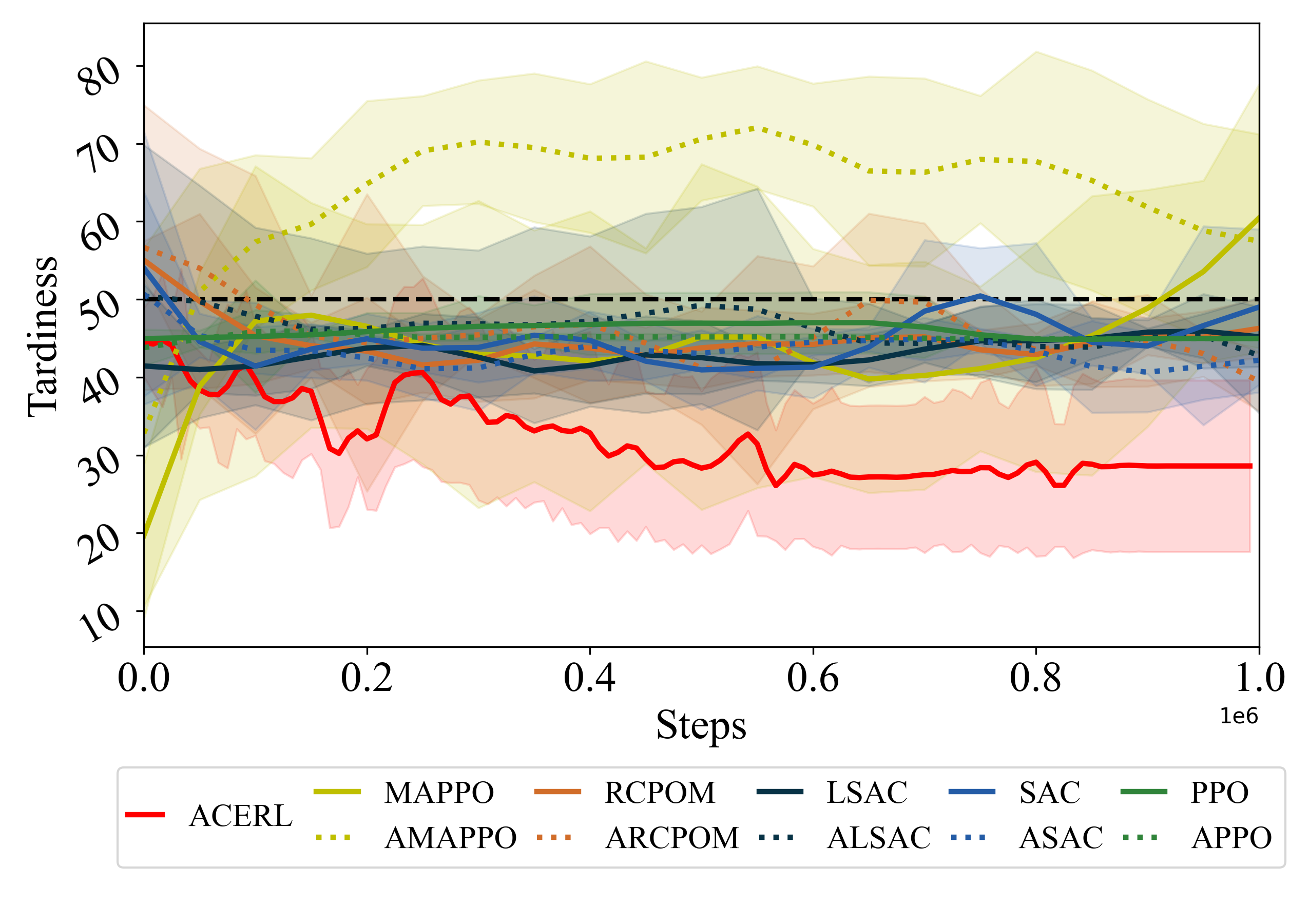}
		\label{DMH-03-t.}\label{fig:tardiness_2}
	}

   \subfigure[Training curves on DMH-04 (Makespan).]{
		\centering

 		\includegraphics[width=0.4\columnwidth]{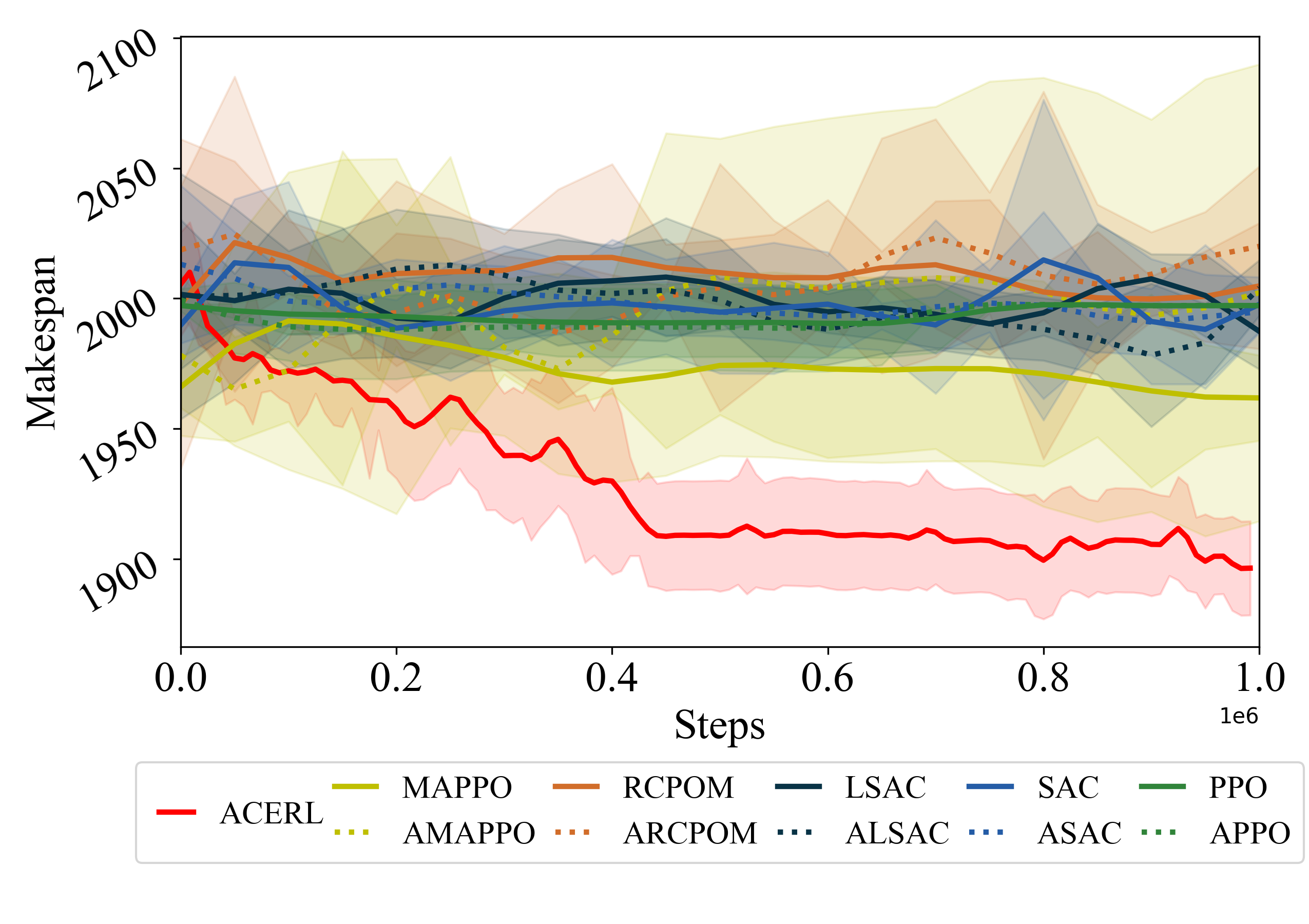}
		\label{DMH-04-m}\label{fig:makespan_3}
 }	
   \subfigure[Training curves on DMH-04 (Tardiness).]{
		\centering
		\includegraphics[width=0.4\columnwidth]{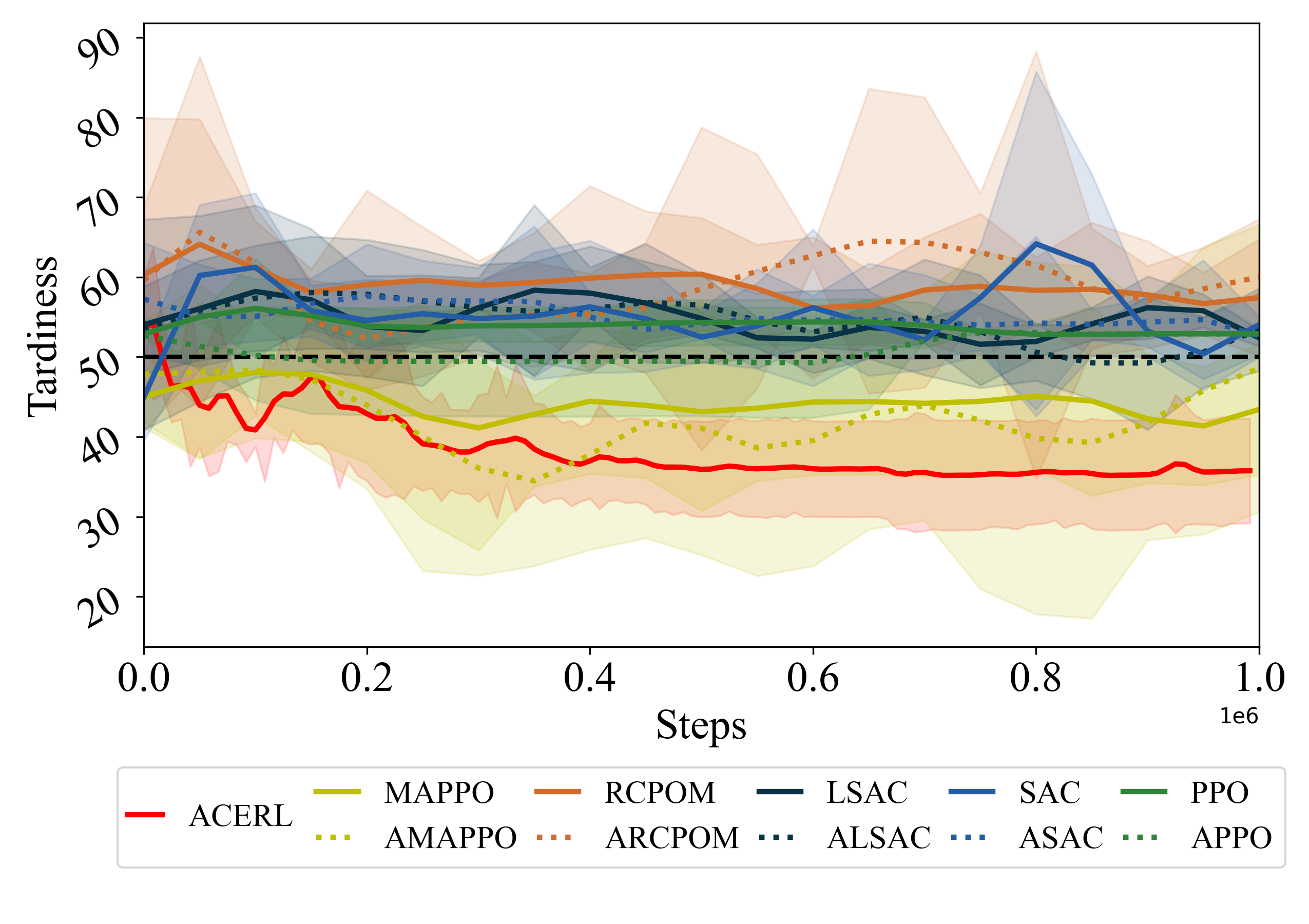}
		\label{DMH-04-m.}\label{fig:tardiness_3}
	}

     \subfigure[Training curves on DMH-05 (Makespan).]{
		\centering

 		\includegraphics[width=0.4\columnwidth]{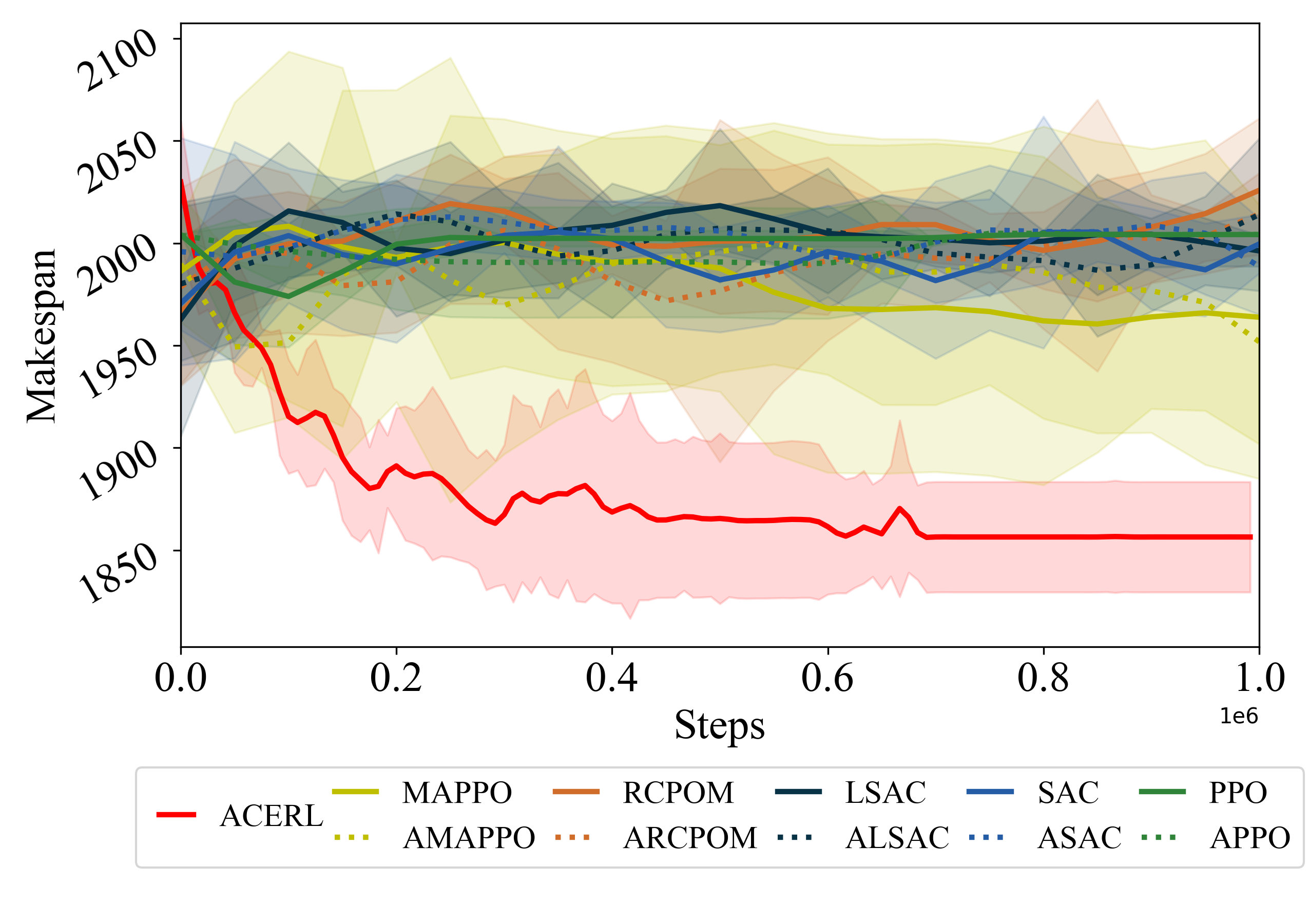}
		\label{DMH-05-m}\label{fig:makespan_4}
 }	
   \subfigure[Training curves on DMH-05 (Tardiness).]{
		\centering
		\includegraphics[width=0.4\columnwidth]{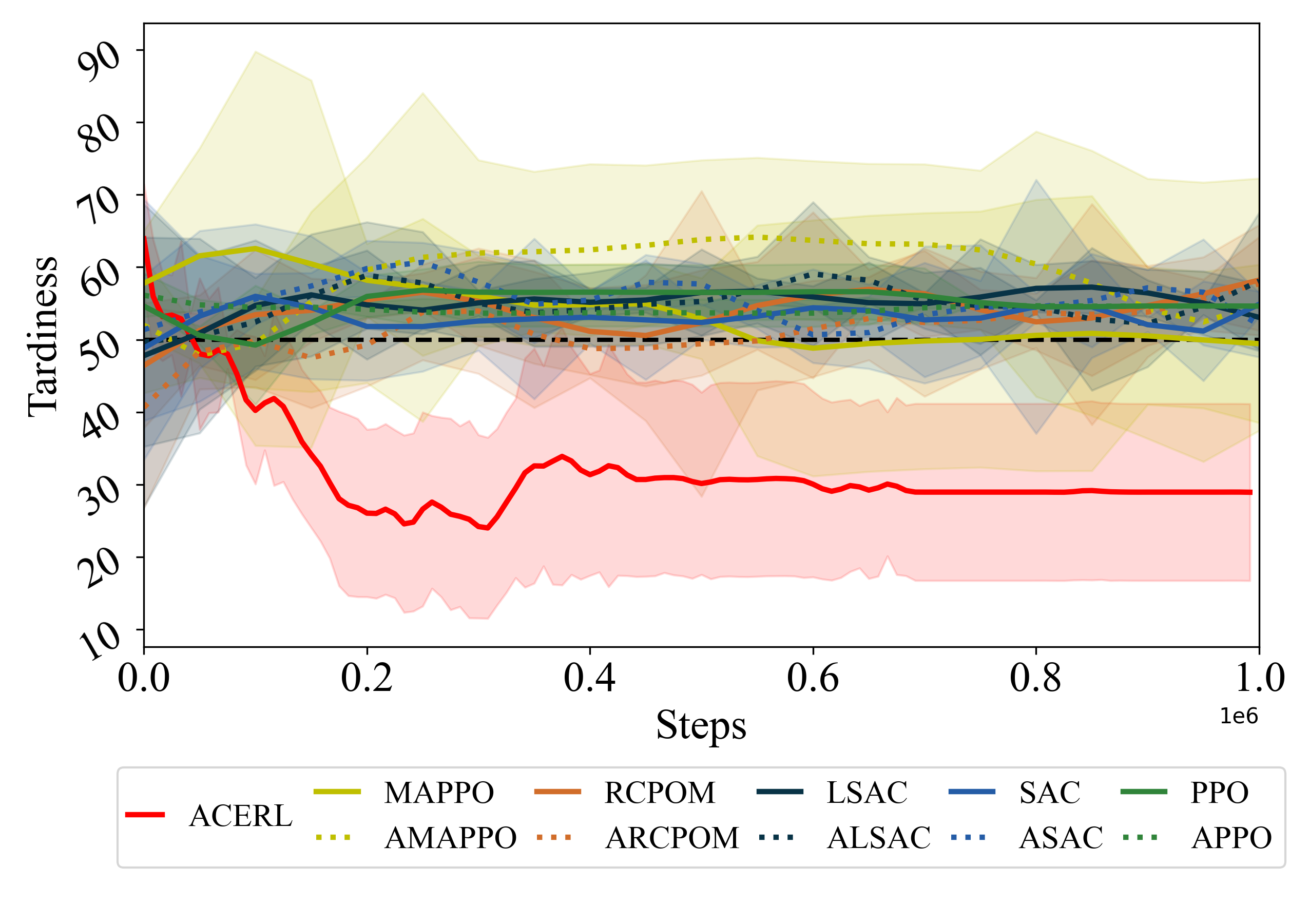}
		\label{DMH-05-t}\label{fig:tardiness_4}
	}

\end{figure}

\begin{figure}[ht]
  \centering

   \subfigure[Training curves on DMH-06 (Makespan).]{
		\centering

 		\includegraphics[width=0.4\columnwidth]{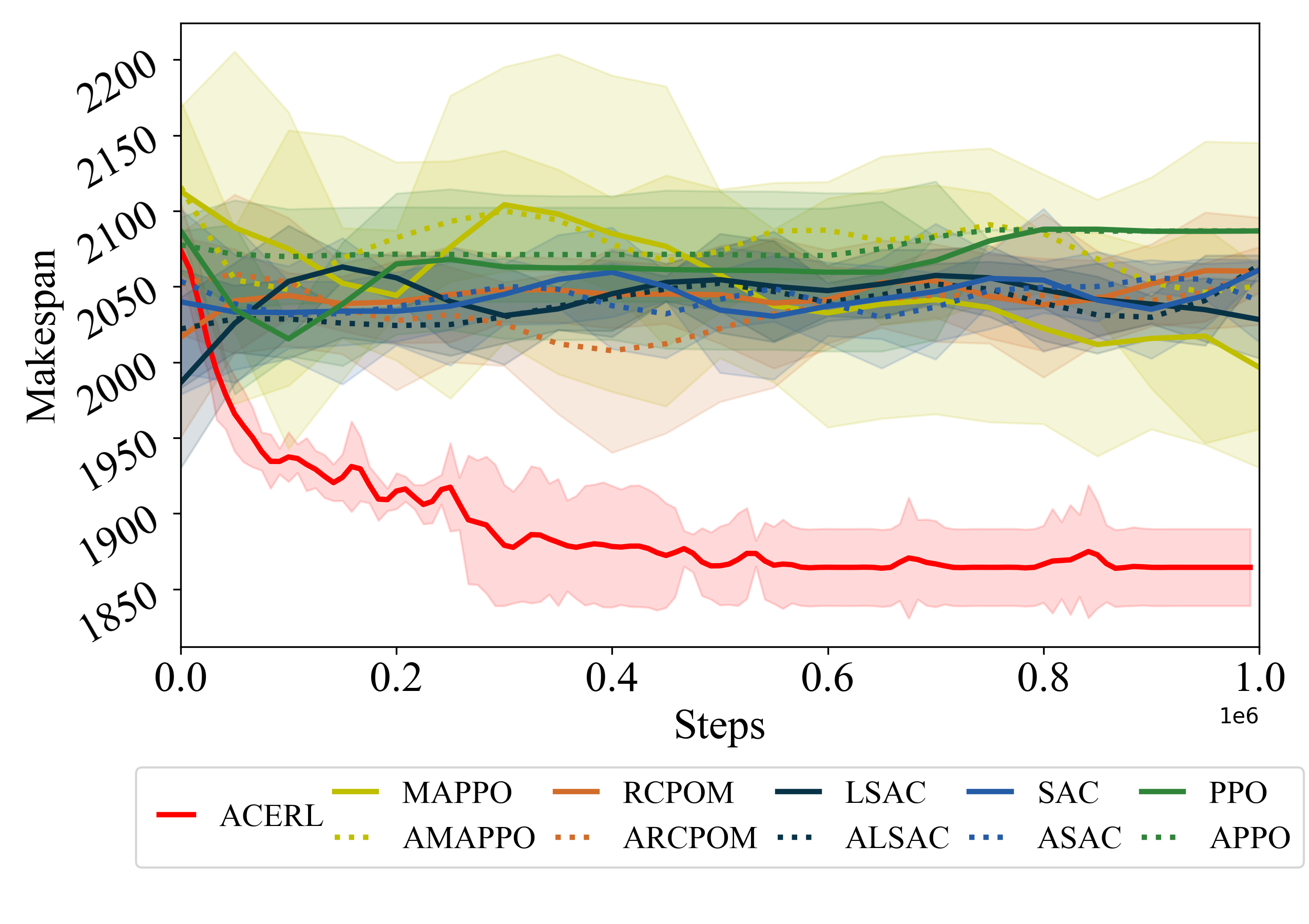}
		\label{DMH-06-m}\label{fig:makespan_5}
 }	
   \subfigure[Training curves on DMH-06 (Tardiness).]{
		\centering
		\includegraphics[width=0.4\columnwidth]{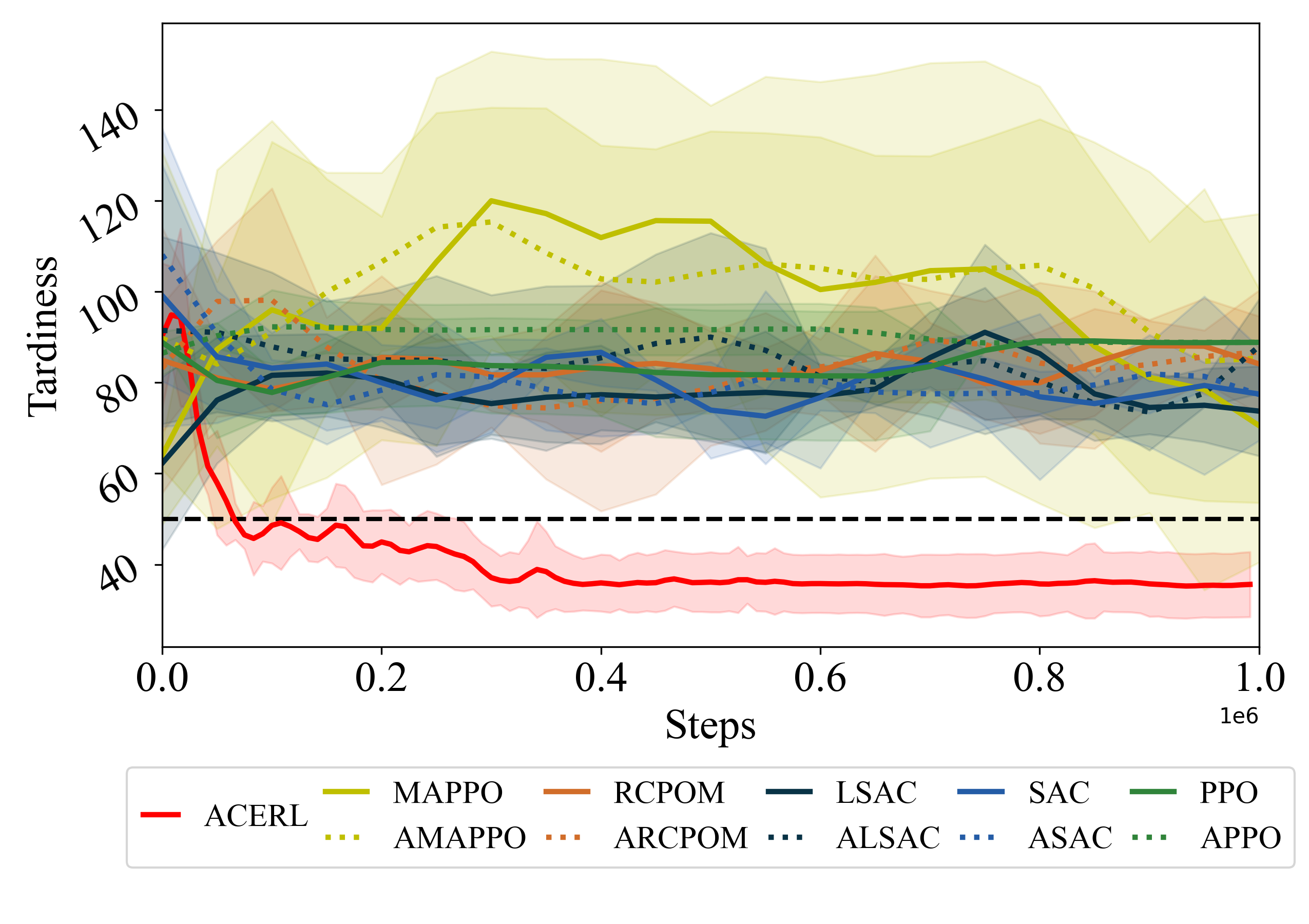}
		\label{DMH-06-t}\label{fig:tardiness_5}
	}

   \subfigure[Training curves on DMH-07 (Makespan).]{
		\centering

 		\includegraphics[width=0.4\columnwidth]{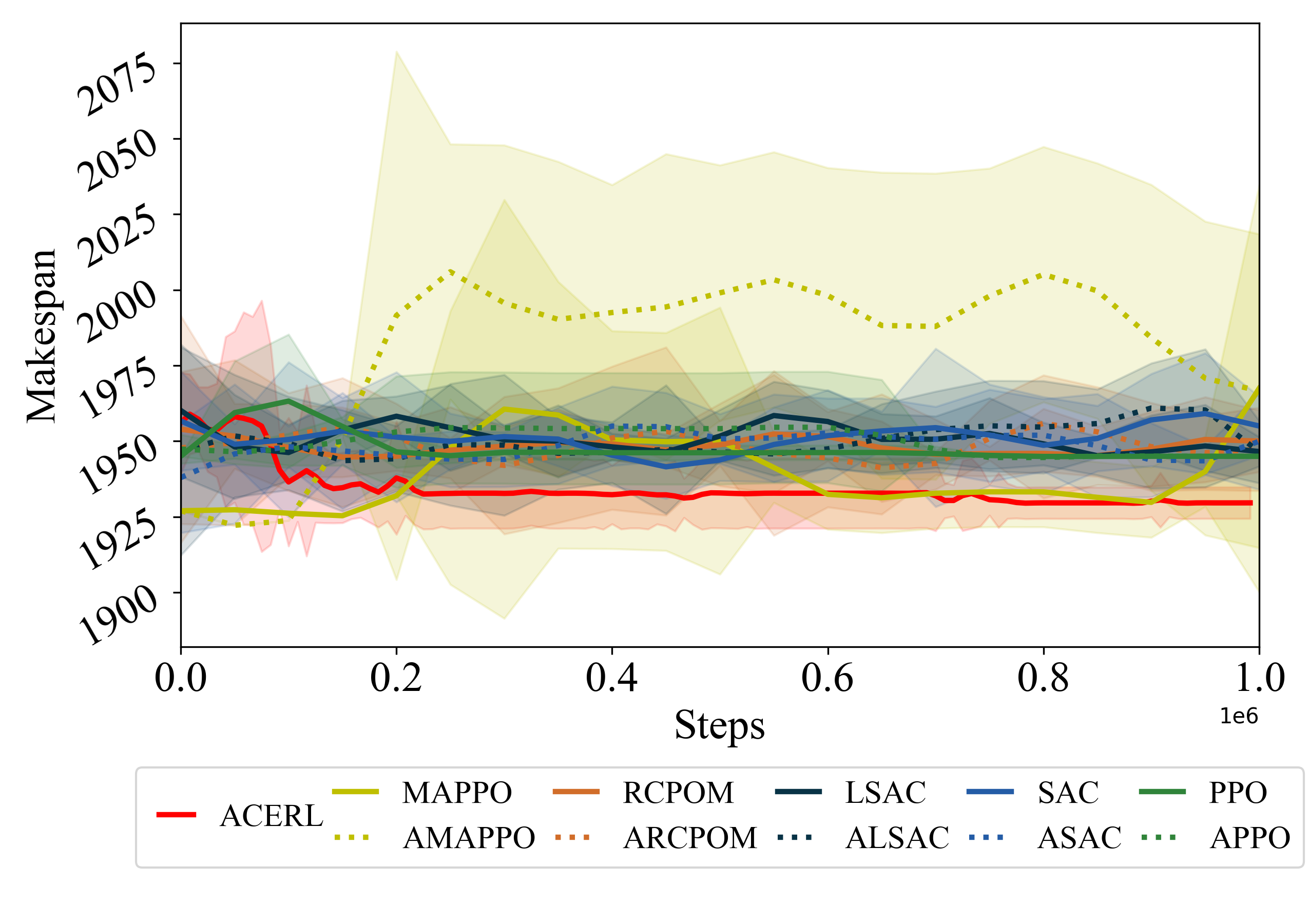}
		\label{DMH-07-m}\label{fig:makespan_6}
 }	
   \subfigure[Training curves on DMH-07 (Tardiness).]{
		\centering
		\includegraphics[width=0.4\columnwidth]{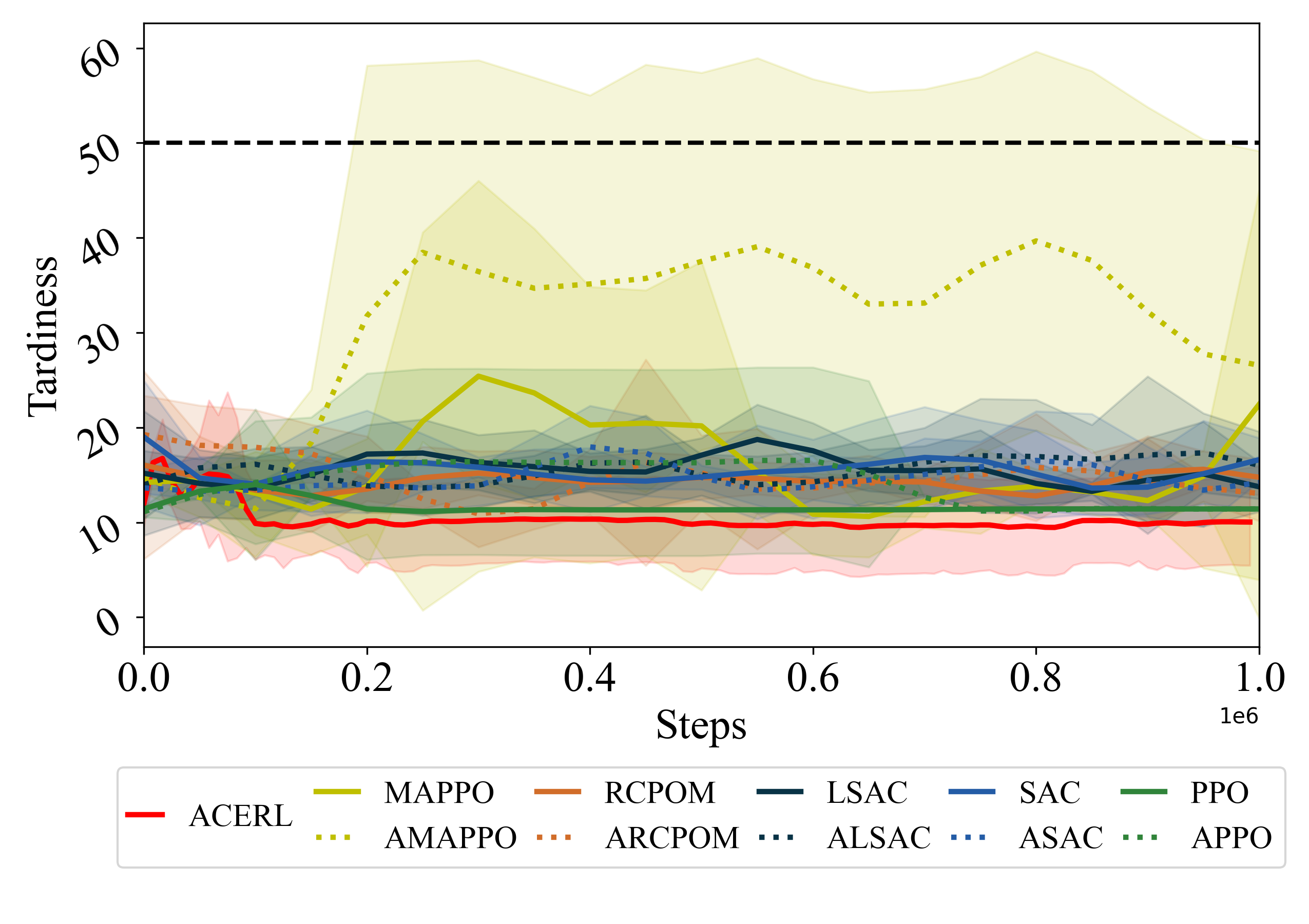}
		\label{DMH-07-t}\label{fig:tardiness_6}
	}

    \subfigure[Training curves on DMH-08 (Makespan).]{
		\centering

 		\includegraphics[width=0.4\columnwidth]{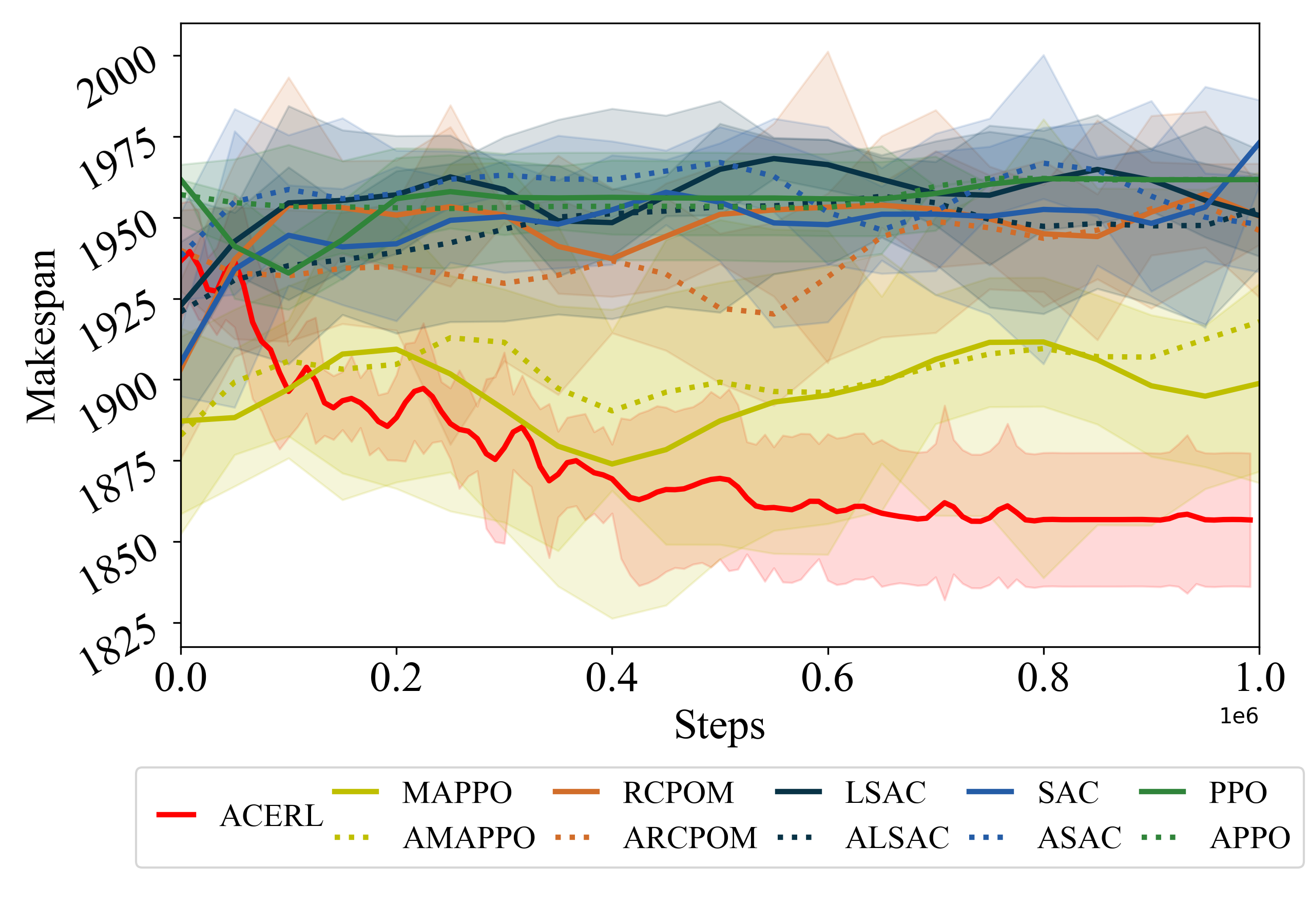}
		\label{DMH-08-m}\label{fig:makespan_7}
 }	
   \subfigure[Training curves on DMH-08 (Tardiness).]{
		\centering
		\includegraphics[width=0.4\columnwidth]{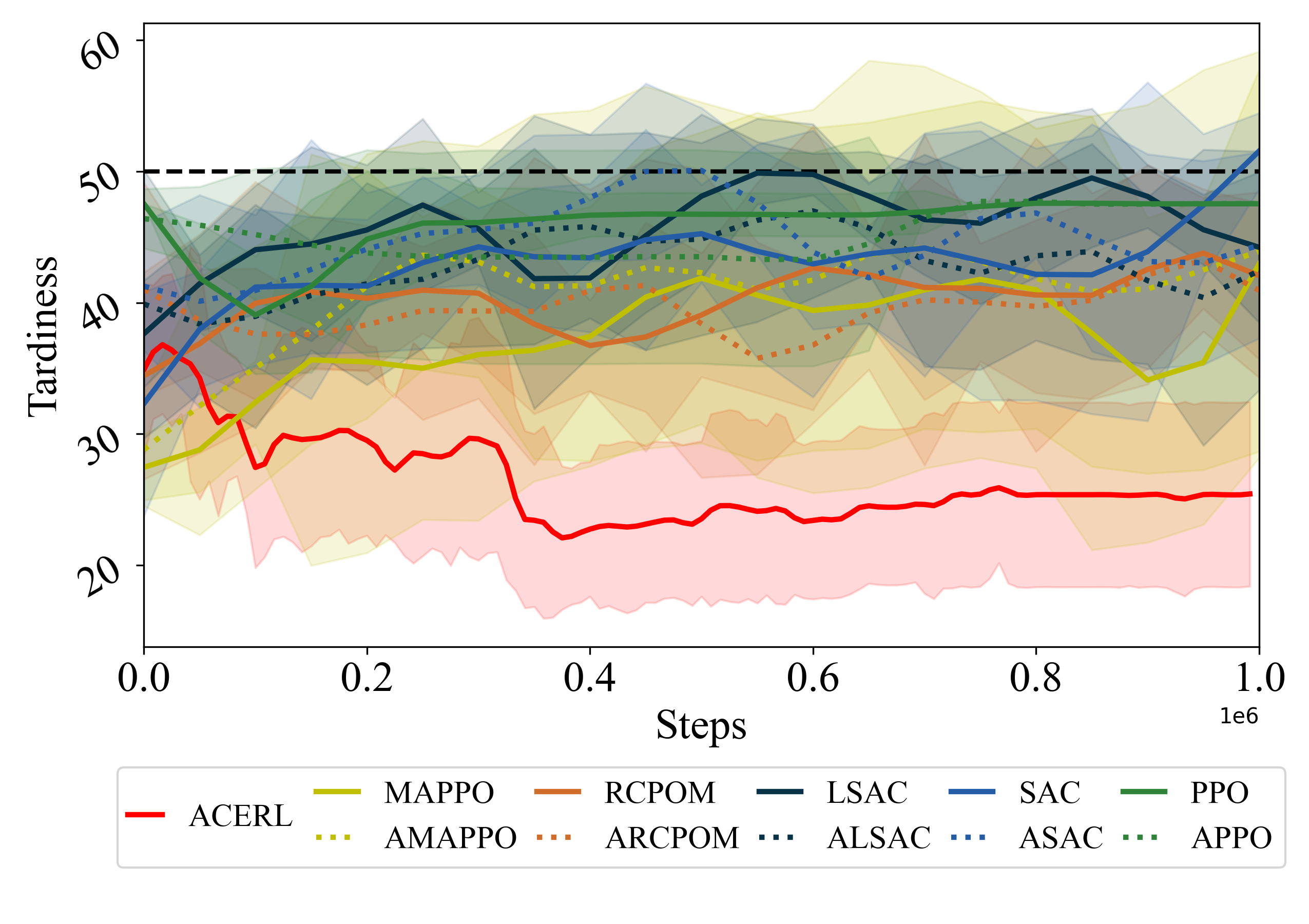}
		\label{DMH-08-t}\label{fig:tardiness_7}
	}

	\caption{\label{fig:tcs}Training curves of ACERL, RCPOM, LSAC, ARCPOM, ALSAC, SAC, PPO, ASAC and APPO. ACERL (red) performs the best.}
\end{figure}

\newpage
\onecolumn

\subsection{Supplementary experimental results of noised datasets}\label{sec:supnoise} 
Tables \ref{tab:noised+10}, \ref{tab:noised+15}, \ref{tab:noised+20}, \ref{tab:noised+25} and \ref{tab:noised+30} are results of experimental study presented in Section V-D of the main manuscript on noised datasets DMH$\pm 10$, DMH$\pm 15$, DMH$\pm 20$, DMH$\pm 25$ and DMH$\pm30$, respectively. Each table details the comparisons between ACERL and other state-of-the-arts and baselines. The results on the noised datasets denote that ACERL achieves good convergence on makespan and superior constraint satisfaction.

\begin{table*}[hb]
    \centering
        \caption{
        Performance on DMH$\pm$10: ``$M/C (P)$'' indicates the average normalised makespan, tardiness and percentage of constraint satisfaction over 30 independent trials of five different seeds.
        Bold number indicates the best value in the corresponding column. ``+'',``$\approx$'' and ``-'' indicate the policy performs statistically better/similar/worse than ACERL policy. The number of policies that are ``better'', ``similar'' and ``worse'' than ACERL in terms of makespan and tardiness on each instance is summarised in the bottom row. Horizontal rules in the table separate different groups of algorithms.}
    \resizebox{\textwidth}{!}{
      \setlength{\tabcolsep}{2pt}
    \begin{tabular}{c|c|c|c|c|c|c|c|c|c}
\toprule
\multirow{2}{*}{Algorithm}  &DMH$\pm10$-01 & DMH$\pm10$-02 & DMH$\pm10$-03 & DMH$\pm10$-04 & DMH$\pm10$-05 & DMH$\pm10$-06 & DMH$\pm10$-07 & DMH$\pm10$-08  & \multirow{2}{*}{$M/C (P)$} \\

& $F_m$/$F_t$ & $F_m$/$F_t$& $F_m$/$F_t$& $F_m$/$F_t$& $F_m$/$F_t$& $F_m$/$F_t$& $F_m$/$F_t$& $F_m$/$F_t$&\\
\midrule
ACERL&1867.2/48.5 & 1936.8/42.6 & \textbf{1900.5}/\textbf{34.4} & 1999.5/46.6 & \textbf{1891.8}/\textbf{34.4} & \textbf{1876.0}/\textbf{33.3} & 1961.0/\textbf{10.8} & \textbf{1886.0}/32.5&0.90/0.94 (80\%)\\
\midrule
MAPPO &1859.0+/55.4- & \textbf{1912.6}$\approx$/35.8+ & 1950.2-/53.3- & \textbf{1950.2}+/34.3+ & 1977.0-/45.2- & 2017.2-/93.3- & 1982.2-/23.7- & 1888.2$\approx$/29.3+&0.79/0.83 (65\%)\\
 RCPOM &1857.6+/43.3+ & 1942.6-/47.8- & 1924.3-/53.3- & 1960.3+/41.5+ & 1945.1-/42.0- & 1935.3-/57.3- & 1962.2-/13.0- & 1899.8$\approx$/32.9$\approx$&0.85/0.88 (64\%)\\
LSAC &1906.5-/52.8- & 1975.3-/47.2$\approx$ & 1932.7-/49.6- & 2013.5$\approx$/59.7- & 1980.4-/51.1- & 1995.4-/69.6- & 1964.2$\approx$/13.7- & 1923.8-/38.8-&0.71/0.81 (57\%)\\
AMAPPO &1837.6+/58.6- & 1917.6$\approx$/\textbf{31.6}+ & 1954.8-/44.9- & 1974.2+/39.9+ & 1978.4-/47.0- & 2033.4-/103.2- & 1953.8$\approx$/17.5- & 1925.6-/39.4$\approx$&0.78/0.82 (67\%)\\
ARCPOM &1875.2-/59.5- & 1932.3$\approx$/46.6- & 1922.9-/60.2- & 1965.2+/44.3$\approx$ & 1938.8-/42.2- & 1976.4-/77.9- & 1965.4$\approx$/17.6- & 1899.3-/35.1-&0.83/0.81 (57\%)\\
ALSAC &1875.1$\approx$/\textbf{39.4}+ & 1952.8-/45.6$\approx$ & 1913.9$\approx$/46.3- & 1993.8$\approx$/50.7$\approx$ & 1982.4-/54.8- & 1970.3-/66.0- & 1976.4-/17.9- & 1898.4$\approx$/37.5$\approx$&0.77/0.84 (63\%)\\
\midrule
SAC &1900.8-/53.2- & 1964.8-/47.6- & 1944.3-/46.1- & 1999.6$\approx$/51.4$\approx$ & 1989.4-/55.6- & 2006.8-/77.6- & 1950.4$\approx$/13.4- & 1922.3-/35.6$\approx$&0.73/0.82 (59\%)\\
PPO &1871.2$\approx$/50.4$\approx$ & 1947.8-/46.7$\approx$ & 1927.6-/46.4- & 1982.4+/53.0- & 1965.6-/50.4- & 1981.5-/71.2- & 1976.9-/16.2- & 1913.9-/32.9$\approx$&0.77/0.83 (60\%)\\
ASAC &1905.9-/58.8- & 1987.9-/50.4- & 1946.8-/47.0- & 1995.5$\approx$/52.0- & 1989.5-/54.2- & 2031.7-/90.3- & 1952.8$\approx$/15.5- & 1919.8-/35.4$\approx$&0.70/0.78 (55\%)\\
APPO &1872.6$\approx$/58.5- & 1925.8$\approx$/42.1$\approx$ & 1926.8-/54.7- & 1963.7+/41.6+ & 1956.5-/52.2- & 1997.8-/97.8- & 1968.0-/17.0- & 1888.8$\approx$/35.8-&0.82/0.79 (58\%)\\
\midrule
MIX &1923.1-/57.9- & 2009.3-/58.0- & 1940.9-/45.4- & 2017.9$\approx$/55.5$\approx$ & 2027.1-/64.5- & 2024.2-/69.2- & 1971.0$\approx$/16.6- & 1958.7-/42.5-&0.60/0.76 (54\%)\\
FCFS &2062.2-/112.8- & 2112.1-/92.2- & 2008.7-/56.3- & 2141.8-/99.1- & 2136.9-/86.4- & 2115.3-/72.1- & 2005.1-/17.9- & 1934.5-/\textbf{28.9}$\approx$&0.24/0.54 (37\%)\\
EDD &1947.5-/59.3- & 2025.8-/45.2$\approx$ & 1985.4-/51.6- & 2000.3$\approx$/43.4$\approx$ & 2067.3-/50.6- & 2036.1-/51.3- & \textbf{1932.2}$\approx$/12.5- & 2009.9-/44.9-&0.55/0.84 (65\%)\\
NVF &\textbf{1823.8}+/46.7$\approx$ & 1923.1$\approx$/45.3$\approx$ & 1936.7-/41.8- & 1996.5$\approx$/61.1- & 1958.6-/50.9- & 1933.5-/68.0- & 2010.0-/22.2- & 1920.1-/33.4$\approx$&0.77/0.82 (61\%)\\
STD &1861.3$\approx$/51.6$\approx$ & 1939.2$\approx$/52.1- & 1915.9$\approx$/47.1- & 1959.1+/\textbf{33.8}+ & 1963.9-/50.6- & 1960.6-/70.4- & 1941.1$\approx$/16.6- & 1902.7$\approx$/39.0-&0.86/0.83 (62\%)\\
Random &2096.2-/126.3- & 2110.4-/105.5- & 2097.5-/145.1- & 2142.0-/126.7- & 2158.1-/117.8- & 2167.4-/145.5- & 2043.5-/57.2- & 2072.5-/98.2-&0.00/0.00 (10\%)\\
\midrule
& (4/4/8) / (2/3/11)& (0/6/10) / (2/6/8)& (0/2/14) / (0/0/16)& (7/7/2) / (5/5/6)& (0/0/16) / (0/0/16)& (0/0/16) / (0/0/16)& (0/8/8) / (0/0/16)& (0/5/11) / (1/8/7)\\

\bottomrule
    \end{tabular}
    }

    \label{tab:noised+10}
\end{table*}

\begin{table*}[htbp]
    \centering
        \caption{Performance on DMH$\pm$15: ``$M/C (P)$'' indicates the average normalised makespan, tardiness and percentage of constraint satisfaction over 30 independent trials of five different seeds.
        Bold number indicates the best value in the corresponding column. ``+'',``$\approx$'' and ``-'' indicate the policy performs statistically better/similar/worse than ACERL policy. The number of policies that are ``better'', ``similar'' and ``worse'' than ACERL in terms of makespan and tardiness on each instance is summarised in the bottom row. Horizontal rules in the table separate different groups of algorithms.}
    \resizebox{\textwidth}{!}{
      \setlength{\tabcolsep}{2pt}
    \begin{tabular}{c|c|c|c|c|c|c|c|c|c}
\toprule
\multirow{2}{*}{Algorithm}  &DMH$\pm15$-01 & DMH$\pm15$-02 & DMH$\pm15$-03 & DMH$\pm15$-04 & DMH$\pm15$-05 & DMH$\pm15$-06 & DMH$\pm15$-07 & DMH$\pm15$-08  & \multirow{2}{*}{$M/C (P)$} \\

& $F_m$/$F_t$ & $F_m$/$F_t$& $F_m$/$F_t$& $F_m$/$F_t$& $F_m$/$F_t$& $F_m$/$F_t$& $F_m$/$F_t$& $F_m$/$F_t$&\\
\midrule
ACERL&\textbf{1821.4}/\textbf{26.0} & \textbf{1906.2}/35.5 & \textbf{1873.9}/30.3 & 1977.6/44.1 & \textbf{1901.2}/\textbf{25.3} & \textbf{1875.8}/\textbf{39.3} & 1974.8/17.0 & 1905.0/33.0&0.94/0.94 (95\%)\\
\midrule
MAPPO &1853.2-/62.7- & 1919.8-/38.7- & 1932.8-/31.7$\approx$ & \textbf{1955.4}+/\textbf{34.8}+ & 1952.2-/39.6- & 1970.0-/59.7$\approx$ & \textbf{1946.2}+/11.3+ & 1875.0+/31.7+&0.87/0.86 (85\%)\\
RCPOM &1889.0-/54.0- & 1923.7-/38.0$\approx$ & 1915.2-/45.6- & 1963.3+/38.2+ & 1932.7-/36.2- & 1932.2-/46.5- & 1971.0$\approx$/\textbf{10.1}+ & 1890.9+/32.4$\approx$&0.85/0.88 (82\%)\\
LSAC &1915.6-/62.7- & 1969.3-/43.0$\approx$ & 1947.6-/54.1- & 2004.3-/55.3- & 1954.8-/40.7- & 2012.6-/71.2- & 1996.4-/20.0$\approx$ & 1919.2$\approx$/39.7-&0.67/0.74 (56\%)\\
AMAPPO &1894.8-/70.8- & 1919.0-/37.4- & 1937.8-/38.6- & 1978.4$\approx$/39.8+ & 1961.8-/40.5- & 2010.8-/71.6- & 1952.6+/14.5$\approx$ & 1897.4+/41.3-&0.79/0.79 (70\%)\\
ARCPOM &1896.5-/72.0- & 1936.6-/44.4- & 1913.1-/40.2- & 1973.7$\approx$/43.9$\approx$ & 1925.5$\approx$/31.1- & 1974.2-/79.0- & 1996.3-/15.1$\approx$ & 1936.9-/47.6-&0.76/0.76 (61\%)\\
ALSAC &1871.1-/59.3- & 1939.2-/45.6- & 1928.4-/45.3- & 1982.7$\approx$/48.0- & 1940.1-/39.7- & 1982.8-/59.4- & 2002.3-/20.1- & 1907.3$\approx$/39.3-&0.76/0.78 (66\%)\\
\midrule
SAC &1899.0-/56.9- & 1958.8-/45.5- & 1929.5-/47.3- & 1998.7-/48.8- & 1979.2-/46.8- & 1996.9-/66.6- & 1987.5$\approx$/21.4- & 1921.6$\approx$/38.6-&0.70/0.76 (66\%)\\
PPO &1886.4-/54.8- & 1935.7-/42.2- & 1919.8-/49.9- & 1991.9-/53.1- & 1958.4-/43.6- & 1961.5-/58.5- & 1992.4-/20.8- & 1914.1$\approx$/35.3$\approx$&0.76/0.79 (63\%)\\
ASAC &1904.2-/67.6- & 1958.4-/43.9- & 1930.3-/46.6- & 1989.1$\approx$/45.2$\approx$ & 1985.1-/46.9- & 2015.8-/75.4- & 1996.0-/22.5- & 1925.3$\approx$/39.5-&0.68/0.74 (63\%)\\
APPO &1897.4-/70.8- & 1926.7-/38.0- & 1916.9-/54.1- & 1971.7$\approx$/40.6+ & 1953.0-/40.7- & 1974.7-/84.0- & 1983.7$\approx$/18.5$\approx$ & 1911.5$\approx$/38.8-&0.78/0.75 (63\%)\\
\midrule
MIX &1902.6-/53.0- & 1983.4-/52.1- & 1942.3-/46.2- & 2011.4-/52.0- & 2018.2-/57.6- & 2001.9-/60.1- & 2003.2-/22.6- & 1964.0-/46.7-&0.60/0.73 (59\%)\\
FCFS &1995.6-/77.4- & 2080.9-/90.7- & 2042.5-/75.6- & 2157.1-/109.5- & 2143.2-/121.7- & 2119.2-/85.5- & 2047.9-/36.2- & 1965.3-/37.9-&0.18/0.37 (35\%)\\
EDD &1907.2-/38.0- & 1955.1-/\textbf{26.6}+ & 1985.6-/\textbf{23.0}+ & 2030.9-/36.7+ & 2108.8-/53.2- & 2029.2-/63.4$\approx$ & 1981.3$\approx$/21.6- & 2009.0-/39.3$\approx$&0.52/0.87 (80\%)\\
NVF &1887.2-/66.3- & 1917.0$\approx$/45.1- & 1916.4-/30.2$\approx$ & 2051.2-/72.7- & 1957.5-/46.1- & 1922.6-/57.4- & 2012.0-/20.0$\approx$ & \textbf{1872.8}+/\textbf{30.0}$\approx$&0.76/0.77 (65\%)\\
STD &1903.4-/68.0- & 1926.8-/40.9- & 1900.0-/33.6$\approx$ & 1987.4$\approx$/40.7$\approx$ & 1963.9-/39.9- & 1936.6-/62.5- & 1965.7$\approx$/16.1$\approx$ & 1889.2+/33.9$\approx$&0.82/0.82 (76\%)\\
Random &2091.0-/114.2- & 2083.1-/93.7- & 2078.0-/131.9- & 2131.0-/118.3- & 2094.5-/93.3- & 2145.1-/124.7- & 2078.2-/69.3- & 2083.4-/93.9-&0.04/0.04 (13\%)\\
\midrule
& (0/0/16) / (0/0/16)& (0/1/15) / (1/2/13)& (0/0/16) / (1/3/12)& (2/6/8) / (5/3/8)& (0/1/15) / (0/0/16)& (0/0/16) / (0/2/14)& (2/5/9) / (2/6/8)& (5/6/5) / (1/5/10)\\

\bottomrule
    \end{tabular}
    }

    \label{tab:noised+15}
\end{table*}

\begin{table*}[htbp]
    \centering
        \caption{Performance on DMH$\pm$20: ``$M/C (P)$'' indicates the average normalised makespan, tardiness and percentage of constraint satisfaction over 30 independent trials of five different seeds.
        Bold number indicates the best value in the corresponding column. ``+'',``$\approx$'' and ``-'' indicate the policy performs statistically better/similar/worse than ACERL policy. The number of policies that are ``better'', ``similar'' and ``worse'' than ACERL in terms of makespan and tardiness on each instance is summarised in the bottom row. Horizontal rules in the table separate different groups of algorithms.}
    \resizebox{\textwidth}{!}{
      \setlength{\tabcolsep}{2pt}
    \begin{tabular}{c|c|c|c|c|c|c|c|c|c}
\toprule
\multirow{2}{*}{Algorithm}  &DMH$\pm20$-01 & DMH$\pm20$-02 & DMH$\pm20$-03 & DMH$\pm20$-04 & DMH$\pm20$-05 & DMH$\pm20$-06 & DMH$\pm20$-07 & DMH$\pm20$-08  & \multirow{2}{*}{$M/C (P)$} \\

&$F_m$/$F_t$ & $F_m$/$F_t$& $F_m$/$F_t$& $F_m$/$F_t$& $F_m$/$F_t$& $F_m$/$F_t$& $F_m$/$F_t$& $F_m$/$F_t$&\\
\midrule
ACERL&1865.8/55.7 & 1890.8/31.4 & 1930.9/42.4 & 1958.6/32.1 & 1957.1/36.2 & 1986.1/79.6 & 1949.4/7.4 & \textbf{1869.5}/\textbf{27.8}&0.87/0.92 (75\%)\\
\midrule
MAPPO &1878.8-/53.1$\approx$ & 1941.8-/39.1- & 1934.4$\approx$/45.4- & 1956.8+/35.4- & 1947.2$\approx$/37.8$\approx$ & 2029.8-/81.9- & 1939.4+/10.7- & 1872.8$\approx$/28.6$\approx$&0.83/0.89 (72\%)\\
RCPOM &1887.5-/56.2$\approx$ & 1928.5-/44.8- & 1917.1+/49.9- & 1934.0+/29.6+ & 1954.4$\approx$/35.5+ & 1950.9+/\textbf{69.1}+ & 1944.8+/\textbf{5.7}$\approx$ & 1894.2-/31.1-&0.87/0.91 (66\%)\\
LSAC &1901.3-/51.5$\approx$ & 1944.4-/47.0- & 1940.3$\approx$/50.2- & 1989.8-/55.4- & 1987.2-/45.9- & 2026.3-/81.0$\approx$ & 1969.2-/14.7- & 1941.1-/40.2-&0.71/0.81 (56\%)\\
AMAPPO &1880.2-/56.1$\approx$ & 1924.0-/42.9- & 1918.2+/47.1- & \textbf{1927.6}+/41.9- & 1954.0$\approx$/\textbf{35.3}$\approx$ & 2022.2-/88.3$\approx$ & 1943.8$\approx$/13.6- & 1889.6-/33.0-&0.86/0.85 (70\%)\\
ARCPOM &1871.8-/66.3- & 1919.2-/40.5$\approx$ & \textbf{1901.0}+/44.3$\approx$ & 1952.5$\approx$/38.9- & 1959.2$\approx$/42.1- & 2001.1$\approx$/89.2- & 1957.2-/8.1- & 1941.5-/42.8-&0.82/0.83 (57\%)\\
ALSAC &1866.2$\approx$/\textbf{40.9}+ & 1938.7-/49.1- & 1904.6+/46.6$\approx$ & 1984.5-/43.1- & 1967.3-/44.9- & 1998.1$\approx$/79.5$\approx$ & 1981.5-/15.5- & 1912.8-/41.2-&0.78/0.85 (60\%)\\
\midrule
SAC &1911.3-/56.4$\approx$ & 1965.6-/49.6- & 1940.7$\approx$/48.9- & 1989.1-/46.3- & 1992.3-/48.1- & 2019.7-/83.9$\approx$ & 1944.6+/11.4- & 1926.2-/37.4-&0.72/0.81 (57\%)\\
PPO &1878.9-/52.0+ & 1935.9-/44.3- & 1926.6$\approx$/50.0- & 1960.1$\approx$/46.8- & 1954.2$\approx$/44.2- & 1978.0$\approx$/72.1+ & 1970.2-/16.7- & 1915.2-/34.7-&0.80/0.85 (60\%)\\
ASAC &1922.6-/64.3- & 1976.7-/50.4- & 1941.9$\approx$/47.7- & 1990.4-/45.2- & 1996.3-/51.6- & 2044.7-/95.5- & 1957.3$\approx$/12.4- & 1940.6-/41.9-&0.68/0.77 (53\%)\\
APPO &1883.6-/62.0- & 1922.4-/40.7- & 1923.6$\approx$/55.9- & 1951.6$\approx$/35.4$\approx$ & 1941.1$\approx$/39.1$\approx$ & 1980.4$\approx$/86.2- & 1964.0-/13.9- & 1906.9-/37.7-&0.83/0.83 (61\%)\\
\midrule
MIX &1930.5-/58.4$\approx$ & 1981.7-/54.5- & 1958.5$\approx$/48.5$\approx$ & 2031.2-/59.8- & 2017.8-/55.9- & 2057.2-/84.5$\approx$ & 1964.7-/19.0- & 1952.4-/42.7-&0.60/0.75 (47\%)\\
FCFS &2097.5-/112.7- & 2117.2-/98.1- & 1982.6-/66.2- & 2131.8-/87.2- & 2096.4-/76.5- & 2222.7-/138.7- & 1960.3-/14.4- & 1982.8-/38.9-&0.26/0.42 (30\%)\\
EDD &1981.3-/66.8- & 1957.2-/31.5- & 2043.3-/92.1- & 2035.8-/40.4- & 2120.2-/61.1- & 2108.6-/102.6$\approx$ & 1959.9-/15.1- & 1985.9-/34.0-&0.43/0.73 (57\%)\\
NVF &\textbf{1812.6}+/48.2+ & \textbf{1861.6}+/\textbf{26.2}+ & 1904.8$\approx$/\textbf{30.3}+ & 1986.8-/54.8- & \textbf{1924.5}+/41.8$\approx$ & \textbf{1936.0}+/72.2$\approx$ & 2011.1-/21.7- & 1914.8-/34.7-&0.86/0.89 (77\%)\\
STD &1838.4+/58.6- & 1962.5-/52.5- & 1912.1+/65.4- & 1971.5$\approx$/\textbf{26.5}+ & 1994.8-/38.0$\approx$ & 1967.2$\approx$/69.8$\approx$ & \textbf{1936.6}+/7.6- & 1918.0-/41.2-&0.82/0.85 (59\%)\\
Random &2117.4-/120.9- & 2071.3-/92.6- & 2073.6-/129.8- & 2130.0-/120.8- & 2117.6-/100.8- & 2190.0-/154.2- & 2064.5-/70.0- & 2103.8-/113.0-&0.04/0.01 (9\%)\\
\midrule
& (2/1/13) / (3/6/7)& (1/0/15) / (1/1/14)& (5/8/3) / (1/3/12)& (3/4/9) / (2/1/13)& (1/6/9) / (1/5/10)& (2/5/9) / (2/8/6)& (4/2/10) / (0/1/15)& (0/1/15) / (0/1/15)\\
\bottomrule
    \end{tabular}
    }

    \label{tab:noised+20}
\end{table*}

\begin{table*}[htbp]
    \centering
        \caption{Performance on DMH$\pm$25: ``$M/C (P)$'' indicates the average normalised makespan, tardiness and percentage of constraint satisfaction over 30 independent trials of five different seeds.
        Bold number indicates the best value in the corresponding column. ``+'',``$\approx$'' and ``-'' indicate the policy performs statistically better/similar/worse than ACERL policy. The number of policies that are ``better'', ``similar'' and ``worse'' than ACERL in terms of makespan and tardiness on each instance is summarised in the bottom row. Horizontal rules in the table separate different groups of algorithms.}
    \resizebox{\textwidth}{!}{
      \setlength{\tabcolsep}{2pt}
    \begin{tabular}{c|c|c|c|c|c|c|c|c|c}
\toprule
\multirow{2}{*}{Algorithm}  &DMH$\pm25$-01 & DMH$\pm25$-02 & DMH$\pm25$-03 & DMH$\pm25$-04 & DMH$\pm25$-05 & DMH$\pm25$-06 & DMH$\pm25$-07 & DMH$\pm25$-08  & \multirow{2}{*}{$M/C (P)$} \\
& $F_m$/$F_t$ & $F_m$/$F_t$& $F_m$/$F_t$& $F_m$/$F_t$& $F_m$/$F_t$& $F_m$/$F_t$& $F_m$/$F_t$& $F_m$/$F_t$&\\
\midrule
ACERL&\textbf{1846.2}/\textbf{41.2} & \textbf{1941.1}/\textbf{41.9} & 1916.6/33.7 & 1967.2/41.9 & \textbf{1879.9}/\textbf{33.1} & \textbf{1905.1}/\textbf{40.9} & 2015.4/25.0 & 1914.6/30.1&0.89/0.96 (84\%)\\
\midrule
MAPPO &1858.8-/53.4- & 2031.8-/73.8- & 1940.0$\approx$/49.3- & \textbf{1910.4}+/\textbf{33.8}+ & 1989.2-/49.5- & 2063.2-/81.6- & 2024.8-/27.2- & 1887.4+/30.1$\approx$&0.70/0.77 (55\%)\\
RCPOM &1851.8-/49.2- & 1972.6-/61.9- & 1935.8$\approx$/52.6- & 1960.5+/43.7- & 1945.1-/42.8- & 1976.9-/66.5- & 2004.3$\approx$/19.4+ & 1886.3+/30.2$\approx$&0.81/0.84 (60\%)\\
LSAC &1875.7-/49.7$\approx$ & 1983.7-/46.9- & 1955.4-/50.3- & 1992.3-/61.2- & 1985.0-/49.7- & 2004.6-/72.4- & 1999.6+/24.7$\approx$ & 1940.2-/37.8-&0.70/0.81 (60\%)\\
AMAPPO &1860.0-/50.8$\approx$ & 2012.4-/60.2- & 1911.4+/50.6- & 1918.8+/41.0$\approx$ & 1947.8-/50.1- & 2024.8-/101.5- & 2024.6-/25.1$\approx$ & \textbf{1873.8}+/32.0-&0.77/0.77 (60\%)\\
ARCPOM &1865.9-/63.4- & 1967.1-/48.6- & \textbf{1898.2}+/42.2- & 1940.4+/42.8$\approx$ & 1935.3-/41.7- & 1978.3-/78.4- & 1988.0+/\textbf{16.5}+ & 1893.4+/38.6-&0.85/0.84 (63\%)\\
ALSAC &1928.4-/52.5- & 1967.1-/52.2- & 1904.1+/43.1- & 2002.7-/56.6- & 1986.1-/52.8- & 1994.2-/69.7- & 2023.1-/28.1- & 1905.3+/36.2-&0.72/0.80 (60\%)\\
\midrule
SAC &1884.2-/53.4- & 1979.9-/51.4- & 1928.2$\approx$/47.7- & 1994.6-/54.9- & 1981.8-/50.2- & 2009.3-/76.6- & 2008.2$\approx$/24.3$\approx$ & 1927.7$\approx$/36.4-&0.71/0.80 (57\%)\\
PPO &1867.4-/51.9- & 1967.3-/49.8- & 1911.6+/47.9- & 1985.3-/53.9- & 1947.0-/44.5- & 1977.1-/68.5- & 1984.2+/24.1$\approx$ & 1915.9$\approx$/31.8$\approx$&0.81/0.84 (60\%)\\
ASAC &1909.9-/59.9- & 1995.4-/53.9- & 1937.4-/49.3- & 1998.6-/53.4- & 1981.2-/50.3- & 2035.4-/90.5- & 2003.9+/23.4+ & 1929.7-/37.1-&0.67/0.76 (55\%)\\
APPO &1879.7-/63.5- & 1953.5-/43.8$\approx$ & 1915.7$\approx$/55.8- & 1960.9$\approx$/40.7$\approx$ & 1939.9-/46.1- & 1988.7-/87.9- & 1971.6+/19.6+ & 1894.2+/34.7-&0.84/0.82 (61\%)\\
\midrule
MIX &1912.7-/53.6- & 2004.9-/51.4- & 1946.6$\approx$/50.9- & 2017.4-/59.5- & 2013.9-/59.0- & 2024.7-/73.9- & 2009.0$\approx$/24.7$\approx$ & 1938.7$\approx$/38.3-&0.63/0.77 (57\%)\\
FCFS &2111.7-/106.9- & 2060.2-/82.3- & 2007.7-/84.8- & 2072.9-/75.8- & 2140.8-/110.0- & 2073.6-/85.3- & 2085.5-/37.5- & 1952.2-/\textbf{29.3}$\approx$&0.26/0.45 (34\%)\\
EDD &2100.6-/114.1- & 2015.6-/44.9$\approx$ & 2033.8-/65.8- & 2022.9-/36.4+ & 2013.9-/41.4- & 2012.8-/54.2- & \textbf{1955.7}+/19.2+ & 2032.2-/55.1-&0.46/0.75 (62\%)\\
NVF &1877.9-/54.1- & 1944.5$\approx$/43.1$\approx$ & 1918.8$\approx$/\textbf{27.5}+ & 2009.1-/65.0- & 1900.2-/42.0- & 1972.2-/77.4- & 1970.1+/19.9+ & 1904.4$\approx$/31.0$\approx$&0.84/0.86 (62\%)\\
STD &1851.4-/54.8- & 1959.0-/51.6- & 1943.4-/66.9- & 1967.3$\approx$/36.8$\approx$ & 2015.6-/51.9- & 1989.4-/69.2- & 2047.2-/25.4$\approx$ & 1913.3$\approx$/40.1-&0.72/0.80 (58\%)\\
Random &2088.7-/123.4- & 2096.9-/93.6- & 2084.5-/133.0- & 2146.7-/121.4- & 2106.9-/95.9- & 2153.2-/130.2- & 2111.2-/79.9- & 2101.4-/98.1-&0.03/0.02 (7\%)\\
\midrule
& (0/0/16) / (0/2/14)& (0/1/15) / (0/3/13)& (4/6/6) / (1/0/15)& (4/2/10) / (2/4/10)& (0/0/16) / (0/0/16)& (0/0/16) / (0/0/16)& (7/3/6) / (6/6/4)& (6/5/5) / (0/5/11)\\

\bottomrule
    \end{tabular}
    }

    \label{tab:noised+25}
\end{table*}

\begin{table*}[htbp]
    \centering
        \caption{Performance on DMH$\pm$30: ``$M/C (P)$'' indicates the average normalised makespan, tardiness and percentage of constraint satisfaction over 30 independent trials of five different seeds.
        Bold number indicates the best value in the corresponding column. ``+'',``$\approx$'' and ``-'' indicate the policy performs statistically better/similar/worse than ACERL policy. The number of policies that are ``better'', ``similar'' and ``worse'' than ACERL in terms of makespan and tardiness on each instance is summarised in the bottom row. Horizontal rules in the table separate different groups of algorithms.}
    \resizebox{\textwidth}{!}{
      \setlength{\tabcolsep}{2pt}
    \begin{tabular}{c|c|c|c|c|c|c|c|c|c}
\toprule
\multirow{2}{*}{Algorithm}  &DMH$\pm30$-01 & DMH$\pm30$-02 & DMH$\pm30$-03 & DMH$\pm30$-04 & DMH$\pm30$-05 & DMH$\pm30$-06 & DMH$\pm30$-07 & DMH$\pm30$-08  & \multirow{2}{*}{$M/C (P)$} \\

& $F_m$/$F_t$ & $F_m$/$F_t$& $F_m$/$F_t$& $F_m$/$F_t$& $F_m$/$F_t$& $F_m$/$F_t$& $F_m$/$F_t$& $F_m$/$F_t$&\\
\midrule
ACERL&\textbf{1837.2}/\textbf{25.6} & 1926.4/39.3 & \textbf{1892.5}/32.4 & 1944.0/43.6 & \textbf{1901.0}/\textbf{30.1} & 1943.5/\textbf{39.2} & 1949.8/\textbf{10.4} & 1876.2/\textbf{22.9}&0.93/0.96 (93\%)\\
\midrule
MAPPO &1882.4-/40.0- & 1969.8-/39.6$\approx$ & 1939.4-/31.5$\approx$ & \textbf{1937.4}$\approx$/50.2- & 2017.4-/59.0- & 2012.2-/67.6- & 1935.2+/19.7- & 1885.4$\approx$/32.5-&0.78/0.82 (70\%)\\
RCPOM &1851.3$\approx$/34.2- & 1953.8-/46.8- & 1910.3-/42.7- & 1938.4+/46.3- & 1969.7-/48.9- & 1938.8$\approx$/55.5- & 1948.4$\approx$/11.4- & 1884.1$\approx$/26.4$\approx$&0.87/0.87 (70\%)\\
LSAC &1891.7-/47.6- & 1974.3-/44.5- & 1931.7-/53.3- & 1990.3-/56.7- & 1980.9-/52.0- & 1982.7-/63.7- & 1950.4$\approx$/14.2- & 1918.3-/37.5-&0.74/0.79 (60\%)\\
AMAPPO &1873.8-/36.3- & 1988.2-/38.9$\approx$ & 1937.6-/31.2$\approx$ & 1953.0$\approx$/48.1- & 1982.4-/58.2- & 1959.8-/64.0- & \textbf{1923.8}+/15.4- & \textbf{1867.4}+/33.0-&0.83/0.84 (70\%)\\
ARCPOM &1844.6$\approx$/39.4- & 1925.8$\approx$/40.8$\approx$ & 1934.2-/40.1- & 1962.9-/44.4$\approx$ & 1928.1-/39.4- & 1963.7-/70.6- & 1956.3$\approx$/20.3- & 1903.4-/43.1-&0.84/0.82 (69\%)\\
ALSAC &1907.2-/44.4- & 1953.1-/46.8- & 1929.3-/45.1- & 1942.5$\approx$/49.7- & 1973.1-/52.2- & 1977.1-/63.1- & 1955.5-/16.3- & 1909.3-/41.7-&0.79/0.79 (68\%)\\
\midrule
SAC &1886.7-/43.5- & 1955.8-/43.9- & 1939.6-/47.8- & 1976.6-/54.5- & 1983.1-/52.1- & 1987.8-/64.8- & 1945.3$\approx$/14.1- & 1925.8-/36.9-&0.75/0.80 (62\%)\\
PPO &1863.3-/42.5- & 1931.6$\approx$/39.6$\approx$ & 1924.6-/50.6- & 1960.2-/51.3- & 1955.0-/49.1- & 1963.7-/61.3- & 1947.1$\approx$/14.9- & 1914.1-/32.5-&0.83/0.83 (66\%)\\
ASAC &1897.2-/44.3- & 1972.5-/45.1- & 1938.2-/46.8- & 1991.6-/52.6- & 1993.9-/55.1- & 2015.1-/74.9- & 1959.7$\approx$/19.6- & 1936.9-/39.1-&0.69/0.76 (60\%)\\
APPO &1857.8$\approx$/44.7- & 1942.5-/41.4- & 1925.6-/49.3- & 1956.8$\approx$/40.9+ & 1952.5-/52.6- & 1967.8-/76.3- & 1956.0-/17.3- & 1898.8-/34.1-&0.83/0.80 (64\%)\\
\midrule
MIX &1935.8-/48.8- & 1989.7-/46.6- & 1960.1-/53.5- & 1995.6-/55.9- & 2026.3-/62.7- & 2013.2-/64.7- & 1944.2$\approx$/15.2- & 1980.4-/47.2-&0.62/0.75 (53\%)\\
FCFS &2104.5-/107.6- & 2087.1-/92.7- & 2046.7-/94.1- & 2121.8-/107.0- & 2162.3-/120.8- & 2134.1-/77.1- & 2036.0-/18.0- & 2007.8-/39.1-&0.12/0.39 (32\%)\\
EDD &2020.6-/67.2- & 2021.5-/38.5+ & 2017.1-/87.8- & 2033.1-/\textbf{37.3}+ & 2037.3-/56.2- & 2027.6-/52.1- & 1979.4-/20.6- & 1991.0-/43.2-&0.46/0.76 (53\%)\\
NVF &1851.3$\approx$/34.0- & \textbf{1892.2}+/\textbf{28.3}+ & 1896.9$\approx$/\textbf{25.2}+ & 1981.7-/55.1$\approx$ & 1909.4$\approx$/41.5- & \textbf{1916.0}+/59.0- & 1936.8+/16.5- & 1912.1-/34.6-&0.92/0.89 (81\%)\\
STD &1862.7$\approx$/42.2- & 1968.7-/49.6- & 1903.5-/32.4$\approx$ & 1964.9-/39.3$\approx$ & 1980.6-/49.5- & 1966.9$\approx$/60.9- & 1935.2+/14.9- & 1899.0-/36.9-&0.82/0.84 (72\%)\\
Random &2115.1-/130.9- & 2104.1-/94.4- & 2097.8-/147.2- & 2137.6-/128.3- & 2171.4-/130.2- & 2119.8-/111.8- & 2057.6-/69.5- & 2086.2-/101.8-&0.01/0.00 (9\%)\\
\midrule
& (0/5/11) / (0/0/16)& (1/2/13) / (2/4/10)& (0/1/15) / (1/3/12)& (1/4/11) / (2/3/11)& (0/1/15) / (0/0/16)& (1/2/13) / (0/0/16)& (4/7/5) / (0/0/16)& (1/2/13) / (0/1/15)\\
\bottomrule
    \end{tabular}
    }

    \label{tab:noised+30}
\end{table*}

\newpage
\onecolumn

\subsection{Supplementary experimental results of cross-validation}\label{sec:suploo}
Tables \ref{tab:l1o2}, \ref{tab:l1o3}, \ref{tab:l1o4}, \ref{tab:l1o5}, \ref{tab:l1o6}, \ref{tab:l1o7} and \ref{tab:l1o8} are experiment results of leave-one-out cross-validation presented in Section V-C of the main manuscript by leaving DMH-01 to DMH-08 out, respectively. ACERL shows the best overall performance on all leave-one-out cross-validations for the highest average makespan $M$, average tardiness $C$ and constraint satisfaction percentages $P$. Although ACERL is not trained on the split instance, it still demonstrates high constraint satisfaction on the split instance, while other algorithms fail to do so.

\begin{table*}[htbp]
    \centering
        \caption{Cross validation using leave-one-out on DMH-02 (highlighted with grey blocks): average makespan and tardiness over 30 independent trials of five different seeds on each instance. Bold numbers indicate the best makespan and tardiness. ``+'',``$\approx$'' and ``-'' indicate the policy performs statistically better/similar/worse than ACERL policy. The number of policies that are ``better'', ``similar'' and ``worse'' than ACERL in terms of makespan and tardiness on each instance is summarised in the bottom row. The last column with header ``$M/C (P)$'' indicates the average normalised makespan, tardiness and percentage of constraint satisfaction. Horizontal rules in the table separate different groups of algorithms.}
    \resizebox{\textwidth}{!}{
      \setlength{\tabcolsep}{2pt}
    \begin{tabular}{c|c|c|c|c|c|c|c|c|c}
\toprule
\multicolumn{1}{c|}{\multirow{2}{*}{Algorithm} }&DMH-01 & \cellcolor[rgb]{.9,.9,.9}DMH-02 & DMH-03 & DMH-04 & DMH-05 & DMH-06 & DMH-07 & DMH-08 & \multirow{2}{*}{$M/C (P)$} \\
 &$F_m$/$F_t$ & \cellcolor[rgb]{.9,.9,.9}$F_m$/$F_t$& $F_m$/$F_t$& $F_m$/$F_t$& $F_m$/$F_t$& $F_m$/$F_t$& $F_m$/$F_t$& $F_m$/$F_t$&\\
\midrule
ACERL&\textbf{1788.2}/39.9 & \cellcolor[rgb]{.9,.9,.9}1945.9/45.9 & \textbf{1855.0}/30.9 & \textbf{1918.2}/\textbf{30.6} & 

\textbf{1866.0}/\textbf{27.4} & \textbf{1863.0}/\textbf{38.9} & 1932.8/\textbf{8.1} & 1860.7/\textbf{24.7}&0.97/0.96 (92\%)\\
\midrule
MAPPO&1886.0-/67.6- & \cellcolor[rgb]{.9,.9,.9}1928.8+/43.6$\approx$ & 1882.8-/52.7- & 1938.2-/36.3- & 1940.0-/46.0- & 2004.6-/66.2- & 1932.8$\approx$/14.5- & 1862.0$\approx$/27.3-&0.83/0.81 (64\%)\\
RCPOM&1912.5-/66.5- & \cellcolor[rgb]{.9,.9,.9}1973.9-/52.5$\approx$ & 1938.0-/49.9- & 1998.7-/49.7- & 1977.8-/44.5- & 2038.4-/90.6- & 1959.5-/11.2$\approx$ & 1917.7-/32.2-&0.64/0.75 (56\%)\\
LSAC&1887.8-/60.9- & \cellcolor[rgb]{.9,.9,.9}1958.6$\approx$/46.1$\approx$ & 1927.5-/46.3- & 1979.6-/42.0- & 1951.3-/44.6- & 1982.2-/52.7- & 1985.1-/22.2- & 1894.3-/26.8$\approx$&0.70/0.82 (65\%)\\
AMAPPO&1854.4-/58.9- & \cellcolor[rgb]{.9,.9,.9}\textbf{1905.4}+/39.6+ & 1922.2-/55.7- & 1938.2-/36.4- & 1968.0-/47.1- & 1955.4-/73.7- & 1950.8-/19.9- & \textbf{1857.8}+/27.1-&0.83/0.81 (64\%)\\
ARCPOM&1904.7-/66.4- & \cellcolor[rgb]{.9,.9,.9}1940.3$\approx$/43.2+ & 1928.7-/49.3- & 1971.8-/48.3- & 1942.8-/39.0- & 1992.7-/75.2- & 1945.0-/13.2- & 1923.1-/33.0-&0.73/0.79 (61\%)\\
ALSAC&1882.2-/62.2- & \cellcolor[rgb]{.9,.9,.9}1947.6$\approx$/49.4- & 1920.3-/46.6- & 1974.4-/42.6- & 1919.6-/40.0- & 1965.6-/69.1- & 1958.5-/16.2- & 1912.5-/40.3-&0.75/0.78 (56\%)\\
\midrule
SAC&1905.9-/62.3- & \cellcolor[rgb]{.9,.9,.9}1974.4-/50.0$\approx$ & 1926.8-/39.9- & 1993.6-/58.2- & 1992.0-/57.9- & 1998.3-/68.7- & 1935.5-/12.4- & 1934.6-/37.2-&0.67/0.76 (55\%)\\
PPO&1913.4-/60.8- & \cellcolor[rgb]{.9,.9,.9}1986.0-/57.1- & 1936.9-/43.2- & 1990.7-/49.5- & 1990.1-/54.4- & 2032.3-/92.3- & 1952.6-/15.6- & 1926.9-/37.3-&0.64/0.72 (55\%)\\
ASAC&1914.1-/67.6- & \cellcolor[rgb]{.9,.9,.9}1950.9$\approx$/49.8- & 1971.0-/65.8- & 2006.8-/54.0- & 1979.2-/54.5- & 1993.7-/58.7- & 1967.1-/17.3- & 1948.1-/47.4-&0.63/0.72 (51\%)\\
APPO&1915.5-/68.5- & \cellcolor[rgb]{.9,.9,.9}1973.8-/53.0- & 1939.1-/46.2- & 1982.7-/48.1- & 1981.5-/55.6- & 2018.3-/92.9- & 1957.6-/15.7- & 1919.2-/35.1-&0.66/0.72 (53\%)\\
\midrule
MIX&1939.9-/60.3- & \cellcolor[rgb]{.9,.9,.9}2000.9-/60.4- & 1932.2-/45.0- & 2006.1-/52.8- & 2028.2-/64.2- & 2027.1-/69.5- & 1969.2-/16.1- & 1971.0-/46.9-&0.56/0.71 (54\%)\\
FCFS&2081.1-/90.2- & \cellcolor[rgb]{.9,.9,.9}2084.5-/82.3- & 2014.7-/64.2- & 2136.7-/105.0- & 2194.6-/122.5- & 2123.5-/85.2- & \textbf{1927.9}+/11.2- & 1933.6-/27.4$\approx$&0.29/0.49 (37\%)\\
EDD&1903.2-/\textbf{33.9}+ & \cellcolor[rgb]{.9,.9,.9}1968.6-/\textbf{29.6}+ & 1977.8-/\textbf{26.5}$\approx$ & 1988.1-/32.8$\approx$ & 1950.7-/33.7- & 2016.8-/46.8- & 1940.5-/12.1- & 2020.8-/45.3-&0.61/0.93 (86\%)\\
NVF&1876.5-/69.6- & \cellcolor[rgb]{.9,.9,.9}1958.4$\approx$/56.4- & 1946.7-/49.6- & 2040.6-/66.7- & 1933.9-/56.3- & 1953.4-/78.5- & 1996.5-/21.8- & 1944.6-/35.5-&0.64/0.69 (51\%)\\
STD&1868.8-/66.6- & \cellcolor[rgb]{.9,.9,.9}1961.7$\approx$/49.2$\approx$ & 1921.3-/51.7- & 1970.7-/35.6- & 1917.1-/41.4- & 1955.4-/62.7- & 1983.7-/30.4- & 1883.5-/35.1-&0.75/0.77 (53\%)\\
Random&2098.7-/124.5- & \cellcolor[rgb]{.9,.9,.9}2113.8-/103.1- & 2091.2-/143.7- & 2135.1-/123.1- & 2149.3-/119.0- & 2159.1-/129.3- & 2083.0-/70.2- & 2067.8-/89.5-&0.02/0.00 (12\%)\\
\midrule
& (0/0/16) / (1/0/15)& \cellcolor[rgb]{.9,.9,.9}(2/6/8) / (3/5/8)& (0/0/16) / (0/1/15)& (0/0/16) / (0/1/15)& (0/0/16) / (0/0/16)& (0/0/16) / (0/0/16)& (1/1/14) / (0/1/15)& (1/1/14) / (0/2/14)\\

\bottomrule
    \end{tabular}
    }

    \label{tab:l1o2}
\end{table*}

\begin{table*}[htbp]
    \centering
        \caption{Cross validation using leave-one-out on DMH-03 (highlighted with grey blocks): average makespan and tardiness over 30 independent trials of five different seeds on each instance. Bold numbers indicate the best makespan and tardiness. ``+'',``$\approx$'' and ``-'' indicate the policy performs statistically better/similar/worse than ACERL policy. The number of policies that are ``better'', ``similar'' and ``worse'' than ACERL in terms of makespan and tardiness on each instance is summarised in the bottom row. The last column with header ``$M/C (P)$'' indicates the average normalised makespan, tardiness and percentage of constraint satisfaction. Horizontal rules in the table separate different groups of algorithms.}
    \resizebox{\textwidth}{!}{
      \setlength{\tabcolsep}{2pt}
    \begin{tabular}{c|c|c|c|c|c|c|c|c|c}
\toprule
\multicolumn{1}{c|}{\multirow{2}{*}{Algorithm} }&DMH-01 & DMH-02 &   \cellcolor[rgb]{.9,.9,.9}DMH-03 & DMH-04 & DMH-05 & DMH-06 & DMH-07 & DMH-08 & \multirow{2}{*}{$M/C (P)$} \\
 &$F_m$/$F_t$ & $F_m$/$F_t$& \cellcolor[rgb]{.9,.9,.9}$F_m$/$F_t$& $F_m$/$F_t$& $F_m$/$F_t$& $F_m$/$F_t$& $F_m$/$F_t$& $F_m$/$F_t$&\\
\midrule
ACERL&\textbf{1796.8}/37.5 & \textbf{1873.8}/36.8 & \cellcolor[rgb]{.9,.9,.9}1921.4/38.7 & \textbf{1889.0}/35.5 & \textbf{1862.0}/\textbf{28.8} & \textbf{1860.4}/\textbf{38.2} & \textbf{1927.0}/\textbf{7.4} & \textbf{1861.0}/\textbf{23.9}&0.98/0.97 (93\%)\\
\midrule
MAPPO&1883.2-/68.9- & 1914.2-/34.0$\approx$ & \cellcolor[rgb]{.9,.9,.9}1906.0+/54.9- & 1960.2-/42.5- & 1943.8-/53.4- & 1986.0-/77.3- & 1927.0$\approx$/13.2- & 1865.2$\approx$/37.2-&0.81/0.78 (55\%)\\
RCPOM&1907.5-/76.1- & 1952.5-/51.9- & \cellcolor[rgb]{.9,.9,.9}1923.7-/52.7- & 1984.0-/51.0- & 1949.4-/42.0- & 2000.9-/83.2- & 1938.2-/12.3- & 1912.8-/37.4-&0.71/0.74 (57\%)\\
LSAC&1927.3-/70.5- & 1991.3-/57.6- & \cellcolor[rgb]{.9,.9,.9}1944.5-/44.8$\approx$ & 1985.2-/50.5- & 1989.3-/53.8- & 2024.9-/85.4- & 1949.7-/13.3- & 1955.5-/38.1-&0.61/0.72 (55\%)\\
AMAPPO&1889.4-/82.7- & 1919.0-/37.4$\approx$ & \cellcolor[rgb]{.9,.9,.9}\textbf{1890.6}+/54.1- & 1958.6-/37.8- & 1909.0-/46.7- & 2045.8-/118.9- & 1927.0$\approx$/15.3- & 1883.4-/39.2-&0.79/0.70 (57\%)\\
ARCPOM&1902.1-/64.1- & 1978.4-/59.7- & \cellcolor[rgb]{.9,.9,.9}1951.1-/51.0- & 1997.4-/45.2- & 1944.2-/44.7- & 2011.1-/79.0- & 1979.3-/23.7- & 1928.6-/29.3-&0.63/0.74 (48\%)\\
ALSAC&1913.0-/62.0- & 1977.0-/55.5- & \cellcolor[rgb]{.9,.9,.9}1942.5-/58.0- & 1954.3-/55.6- & 1934.6-/41.4- & 2008.7-/79.4- & 1947.0-/12.9- & 1921.5-/35.8-&0.69/0.75 (59\%)\\
\midrule
SAC&1886.9-/58.2- & 1980.3-/52.0- & \cellcolor[rgb]{.9,.9,.9}1922.1$\approx$/45.5$\approx$ & 1990.9-/55.9- & 1952.4-/45.2- & 1983.0-/70.3- & 1955.5-/15.8- & 1922.2-/37.7-&0.69/0.77 (58\%)\\
PPO&1902.8-/61.7- & 1978.7-/54.5- & \cellcolor[rgb]{.9,.9,.9}1932.9-/42.3$\approx$ & 1989.7-/48.6- & 1966.5-/51.6- & 2005.3-/86.4- & 1963.6-/17.1- & 1922.9-/35.4-&0.66/0.74 (56\%)\\
ASAC&1898.8-/70.7- & 1960.9-/52.7- & \cellcolor[rgb]{.9,.9,.9}1929.3-/51.5- & 1975.7-/44.6- & 1928.9-/38.2- & 1970.5-/75.0- & 1957.2-/14.9- & 1906.4-/34.3-&0.72/0.77 (52\%)\\
APPO&1893.4-/64.9- & 1974.0-/52.9- & \cellcolor[rgb]{.9,.9,.9}1929.2-/44.8- & 1985.9-/49.1- & 1962.1-/50.3- & 2011.8-/96.6- & 1960.1-/14.3- & 1915.4-/35.3-&0.67/0.73 (54\%)\\
\midrule
MIX&1939.9-/60.3- & 2000.9-/60.4- & \cellcolor[rgb]{.9,.9,.9}1932.2$\approx$/45.0$\approx$ & 2006.1-/52.8- & 2028.2-/64.2- & 2027.1-/69.5- & 1969.2-/16.1- & 1971.0-/46.9-&0.56/0.71 (54\%)\\
FCFS&2081.1-/90.2- & 2084.5-/82.3- & \cellcolor[rgb]{.9,.9,.9}2014.7-/64.2- & 2136.7-/105.0- & 2194.6-/122.5- & 2123.5-/85.2- & 1927.9$\approx$/11.2- & 1933.6-/27.4$\approx$&0.29/0.49 (37\%)\\
EDD&1903.2-/\textbf{33.9}+ & 1968.6-/\textbf{29.6}+ & \cellcolor[rgb]{.9,.9,.9}1977.8-/\textbf{26.5}+ & 1988.1-/\textbf{32.8}+ & 1950.7-/33.7$\approx$ & 2016.8-/46.8$\approx$ & 1940.5-/12.1- & 2020.8-/45.3-&0.60/0.93 (86\%)\\
NVF&1876.5-/69.6- & 1958.4-/56.4- & \cellcolor[rgb]{.9,.9,.9}1946.7-/49.6- & 2040.6-/66.7- & 1933.9-/56.3- & 1953.4-/78.5- & 1996.5-/21.8- & 1944.6-/35.5-&0.64/0.69 (51\%)\\
STD&1868.8-/66.6- & 1961.7-/49.2- & \cellcolor[rgb]{.9,.9,.9}1921.3$\approx$/51.7- & 1970.7-/35.6$\approx$ & 1917.1-/41.4- & 1955.4-/62.7- & 1983.7-/30.4- & 1883.5$\approx$/35.1-&0.74/0.77 (53\%)\\
Random&2098.7-/124.5- & 2113.8-/103.1- &\cellcolor[rgb]{.9,.9,.9} 2091.2-/143.7- & 2135.1-/123.1- & 
2149.3-/119.0- & 2159.1-/129.3- & 2083.0-/70.2- & 2067.8-/89.5-&0.02/0.00 (12\%)\\
\midrule
& (0/0/16) / (1/0/15)& (0/0/16) / (1/2/13)& \cellcolor[rgb]{.9,.9,.9}(2/3/11) / (1/4/11)& (0/0/16) / (1/1/14)& (0/0/16) / (0/1/15)& (0/0/16) / (0/1/15)& (0/3/13) / (0/0/16)& (0/2/14) / (0/1/15)\\

\bottomrule
    \end{tabular}
    }

    \label{tab:l1o3}
\end{table*}

\begin{table*}[htbp]
    \centering
        \caption{Cross validation using leave-one-out on DMH-04 (highlighted with grey blocks): average makespan and tardiness over 30 independent trials of five different seeds on each instance. Bold numbers indicate the best makespan and tardiness. ``+'',``$\approx$'' and ``-'' indicate the policy performs statistically better/similar/worse than ACERL policy. The number of policies that are ``better'', ``similar'' and ``worse'' than ACERL in terms of makespan and tardiness on each instance is summarised in the bottom row. The last column with header ``$M/C (P)$'' indicates the average normalised makespan, tardiness and percentage of constraint satisfaction. Horizontal rules in the table separate different groups of algorithms.}
    \resizebox{\textwidth}{!}{
      \setlength{\tabcolsep}{2pt}
    \begin{tabular}{c|c|c|c|c|c|c|c|c|c}
\toprule
\multicolumn{1}{c|}{\multirow{2}{*}{Algorithm} }&DMH-01 & DMH-02 & DMH-03 &\cellcolor[rgb]{.9,.9,.9} DMH-04 & DMH-05 & DMH-06 & DMH-07 & DMH-08 & \multirow{2}{*}{$M/C (P)$} \\
 &$F_m$/$F_t$ & $F_m$/$F_t$& $F_m$/$F_t$&\cellcolor[rgb]{.9,.9,.9} $F_m$/$F_t$& $F_m$/$F_t$& $F_m$/$F_t$& $F_m$/$F_t$& $F_m$/$F_t$&\\
\midrule
ACERL&\textbf{1801.2}/\textbf{29.0} & \textbf{1862.0}/29.9 & \textbf{1846.8}/\textbf{23.7} &\cellcolor[rgb]{.9,.9,.9} 1969.0/46.1 & \textbf{1844.2}/\textbf{27.1} & \textbf{1843.8}/\textbf{31.3} & 1932.8/\textbf{10.0} & \textbf{1862.0}/30.2&0.98/0.98 (97\%)\\
\midrule
MAPPO&1902.2-/61.0- & 1920.2-/35.6- & 1904.2-/40.4- & \cellcolor[rgb]{.9,.9,.9}\textbf{1948.0}+/48.8- & 1977.2-/58.0- & 1982.0-/65.6- & \textbf{1927.0}+/10.8$\approx$ & 1870.4-/31.9$\approx$&0.79/0.81 (65\%)\\
RCPOM&1891.4-/66.6- & 1956.8-/47.0- & 1929.8-/47.9- & \cellcolor[rgb]{.9,.9,.9}1971.9$\approx$/42.8$\approx$ & 1954.8-/47.7- & 1998.1-/76.3- & 1961.8-/17.6- & 1923.3-/40.5-&0.69/0.76 (54\%)\\
LSAC&1927.2-/68.5- & 1970.8-/46.8- & 1929.3-/42.6- & \cellcolor[rgb]{.9,.9,.9}2017.6-/59.8- & 1998.4-/53.1- & 2023.6-/84.8- & 1942.5$\approx$/17.5- & 1919.0-/35.7-&0.63/0.73 (57\%)\\
AMAPPO&1892.0-/59.6- & 1903.4-/36.1- & 1918.4-/48.3- & \cellcolor[rgb]{.9,.9,.9}1970.0-/40.6$\approx$ & 1894.8-/39.6- & 2002.0-/75.1- & 1927.0+/14.4- & 1866.4$\approx$/28.8+&0.81/0.83 (64\%)\\
ARCPOM&1890.7-/59.2- & 1958.0-/48.1- & 1916.3-/48.0- & \cellcolor[rgb]{.9,.9,.9}1994.6-/47.5- & 1971.1-/51.7- & 1972.3-/83.0- & 1980.9-/24.0- & 1900.5-/30.2$\approx$&0.69/0.75 (55\%)\\
ALSAC&1903.5-/67.5- & 1955.4-/47.2- & 1922.2-/53.3- & \cellcolor[rgb]{.9,.9,.9}1991.0-/51.0- & 1937.2-/40.9- & 2021.8-/94.6- & 1951.3-/16.5- & 1921.9-/36.5-&0.68/0.73 (53\%)\\
\midrule
SAC&1916.9-/73.6- & 1941.6-/46.9- & 1934.7-/55.9- & \cellcolor[rgb]{.9,.9,.9}1974.7$\approx$/48.1$\approx$ & 1987.7-/51.1- & 2028.0-/99.1- & 1970.6-/19.9- & 1933.6-/38.7-&0.65/0.70 (56\%)\\
PPO&1922.3-/64.9- & 1982.1-/55.1- & 1938.1-/44.6- & \cellcolor[rgb]{.9,.9,.9}1987.4-/49.5- & 1987.4-/57.6- & 2030.9-/91.7- & 1950.7-/15.5- & 1932.3-/37.5-&0.63/0.72 (52\%)\\
ASAC&1900.9-/65.0- & 1970.0-/52.3- & 1916.2-/43.8- & \cellcolor[rgb]{.9,.9,.9}1963.8+/44.4$\approx$ & 1967.4-/46.4- & 1987.0-/74.7- & 1978.6-/18.3- & 1908.1-/36.4-&0.69/0.76 (56\%)\\
APPO&1908.3-/68.1- & 1977.6-/53.3- & 1934.4-/44.6- & \cellcolor[rgb]{.9,.9,.9}1989.7-/49.0- & 1978.7-/53.7- & 2019.3-/99.4- & 1958.2-/15.8- & 1922.5-/36.8-&0.65/0.71 (53\%)\\
\midrule
MIX&1939.9-/60.3- & 2000.9-/60.4- & 1932.2-/45.0- & \cellcolor[rgb]{.9,.9,.9}2006.1-/52.8- & 2028.2-/64.2- & 2027.1-/69.5- & 1969.2-/16.1- & 1971.0-/46.9-&0.55/0.71 (54\%)\\
FCFS&2081.1-/90.2- & 2084.5-/82.3- & 2014.7-/64.2- & \cellcolor[rgb]{.9,.9,.9}2136.7-/105.0- & 2194.6-/122.5- & 2123.5-/85.2- & 1927.9+/11.2- & 1933.6-/\textbf{27.4}$\approx$&0.28/0.49 (37\%)\\
EDD&1903.2-/33.9$\approx$ & 1968.6-/\textbf{29.6}$\approx$ & 1977.8-/26.5- & \cellcolor[rgb]{.9,.9,.9}1988.1-/\textbf{32.8}+ & 1950.7-/33.7- & 2016.8-/46.8$\approx$ & 1940.5-/12.1- & 2020.8-/45.3-&0.60/0.92 (86\%)\\
NVF&1876.5-/69.6- & 1958.4-/56.4- & 1946.7-/49.6- & \cellcolor[rgb]{.9,.9,.9}2040.6-/66.7- & 1933.9-/56.3- & 1953.4-/78.5- & 1996.5-/21.8- & 1944.6-/35.5$\approx$&0.63/0.69 (51\%)\\
STD&1868.8-/66.6- & 1961.7-/49.2- & 1921.3-/51.7- & \cellcolor[rgb]{.9,.9,.9}1970.7$\approx$/35.6+ & 1917.1-/41.4- & 1955.4-/62.7- & 1983.7-/30.4- & 1883.5-/35.1$\approx$&0.74/0.77 (53\%)\\
Random&2098.7-/124.5- & 2113.8-/103.1- & 2091.2-/143.7- & \cellcolor[rgb]{.9,.9,.9}2135.1-/123.1- & 2149.3-/119.0- & 2159.1-/129.3- & 2083.0-/70.2- & 2067.8-/89.5-&0.02/0.00 (12\%)\\
\midrule
& (0/0/16) / (0/1/15)& (0/0/16) / (0/1/15)& (0/0/16) / (0/0/16)& \cellcolor[rgb]{.9,.9,.9}(2/3/11) / (2/4/10)& (0/0/16) / (0/0/16)& (0/0/16) / (0/1/15)& (3/1/12) / (0/1/15)& (0/1/15) / (1/5/10)\\

\bottomrule
    \end{tabular}
    }

    \label{tab:l1o4}
\end{table*}

\begin{table*}[htbp]
    \centering
        \caption{Cross validation using leave-one-out on DMH-05 (highlighted with grey blocks): average makespan and tardiness over 30 independent trials of five different seeds on each instance. Bold numbers indicate the best makespan and tardiness. ``+'',``$\approx$'' and ``-'' indicate the policy performs statistically better/similar/worse than ACERL policy. The number of policies that are ``better'', ``similar'' and ``worse'' than ACERL in terms of makespan and tardiness on each instance is summarised in the bottom row. The last column with header ``$M/C (P)$'' indicates the average normalised makespan, tardiness and percentage of constraint satisfaction. Horizontal rules in the table separate different groups of algorithms.}
    \resizebox{\textwidth}{!}{
      \setlength{\tabcolsep}{2pt}
    \begin{tabular}{c|c|c|c|c|c|c|c|c|c}
\toprule
\multicolumn{1}{c|}{\multirow{2}{*}{Algorithm} }&DMH-01 & DMH-02 & DMH-03 & DMH-04 & \cellcolor[rgb]{.9,.9,.9}DMH-05 & DMH-06 & DMH-07 & DMH-08 & \multirow{2}{*}{$M/C (P)$} \\
 &$F_m$/$F_t$ & $F_m$/$F_t$& $F_m$/$F_t$& $F_m$/$F_t$& \cellcolor[rgb]{.9,.9,.9}$F_m$/$F_t$& $F_m$/$F_t$& $F_m$/$F_t$& $F_m$/$F_t$&\\
\midrule
ACERL&\textbf{1805.2}/41.6 & \textbf{1849.0}/\textbf{29.0} & \textbf{1846.0}/28.4 & \textbf{1900.6}/\textbf{31.6} &\cellcolor[rgb]{.9,.9,.9} 1962.1/51.7 & \textbf{1863.2}/\textbf{34.0} & \textbf{1927.0}/\textbf{9.4} & \textbf{1851.4}/28.0&0.98/0.96 (92\%)\\
\midrule
MAPPO&1879.8-/68.8- & 1924.4-/35.5- & 1912.0-/51.6- & 1962.6-/41.1- & \cellcolor[rgb]{.9,.9,.9}1956.8$\approx$/56.1- & 2004.8-/87.2- & 1961.6-/22.8- & 1874.0-/31.7$\approx$&0.75/0.76 (52\%)\\
RCPOM&1904.7-/72.2- & 1940.4-/46.3- & 1941.2-/48.0- & 1992.7-/47.8- & \cellcolor[rgb]{.9,.9,.9}1967.8$\approx$/46.5$\approx$ & 2004.3-/82.7- & 1952.5-/13.3- & 1925.8-/33.1$\approx$&0.67/0.77 (57\%)\\
LSAC&1917.8-/68.7- & 1935.9-/43.9- & 1923.5-/47.3- & 1959.4-/39.2- & \cellcolor[rgb]{.9,.9,.9}1938.1+/41.3+ & 1961.1-/66.8- & 1965.4-/17.1- & 1929.7-/38.8-&0.71/0.80 (62\%)\\
AMAPPO&1857.2-/48.0- & 1930.4-/34.4- & 1944.6-/43.8- & 1961.6-/42.0- & \cellcolor[rgb]{.9,.9,.9}1975.2$\approx$/57.2- & 1988.6-/68.1- & 1927.0$\approx$/14.4- & 1872.6-/28.5$\approx$&0.77/0.85 (70\%)\\
ARCPOM&1908.7-/57.4- & 1965.2-/49.6- & 1927.4-/45.1- & 2003.2-/55.4- & \cellcolor[rgb]{.9,.9,.9}1955.1$\approx$/48.9$\approx$ & 1987.0-/76.6- & 1979.0-/14.7- & 1922.9-/32.5$\approx$&0.65/0.78 (56\%)\\
ALSAC&1896.8-/65.2- & 1975.8-/49.8- & 1930.6-/48.8- & 1988.3-/48.8- & \cellcolor[rgb]{.9,.9,.9}1984.0-/52.3$\approx$ & 2021.2-/92.6- & 1954.7-/15.7- & 1917.4-/35.3-&0.65/0.74 (53\%)\\
\midrule
SAC&1926.4-/64.0- & 1952.3-/50.2- & 1912.7-/57.6- & 1999.6-/53.3- & \cellcolor[rgb]{.9,.9,.9}1968.5$\approx$/48.1+ & 2005.9-/75.1- & 1953.2-/13.7- & 1943.7-/41.4-&0.66/0.75 (52\%)\\
PPO&1916.9-/61.2- & 1997.0-/58.3- & 1939.7-/42.2- & 1992.1-/51.8- & \cellcolor[rgb]{.9,.9,.9}1988.7-/56.3- & 2032.2-/89.1- & 1956.5-/15.0- & 1925.0-/36.4-&0.62/0.73 (53\%)\\
ASAC&1887.9-/74.7- & 1945.1-/46.2- & 1917.5-/55.4- & 1942.9-/34.7$\approx$ & \cellcolor[rgb]{.9,.9,.9}1923.6+/47.1+ & 1981.1-/93.3- & 1965.4-/19.3- & 1955.4-/49.5-&0.72/0.72 (56\%)\\
APPO&1906.1-/64.7- & 1976.0-/53.8- & 1933.9-/47.6- & 1994.5-/51.9- & \cellcolor[rgb]{.9,.9,.9}1975.6$\approx$/53.2$\approx$ & 2021.1-/92.9- & 1961.9-/15.4- & 1921.3-/35.2-&0.64/0.73 (52\%)\\
\midrule
MIX&1939.9-/60.3- & 2000.9-/60.4- & 1932.2-/45.0- & 2006.1-/52.8- & \cellcolor[rgb]{.9,.9,.9}2028.2-/64.2- & 2027.1-/69.5- & 1969.2-/16.1- & 1971.0-/46.9-&0.55/0.72 (54\%)\\
FCFS&2081.1-/90.2- & 2084.5-/82.3- & 2014.7-/64.2- & 2136.7-/105.0- & \cellcolor[rgb]{.9,.9,.9}2194.6-/122.5- & 2123.5-/85.2- & 1927.9$\approx$/11.2- & 1933.6-/\textbf{27.4}$\approx$&0.28/0.50 (37\%)\\
EDD&1903.2-/\textbf{33.9}+ & 1968.6-/29.6$\approx$ & 1977.8-/\textbf{26.5}$\approx$ & 1988.1-/32.8$\approx$ &\cellcolor[rgb]{.9,.9,.9} 1950.7$\approx$/\textbf{33.7}+ & 2016.8-/46.8$\approx$ & 1940.5-/12.1- & 2020.8-/45.3-&0.60/0.94 (86\%)\\
NVF&1876.5-/69.6- & 1958.4-/56.4- & 1946.7-/49.6- & 2040.6-/66.7- & \cellcolor[rgb]{.9,.9,.9}1933.9+/56.3$\approx$ & 1953.4-/78.5- & 1996.5-/21.8- & 1944.6-/35.5-&0.64/0.70 (51\%)\\
STD&1868.8-/66.6- & 1961.7-/49.2- & 1921.3-/51.7- & 1970.7-/35.6- & \cellcolor[rgb]{.9,.9,.9}\textbf{1917.1}+/41.4+ & 1955.4-/62.7- & 1983.7-/30.4- & 1883.5-/35.1$\approx$&0.74/0.78 (53\%)\\
Random&2098.7-/124.5- & 2113.8-/103.1- & 2091.2-/143.7- & 2135.1-/123.1- & \cellcolor[rgb]{.9,.9,.9}2149.3-/119.0- & 2159.1-/129.3- & 2083.0-/70.2- & 2067.8-/89.5-&0.02/0.00 (12\%)\\
\midrule
& (0/0/16) / (1/0/15)& (0/0/16) / (0/1/15)& (0/0/16) / (0/1/15)& (0/0/16) / (0/2/14)& \cellcolor[rgb]{.9,.9,.9}(4/7/5) / (5/5/6)& (0/0/16) / (0/1/15)& (0/2/14) / (0/0/16)& (0/0/16) / (0/6/10)\\

\bottomrule
    \end{tabular}
    }

    \label{tab:l1o5}
\end{table*}

\begin{table*}[t]
    \centering
        \caption{Cross validation using leave-one-out on DMH-06 (highlighted with grey blocks): average makespan and tardiness over 30 independent trials of five different seeds on each instance. Bold numbers indicate the best makespan and tardiness. ``+'',``$\approx$'' and ``-'' indicate the policy performs statistically better/similar/worse than ACERL policy. The number of policies that are ``better'', ``similar'' and ``worse'' than ACERL in terms of makespan and tardiness on each instance is summarised in the bottom row. The last column with header ``$M/C (P)$'' indicates the average normalised makespan, tardiness and percentage of constraint satisfaction. Horizontal rules in the table separate different groups of algorithms.}
    \resizebox{\textwidth}{!}{
      \setlength{\tabcolsep}{2pt}
    \begin{tabular}{c|c|c|c|c|c|c|c|c|c}
\toprule
\multicolumn{1}{c|}{\multirow{2}{*}{Algorithm} }&DMH-01 & DMH-02 & DMH-03 & DMH-04 & DMH-05 & \cellcolor[rgb]{.9,.9,.9}DMH-06 & DMH-07 & DMH-08 & \multirow{2}{*}{$M/C (P)$} \\
&$F_m$/$F_t$ & $F_m$/$F_t$& $F_m$/$F_t$& $F_m$/$F_t$& $F_m$/$F_t$& \cellcolor[rgb]{.9,.9,.9}$F_m$/$F_t$& $F_m$/$F_t$& $F_m$/$F_t$&\\
\midrule
ACERL&\textbf{1782.4}/37.9 & \textbf{1866.0}/31.5 & \textbf{1838.5}/\textbf{24.1} & \textbf{1919.0}/\textbf{32.7} & \textbf{1852.4}/\textbf{26.7} & \cellcolor[rgb]{.9,.9,.9}\textbf{1933.1}/54.9 & 1932.8/\textbf{10.8} & \textbf{1861.0}/29.0&1.00/0.97 (87\%)\\
\midrule
MAPPO&1867.2-/71.3- & 1905.2-/38.0- & 1909.6-/58.2- & 1955.2-/42.6- & 1924.4-/50.2- & \cellcolor[rgb]{.9,.9,.9}1973.4-/85.2- & \textbf{1927.0}+/15.3- & 1863.6$\approx$/39.7-&0.84/0.76 (50\%)\\
RCPOM&1889.4-/64.6- & 1972.5-/49.8- & 1931.7-/45.3- & 1958.7-/39.6- & 1979.2-/51.6- & \cellcolor[rgb]{.9,.9,.9}1987.8-/70.1- & 1961.7-/16.0- & 1892.5-/28.6$\approx$&0.71/0.81 (63\%)\\
LSAC&1912.0-/59.4- & 1968.9-/56.3- & 1946.0-/45.9- & 2003.8-/54.0- & 1935.4-/40.8- & \cellcolor[rgb]{.9,.9,.9}2001.4-/88.4- & 1946.9$\approx$/11.9$\approx$ & 1917.2-/31.7$\approx$&0.68/0.77 (57\%)\\
AMAPPO&1879.2-/50.4- & 1940.6-/38.3- & 1884.0-/36.8- & 1984.2-/51.9- & 1934.4-/56.9- & \cellcolor[rgb]{.9,.9,.9}1989.2-/58.0$\approx$ & 1927.0+/13.7- & 1866.0$\approx$/\textbf{26.4}+&0.80/0.86 (70\%)\\
ARCPOM&1921.4-/67.0- & 1979.8-/49.1- & 1914.8-/42.2- & 2002.9-/57.5- & 1945.6-/44.8- & \cellcolor[rgb]{.9,.9,.9}1996.6-/76.0- & 1948.7-/14.0- & 1924.0-/35.4-&0.68/0.78 (57\%)\\
ALSAC&1908.6-/66.1- & 1957.6-/45.1- & 1952.7-/48.6- & 1976.7-/49.3- & 1969.3-/52.3- & \cellcolor[rgb]{.9,.9,.9}2004.7-/67.6- & 1964.3-/19.8- & 1919.5-/36.5-&0.67/0.78 (58\%)\\
\midrule
SAC&1910.5-/62.3- & 1942.9-/50.5- & 1911.8-/39.0- & 1985.2-/47.7- & 1946.6-/37.3- & \cellcolor[rgb]{.9,.9,.9}1986.9-/73.3- & 1974.7-/14.9- & 1914.3-/32.8-&0.70/0.81 (63\%)\\
PPO&1898.7-/61.3- & 1976.8-/53.9- & 1926.4-/44.7- & 1983.9-/47.3- & 1975.9-/53.4- & \cellcolor[rgb]{.9,.9,.9}2008.4-/88.3- & 1954.5-/15.5- & 1922.5-/36.8-&0.67/0.75 (56\%)\\
ASAC&1911.9-/70.7- & 1966.8-/51.3- & 1932.5-/40.7- & 1982.0-/52.5- & 1966.6-/45.1- & \cellcolor[rgb]{.9,.9,.9}2027.6-/86.0- & 1966.3-/17.8$\approx$ & 1921.8-/36.7-&0.65/0.75 (61\%)\\
APPO&1906.0-/66.2- & 1971.5-/52.1- & 1934.2-/48.1- & 1983.5-/48.0- & 1976.3-/53.4- & \cellcolor[rgb]{.9,.9,.9}2015.5-/94.1- & 1956.3-/15.6- & 1916.0-/35.6-&0.67/0.74 (53\%)\\
\midrule
MIX&1939.9-/60.3- & 2000.9-/60.4- & 1932.2-/45.0- & 2006.1-/52.8- & 2028.2-/64.2- & \cellcolor[rgb]{.9,.9,.9}2027.1-/69.5$\approx$ & 1969.2-/16.1$\approx$ & 1971.0-/46.9-&0.56/0.73 (54\%)\\
FCFS&2081.1-/90.2- & 2084.5-/82.3- & 2014.7-/64.2- & 2136.7-/105.0- & 2194.6-/122.5- & \cellcolor[rgb]{.9,.9,.9}2123.5-/85.2- & 1927.9+/11.2$\approx$ & 1933.6-/27.4$\approx$&0.28/0.51 (37\%)\\
EDD&1903.2-/\textbf{33.9}$\approx$ & 1968.6-/\textbf{29.6}+ & 1977.8-/26.5$\approx$ & 1988.1-/32.8$\approx$ & 1950.7-/33.7- & \cellcolor[rgb]{.9,.9,.9}2016.8-/\textbf{46.8}+ & 1940.5-/12.1- & 2020.8-/45.3-&0.60/0.95 (86\%)\\
NVF&1876.5-/69.6- & 1958.4-/56.4- & 1946.7-/49.6- & 2040.6-/66.7- & 1933.9-/56.3- & \cellcolor[rgb]{.9,.9,.9}1953.4-/78.5- & 1996.5-/21.8- & 1944.6-/35.5$\approx$&0.65/0.70 (51\%)\\
STD&1868.8-/66.6- & 1961.7-/49.2- & 1921.3-/51.7- & 1970.7-/35.6$\approx$ & 1917.1-/41.4- & \cellcolor[rgb]{.9,.9,.9}1955.4-/62.7$\approx$ & 1983.7-/30.4- & 1883.5-/35.1$\approx$&0.75/0.79 (53\%)\\
Random&2098.7-/124.5- & 2113.8-/103.1- & 2091.2-/143.7- & 2135.1-/123.1- & 2149.3-/119.0- & \cellcolor[rgb]{.9,.9,.9}2159.1-/129.3- & 2083.0-/70.2- & 2067.8-/89.5-&0.02/0.00 (12\%)\\
\midrule
& (0/0/16) / (0/1/15)& (0/0/16) / (1/0/15)& (0/0/16) / (0/1/15)& (0/0/16) / (0/2/14)& (0/0/16) / (0/0/16)& \cellcolor[rgb]{.9,.9,.9}(0/0/16) / (1/3/12)& (3/1/12) / (0/4/12)& (0/2/14) / (1/5/10)\\

\bottomrule
    \end{tabular}
    }

    \label{tab:l1o6}
\end{table*}

\begin{table*}[htbp]
    \centering
        \caption{Cross validation using leave-one-out on DMH-07 (highlighted with grey blocks): average makespan and tardiness over 30 independent trials of five different seeds on each instance. Bold numbers indicate the best makespan and tardiness. ``+'',``$\approx$'' and ``-'' indicate the policy performs statistically better/similar/worse than ACERL policy. The number of policies that are ``better'', ``similar'' and ``worse'' than ACERL in terms of makespan and tardiness on each instance is summarised in the bottom row. The last column with header ``$M/C (P)$'' indicates the average normalised makespan, tardiness and percentage of constraint satisfaction. Horizontal rules in the table separate different groups of algorithms.}
    \resizebox{\textwidth}{!}{
      \setlength{\tabcolsep}{2pt}
    \begin{tabular}{c|c|c|c|c|c|c|c|c|c}
\toprule
\multicolumn{1}{c|}{\multirow{2}{*}{Algorithm} }&DMH-01 & DMH-02 & DMH-03 & DMH-04 & DMH-05 & DMH-06 & \cellcolor[rgb]{.9,.9,.9}DMH-07 & DMH-08 & \multirow{2}{*}{$M/C (P)$} \\
&$F_m$/$F_t$ & $F_m$/$F_t$& $F_m$/$F_t$& $F_m$/$F_t$& $F_m$/$F_t$& $F_m$/$F_t$&\cellcolor[rgb]{.9,.9,.9} $F_m$/$F_t$& $F_m$/$F_t$&\\
\midrule
ACREL &  \textbf{1783.2}/35.6 & \textbf{1865.2}/33.7 & \textbf{1852.8}/27.8 & \textbf{1888.6}/33.1 & \textbf{1864.6}/35.2 & \textbf{1848.1}/\textbf{38.4} &\cellcolor[rgb]{.9,.9,.9} \cellcolor[rgb]{.9,.9,.9}1962.6/13.4 & \textbf{1848.6}/34.6&\textbf{0.97}/\textbf{0.97} (\textbf{95\%})\\
\midrule
ACERL&\textbf{1783.2}/35.6 & \textbf{1865.2}/33.7 & \textbf{1852.8}/27.8 & \textbf{1888.6}/33.1 & \textbf{1864.6}/35.2 & \textbf{1848.1}/\textbf{38.4} & \cellcolor[rgb]{.9,.9,.9}1962.6/13.4 & \textbf{1848.6}/34.6&0.97/0.97 (95\%)\\
MAPPO&1910.4-/74.9- & 1927.4-/37.8- & 1890.2-/46.0- & 1941.6-/42.0- & 1921.0-/41.1- & 1958.8-/92.9- & \cellcolor[rgb]{.9,.9,.9}\textbf{1927.0}+/12.8+ & 1886.4-/27.7+&0.78/0.80 (62\%)\\
RCPOM&1903.5-/67.9- & 1951.6-/45.3- & 1946.0-/48.9- & 1970.2-/45.7- & 1946.5-/42.4- & 1999.5-/88.1- & \cellcolor[rgb]{.9,.9,.9}1973.0-/15.6$\approx$ & 1909.6-/40.5-&0.66/0.77 (61\%)\\
LSAC&1900.2-/60.4- & 1949.1-/50.9- & 1907.0-/36.2- & 1983.2-/52.9- & 1950.9-/45.2- & 2015.4-/80.1- & \cellcolor[rgb]{.9,.9,.9}1979.6-/17.5$\approx$ & 1905.2-/33.1$\approx$&0.66/0.79 (59\%)\\
AMAPPO&1869.4-/58.7- & 1935.6-/37.9- & 1894.0-/45.4- & 1931.6-/36.7- & 1909.4-/42.3- & 1974.8-/76.6- & \cellcolor[rgb]{.9,.9,.9}1932.8+/14.3- & 1871.6-/26.7+&0.80/0.85 (70\%)\\
ARCPOM&1898.4-/58.8- & 1962.6-/61.9- & 1933.2-/53.9- & 1987.8-/52.7- & 1964.3-/39.7$\approx$ & 2039.4-/79.4- & \cellcolor[rgb]{.9,.9,.9}1951.8+/15.0- & 1925.3-/\textbf{26.5}+&0.63/0.78 (50\%)\\
ALSAC&1895.7-/63.5- & 1973.5-/51.6- & 1918.0-/40.3- & 1987.7-/49.1- & 1957.8-/46.4- & 2000.6-/81.3- & \cellcolor[rgb]{.9,.9,.9}1961.6$\approx$/16.3$\approx$ & 1912.1-/34.1+&0.66/0.78 (57\%)\\
\midrule
SAC&1896.2-/59.4- & 1972.4-/54.8- & 1941.7-/55.9- & 1965.1-/48.5- & 1947.2-/44.0- & 2002.0-/78.8- & \cellcolor[rgb]{.9,.9,.9}1969.8$\approx$/18.4- & 1921.0-/34.0$\approx$&0.65/0.77 (59\%)\\
PPO&1891.5-/63.2- & 1959.4-/53.1- & 1925.2-/43.4- & 1979.3-/47.3- & 1967.1-/54.9- & 2008.9-/90.3- & \cellcolor[rgb]{.9,.9,.9}1969.7$\approx$/18.0$\approx$ & 1917.1-/34.4$\approx$&0.65/0.75 (54\%)\\
ASAC&1902.2-/66.2- & 1919.5-/42.6- & 1941.0-/46.1- & 1979.0-/45.4- & 1936.3-/48.2- & 1984.4-/74.8- & \cellcolor[rgb]{.9,.9,.9}1964.5$\approx$/16.4$\approx$ & 1930.4-/38.7$\approx$&0.68/0.79 (58\%)\\
APPO&1894.0-/69.7- & 1953.5-/49.5- & 1931.7-/51.4- & 1979.6-/45.7- & 1958.4-/52.5- & 2003.7-/99.0- & \cellcolor[rgb]{.9,.9,.9}1977.1-/18.3- & 1908.2-/34.8$\approx$&0.65/0.73 (53\%)\\
\midrule
MIX&1939.9-/60.3- & 2000.9-/60.4- & 1932.2-/45.0- & 2006.1-/52.8- & 2028.2-/64.2- & 2027.1-/69.5- & \cellcolor[rgb]{.9,.9,.9}1969.2$\approx$/16.1$\approx$ & 1971.0-/46.9-&0.53/0.73 (54\%)\\
FCFS&2081.1-/90.2- & 2084.5-/82.3- & 2014.7-/64.2- & 2136.7-/105.0- & 2194.6-/122.5- & 2123.5-/85.2- & \cellcolor[rgb]{.9,.9,.9}1927.9+/\textbf{11.2}+ & 1933.6-/27.4+&0.28/0.50 (37\%)\\
EDD&1903.2-/\textbf{33.9}$\approx$ & 1968.6-/\textbf{29.6}+ & 1977.8-/\textbf{26.5}$\approx$ & 1988.1-/\textbf{32.8}$\approx$ & 1950.7-/\textbf{33.7}$\approx$ & 2016.8-/46.8$\approx$ & \cellcolor[rgb]{.9,.9,.9}1940.5+/12.1$\approx$ & 2020.8-/45.3-&0.58/0.95 (86\%)\\
NVF&1876.5-/69.6- & 1958.4-/56.4- & 1946.7-/49.6- & 2040.6-/66.7- & 1933.9-/56.3- & 1953.4-/78.5- & \cellcolor[rgb]{.9,.9,.9}1996.5-/21.8- & 1944.6-/35.5$\approx$&0.61/0.71 (51\%)\\
STD&1868.8-/66.6- & 1961.7-/49.2- & 1921.3-/51.7- & 1970.7-/35.6$\approx$ & 1917.1-/41.4$\approx$ & 1955.4-/62.7- & \cellcolor[rgb]{.9,.9,.9}1983.7$\approx$/30.4$\approx$ & 1883.5-/35.1$\approx$&0.71/0.79 (53\%)\\
Random&2098.7-/124.5- & 2113.8-/103.1- & 2091.2-/143.7- & 2135.1-/123.1- & 2149.3-/119.0- & 2159.1-/129.3- & \cellcolor[rgb]{.9,.9,.9}2083.0-/70.2- & 2067.8-/89.5-&0.02/0.00 (12\%)\\
\midrule
& (0/0/16) / (0/1/15)& (0/0/16) / (1/0/15)& (0/0/16) / (0/1/15)& (0/0/16) / (0/2/14)& (0/0/16) / (0/3/13)& (0/0/16) / (0/1/15)& \cellcolor[rgb]{.9,.9,.9}(5/6/5) / (2/8/6)& (0/0/16) / (5/7/4)\\

\bottomrule
    \end{tabular}
    }

    \label{tab:l1o7}
\end{table*}

\begin{table*}[htbp]
    \centering
        \caption{Cross validation using leave-one-out on DMH-08 (highlighted with grey blocks): average makespan and tardiness over 30 independent trials of five different seeds on each instance. Bold numbers indicate the best makespan and tardiness. ``+'',``$\approx$'' and ``-'' indicate the policy performs statistically better/similar/worse than ACERL policy. The number of policies that are ``better'', ``similar'' and ``worse'' than ACERL in terms of makespan and tardiness on each instance is summarised in the bottom row. The last column with header ``$M/C (P)$'' indicates the average normalised makespan, tardiness and percentage of constraint satisfaction. Horizontal rules in the table separate different groups of algorithms.}
    \resizebox{\textwidth}{!}{
      \setlength{\tabcolsep}{2pt}
    \begin{tabular}{c|c|c|c|c|c|c|c|c|c}
\toprule
\multicolumn{1}{c|}{\multirow{2}{*}{Algorithm} }&DMH-01 & DMH-02 & DMH-03 & DMH-04 & DMH-05 & DMH-06 & DMH-07 & \cellcolor[rgb]{.9,.9,.9}DMH-08 & \multirow{2}{*}{$M/C (P)$} \\
&$F_m$/$F_t$ & $F_m$/$F_t$& $F_m$/$F_t$& $F_m$/$F_t$& $F_m$/$F_t$& $F_m$/$F_t$& $F_m$/$F_t$&\cellcolor[rgb]{.9,.9,.9} $F_m$/$F_t$&\\
\midrule
ACERL&\textbf{1793.3}/\textbf{26.5} & \textbf{1867.0}/32.5 & \textbf{1843.8}/\textbf{23.2} & \textbf{1908.6}/37.5 & \textbf{1832.4}/\textbf{25.9} & \textbf{1861.6}/\textbf{37.6} & 1929.6/\textbf{8.0} & \cellcolor[rgb]{.9,.9,.9}1942.3/28.8&0.95/0.99 (100\%)\\
\midrule
MAPPO&1883.4-/72.5- & 1919.0-/39.5- & 1911.8-/54.5- & 1950.2-/43.8- & 1919.4-/47.4- & 1998.8-/116.2- & 1927.0+/12.9- & \cellcolor[rgb]{.9,.9,.9}1884.4+/38.7-&0.77/0.71 (55\%)\\
RCPOM&1910.6-/67.9- & 1970.6-/53.4- & 1924.3-/48.0- & 2003.3-/47.0- & 1932.0-/44.3- & 1995.6-/71.1- & 1948.6-/12.5- & \cellcolor[rgb]{.9,.9,.9}1905.0+/35.4-&0.67/0.77 (56\%)\\
LSAC&1925.4-/61.0- & 1983.2-/53.3- & 1942.6-/48.9- & 1997.1-/55.5- & 2008.1-/57.7- & 2023.9-/73.5- & 1967.2-/18.9- & \cellcolor[rgb]{.9,.9,.9}1938.0+/39.8-&0.58/0.72 (55\%)\\
AMAPPO&1869.6-/51.8- & 1911.2-/39.0- & 1909.0-/55.4- & 1958.4-/44.9- & 1928.4-/43.4- & 1999.0-/67.3- & \textbf{1926.8}+/13.8- & \cellcolor[rgb]{.9,.9,.9}\textbf{1852.8}+/29.3-&0.80/0.82 (55\%)\\
ARCPOM&1902.5-/70.1- & 1950.9-/49.9- & 1925.4-/53.3- & 1999.2-/51.0- & 1959.3-/56.0- & 1999.8-/99.0- & 1959.2-/19.8- & \cellcolor[rgb]{.9,.9,.9}1938.8$\approx$/43.2-&0.64/0.68 (48\%)\\
ALSAC&1891.4-/61.9- & 1968.9-/52.0- & 1924.8-/44.4- & 1982.7-/50.1- & 1923.3-/43.0- & 2009.1-/80.4- & 1965.5-/18.6- & \cellcolor[rgb]{.9,.9,.9}1926.5+/34.7-&0.66/0.75 (57\%)\\
\midrule
SAC&1905.8-/55.1- & 1964.0-/53.7- & 1927.9-/39.6- & 1992.8-/53.7- & 2002.4-/53.0- & 2026.9-/77.3- & 1950.1$\approx$/13.8- & \cellcolor[rgb]{.9,.9,.9}1941.8$\approx$/41.9-&0.62/0.75 (61\%)\\
PPO&1922.0-/61.6- & 2006.3-/60.5- & 1937.1-/43.4- & 2003.1-/53.3- & 2012.4-/59.9- & 2041.6-/91.8- & 1944.6$\approx$/14.6- & \cellcolor[rgb]{.9,.9,.9}1942.7$\approx$/39.8-&0.57/0.70 (52\%)\\
ASAC&1905.2-/64.2- & 1974.6-/57.0- & 1926.7-/48.0- & 1981.7-/48.1- & 1961.2-/48.7- & 2009.6-/84.8- & 1978.1-/17.6- & \cellcolor[rgb]{.9,.9,.9}1937.6$\approx$/37.5-&0.62/0.72 (54\%)\\
APPO&1914.9-/65.0- & 1985.9-/55.7- & 1946.2-/44.8- & 1991.4-/49.4- & 1990.4-/55.7- & 2041.7-/97.3- & 1949.0-/14.5- &\cellcolor[rgb]{.9,.9,.9} 1930.0$\approx$/39.5-&0.60/0.70 (52\%)\\
\midrule
MIX&1939.9-/60.3- & 2000.9-/60.4- & 1932.2-/45.0- & 2006.1-/52.8- & 2028.2-/64.2- & 2027.1-/69.5- & 1969.2-/16.1- & \cellcolor[rgb]{.9,.9,.9}1971.0$\approx$/46.9-&0.53/0.71 (54\%)\\
FCFS&2081.1-/90.2- & 2084.5-/82.3- & 2014.7-/64.2- & 2136.7-/105.0- & 2194.6-/122.5- & 2123.5-/85.2- & 1927.9$\approx$/11.2- & \cellcolor[rgb]{.9,.9,.9}1933.6$\approx$/\textbf{27.4}+&0.28/0.49 (37\%)\\
EDD&1903.2-/33.9$\approx$ & 1968.6-/\textbf{29.6}+ & 1977.8-/26.5- & 1988.1-/\textbf{32.8}+ & 1950.7-/33.7- & 2016.8-/46.8$\approx$ & 1940.5-/12.1- &\cellcolor[rgb]{.9,.9,.9} 2020.8-/45.3$\approx$&0.58/0.92 (86\%)\\
NVF&1876.5-/69.6- & 1958.4-/56.4- & 1946.7-/49.6- & 2040.6-/66.7- & 1933.9-/56.3- & 1953.4-/78.5- & 1996.5-/21.8- &\cellcolor[rgb]{.9,.9,.9} 1944.6$\approx$/35.5-&0.61/0.69 (51\%)\\
STD&1868.8-/66.6- & 1961.7-/49.2- & 1921.3-/51.7- & 1970.7-/35.6$\approx$ & 1917.1-/41.4- & 1955.4-/62.7- & 1983.7-/30.4- & \cellcolor[rgb]{.9,.9,.9}1883.5+/35.1-&0.72/0.77 (53\%)\\
Random&2098.7-/124.5- & 2113.8-/103.1- & 2091.2-/143.7- & 2135.1-/123.1- & 2149.3-/119.0- & 2159.1-/129.3- & 2083.0-/70.2- & \cellcolor[rgb]{.9,.9,.9}2067.8-/89.5-&0.02/0.00 (12\%)\\
\midrule
& (0/0/16) / (0/1/15)& (0/0/16) / (1/0/15)& (0/0/16) / (0/0/16)& (0/0/16) / (1/1/14)& (0/0/16) / (0/0/16)& (0/0/16) / (0/1/15)& (2/3/11) / (0/0/16)& \cellcolor[rgb]{.9,.9,.9}(6/8/2) / (1/1/14)\\

\bottomrule
    \end{tabular}
    }

    \label{tab:l1o8}
\end{table*}

\clearpage
\subsection{Supplementary experimental results of ablation study}\label{sec:supab}
Tables~\ref{tab:ablation_single} and~\ref{tab:ablation_unseen_single} present the results of training with only one instance. When only one instance is included in the training process, although the policy exhibits excellent performance on the training instance, it performs poorly on new instances. Specifically, it obtains a high makespan and struggles to satisfy the constraints on new instances.
This reveals that the performance of the policy obtained by considering only one instance during training is limited.

\begin{table*}[!ht]
    \centering
        \caption{Ablation study on training instances (DMH-01 to DMH-08). ``1--8" means that the agent is only trained on the specific instance. Bold number indicates the best makespan and tardiness. ``$M/C (P)$'' indicates the average normalised makespan, tardiness and percentage of constraint satisfaction.  Strategy denotes the ranking strategies including ISR, CF, R and F.
        }
    \resizebox{\textwidth}{!}{
      \setlength{\tabcolsep}{2pt}
    \begin{tabular}{lc|c|c|c|c|c|c|c|c|c}
\toprule
\multirow{2}{*}{Strategy}& \multirow{2}{*}{Mode}&DMH-01 & DMH-02 & DMH-03 & DMH-04 & DMH-05 & DMH-06 & DMH-07 & DMH-08 & \multirow{2}{*}{$M/C (P)$} \\
& &$F_m$/$F_t$ & $F_m$/$F_t$& $F_m$/$F_t$& $F_m$/$F_t$& $F_m$/$F_t$& $F_m$/$F_t$& $F_m$/$F_t$& $F_m$/$F_t$&\\
\midrule

ISR&DMH-1&\textbf{1750.8}/\textbf{19.9} & 1909.8/35.5 & 1926.8/47.0 & 2013.4/65.7 & 1985.4/65.6 & 1949.2/54.7 & 1963.0/12.7 & 1924.1/29.5&0.67/0.77 (63\%)\\
CF&DMH-1&1772.2/25.8 & 1936.2/42.4 & 1956.3/50.7 & 2001.9/49.3 & 2016.4/64.4 & 1962.9/53.5 & 1957.0/11.0 & 1924.6/21.1&0.63/0.79 (72\%)\\
R&DMH-1&1769.6/28.3 & 1950.9/52.9 & 1951.2/50.1 & 1963.3/41.9 & 2004.9/49.6 & 1979.4/59.3 & 1960.7/12.9 & 1934.4/23.4&0.63/0.78 (68\%)\\
F&DMH-1&1785.8/41.2 & 1951.2/52.8 & 1930.7/47.0 & 1997.2/58.3 & 1966.9/55.0 & 1941.4/65.5 & 1973.1/15.1 & 1906.6/24.9&0.65/0.73 (61\%)\\
ISR&DMH-2&1907.7/67.5 & 1846.0/27.4 & 1956.1/63.2 & 2002.6/53.6 & 1932.5/51.7 & 2016.7/69.0 & 1962.0/14.6 & 1921.6/36.8&0.62/0.71 (55\%)\\
CF&DMH-2&1929.7/73.0 & 1840.8/24.3 & 1926.1/47.8 & 1993.1/45.5 & 1926.4/48.1 & 2030.8/96.7 & 1949.8/12.6 & 1938.2/40.6&0.63/0.70 (56\%)\\
R&DMH-2&1972.8/91.7 & \textbf{1840.2}/\textbf{23.9} & 1962.9/71.4 & 2018.7/55.3 & 1950.6/54.1 & 2032.7/85.9 & 1939.2/12.4 & 1937.6/44.8&0.59/0.64 (43\%)\\
F&DMH-2&1912.5/63.7 & 1842.6/24.3 & 1923.2/44.1 & 1981.9/46.7 & 1947.6/50.3 & 2002.0/89.4 & 1980.8/21.6 & 1900.2/30.8&0.65/0.72 (65\%)\\
ISR&DMH-3&1931.0/58.6 & 1964.4/41.4 & \textbf{1831.0}/26.1 & 1980.6/54.0 & 1984.8/39.5 & 1985.8/46.4 & 1946.0/11.1 & 1934.9/41.8&0.63/0.77 (68\%)\\
CF&DMH-3&1916.8/49.2 & 1929.8/45.3 & 1834.8/\textbf{18.9} & 1995.2/59.6 & 1949.3/42.2 & 1961.6/57.1 & 1992.2/13.1 & 1908.9/30.5&0.64/0.78 (63\%)\\
R&DMH-3&1938.4/67.7 & 1922.8/38.9 & 1839.4/23.2 & 1969.8/52.6 & 1955.6/56.4 & 2011.9/76.5 & 1927.0/11.5 & 1923.8/42.7&0.67/0.71 (61\%)\\
F&DMH-3&1915.0/67.4 & 1953.9/47.5 & 1838.6/27.3 & 2017.0/57.1 & 1940.1/37.4 & 2014.8/65.7 & 1963.9/19.9 & 1888.6/35.9&0.64/0.72 (51\%)\\
ISR&DMH-4&1913.8/72.8 & 1977.9/43.1 & 1928.6/46.1 & 1892.0/33.7 & 1963.7/48.2 & 2039.3/103.7 & 1963.5/15.3 & 1983.5/53.3&0.58/0.66 (63\%)\\
CF&DMH-4&1902.8/68.5 & 1953.2/44.9 & 1919.7/43.1 & 1885.4/\textbf{28.8} & 1919.9/47.1 & 2002.6/77.1 & 1976.4/18.0 & 1958.5/51.3&0.63/0.70 (62\%)\\
R&DMH-4&1975.6/81.5 & 1973.4/47.1 & 1911.4/61.5 & 1878.8/32.9 & 1935.1/42.3 & 2077.2/110.1 & 1978.6/15.4 & 1917.7/38.8&0.59/0.65 (54\%)\\
F&DMH-4&1921.1/77.5 & 1969.2/47.8 & 1951.4/58.9 & \textbf{1877.2}/35.4 & 1936.7/34.5 & 2000.4/72.7 & 1990.7/27.4 & 1895.8/35.1&0.62/0.69 (57\%)\\
ISR&DMH-5&1893.4/69.3 & 1940.4/49.2 & 1940.7/58.0 & 1952.0/49.6 & 1825.0/14.6 & 1962.6/70.1 & 1998.3/22.7 & 1893.7/24.2&0.66/0.73 (63\%)\\
CF&DMH-5&1926.3/84.7 & 1933.2/39.2 & 1960.9/61.1 & 1999.4/59.6 & 1813.4/\textbf{12.3} & 1956.5/96.6 & 2012.0/28.6 & 1923.6/33.3&0.60/0.66 (54\%)\\
R&DMH-5&1937.5/86.2 & 1942.6/39.4 & 1911.9/63.6 & 1988.2/54.6 & \textbf{1809.0}/15.8 & 1985.5/102.9 & 2048.8/30.3 & 1912.2/26.0&0.59/0.66 (57\%)\\
F&DMH-5&1919.8/85.2 & 1929.1/46.5 & 1932.4/67.7 & 1959.5/47.9 & 1826.0/18.0 & 1953.5/98.0 & 2035.2/31.0 & 1902.4/30.1&0.63/0.64 (52\%)\\
ISR&DMH-6&1907.0/47.7 & 1968.4/56.3 & 1962.9/52.0 & 1964.6/38.3 & 1969.1/45.4 & 1835.0/31.4 & 1941.6/9.0 & 1899.4/27.8&0.67/0.79 (71\%)\\
CF&DMH-6&1890.1/49.4 & 1910.7/36.4 & 1931.6/45.3 & 1988.1/51.7 & 1951.2/40.7 & \textbf{1824.6}/\textbf{26.7} & 1981.1/19.2 & 1902.1/30.8&0.68/0.80 (83\%)\\
R&DMH-6&1882.5/53.4 & 1932.6/48.0 & 1941.7/41.0 & 1964.7/44.3 & 1960.8/50.2 & 1846.0/33.6 & 1976.6/16.4 & 1894.6/29.5&0.68/0.78 (74\%)\\
F&DMH-6&1905.7/56.3 & 1923.3/42.7 & 1935.2/43.9 & 1959.1/52.6 & 1935.6/37.7 & 1830.4/29.9 & 2005.9/25.1 & 1896.4/33.5&0.67/0.77 (72\%)\\
ISR&DMH-7&1972.5/84.4 & 1980.2/42.6 & 1919.0/37.2 & 2002.8/53.2 & 2033.6/57.6 & 2067.5/116.8 & 1916.0/9.8 & 1954.2/33.5&0.52/0.64 (52\%)\\
CF&DMH-7&1949.5/77.7 & 1950.4/46.8 & 1952.5/45.2 & 2024.3/49.1 & 1964.8/50.6 & 2059.0/107.1 & 1927.0/\textbf{3.3} & 1922.5/41.5&0.55/0.66 (67\%)\\
R&DMH-7&1939.7/84.2 & 1991.0/52.0 & 1937.1/42.9 & 2055.9/57.6 & 1983.8/46.1 & 2078.0/104.6 & \textbf{1910.6}/12.9 & 1917.2/39.6&0.53/0.63 (53\%)\\
F&DMH-7&1924.1/76.7 & 1989.8/58.1 & 1954.6/51.7 & 1991.5/57.5 & 1961.0/37.0 & 2076.4/102.2 & 1910.6/14.8 & 1898.3/34.8&0.58/0.64 (53\%)\\
ISR&DMH-8&1962.8/97.8 & 1992.0/43.2 & 1952.6/52.3 & 2038.6/62.5 & 1947.3/40.0 & 2080.2/131.2 & 1964.7/19.8 & \textbf{1832.4}/26.6&0.54/0.59 (54\%)\\
CF&DMH-8&1965.0/69.3 & 1955.5/50.7 & 1915.6/41.4 & 1971.4/38.3 & 1980.1/38.5 & 2001.3/79.0 & 1984.9/25.1 & 1838.4/\textbf{15.1}&0.61/0.73 (58\%)\\
R&DMH-8&1934.0/84.6 & 2015.1/53.2 & 1926.1/57.4 & 1998.0/52.3 & 1957.0/44.0 & 2048.4/104.2 & 1939.0/8.5 & 1837.6/26.0&0.60/0.65 (53\%)\\
F&DMH-8&1948.4/90.9 & 1969.2/44.7 & 1889.0/62.8 & 2004.2/46.6 & 1960.7/50.0 & 2039.6/133.2 & 1965.7/26.6 & 1842.8/29.3&0.61/0.58 (51\%)\\

\bottomrule
 \multicolumn{5}{l}{``ISR": Intrinsic stochastic ranking with rank-based fitness} & \multicolumn{5}{l}{``CF": Weighted sum of raw rewards and penalties} \\
  \multicolumn{5}{l}{``R": Raw reward with rank-based fitness} & \multicolumn{3}{l}{``F": Raw reward}\\
\multicolumn{5}{l}{``AIS": Adaptive instance sampler} & \multicolumn{3}{l}{``Uniform": Fixed instances}& \multicolumn{3}{l}{``Random": Randomly selecting instances}\\
    \end{tabular}
    }

    \label{tab:ablation_single}
\end{table*}

\begin{table*}[!t]
    \centering
        \caption{Ablation study on test instances (DMH-09 to DMH-16). ``1--8" means that the agent is only trained on the specific instance. Bold number indicates the best makespan and tardiness. ``$M/C (P)$'' indicates the average normalised makespan, tardiness and percentage of constraint satisfaction.  Strategy denotes the ranking strategies including ISR, CF, R and F.
        }
    \resizebox{\textwidth}{!}{
      \setlength{\tabcolsep}{2pt}
    \begin{tabular}{lc|c|c|c|c|c|c|c|c|c}
\toprule
\multirow{2}{*}{Strategy}& \multirow{2}{*}{Mode}&DMH-09 & DMH-10 & DMH-11 & DMH-12 & DMH-13 & DMH-14 & DMH-15 & DMH-16 & \multirow{2}{*}{$M/C (P)$} \\
& &$F_m$/$F_t$ & $F_m$/$F_t$& $F_m$/$F_t$& $F_m$/$F_t$& $F_m$/$F_t$& $F_m$/$F_t$& $F_m$/$F_t$& $F_m$/$F_t$&\\
\midrule

ISR&DMH-1&1791.4/30.9 & 1906.9/34.5 & 1942.3/49.9 & 2025.6/71.4 & 1986.4/66.7 & 1925.4/51.1 & 1979.4/15.7 & 1926.4/28.4&0.65/0.76 (65\%)\\
CF&DMH-1&1773.2/\textbf{26.2} & 1941.5/43.1 & 1956.6/52.6 & 1982.5/47.0 & 2014.8/64.9 & 1977.2/55.4 & 1958.8/11.0 & 1925.9/\textbf{22.0}&0.64/0.80 (70\%)\\
R&DMH-1&\textbf{1770.4}/28.7 & 1950.8/52.8 & 1945.2/50.1 & 1967.3/41.4 & 2023.5/52.7 & 1979.7/61.2 & 1964.7/13.3 & 1930.9/25.8&0.64/0.79 (65\%)\\
F&DMH-1&1787.2/41.4 & 1945.1/51.9 & 1894.5/46.4 & 1983.7/51.1 & 1976.9/58.0 & 1947.3/66.2 & 1977.5/17.0 & 1928.8/27.5&0.68/0.75 (63\%)\\
ISR&DMH-2&1884.5/70.1 & 1859.4/27.6 & 1956.5/61.5 & 2008.6/53.8 & 1957.8/50.4 & 2007.9/75.9 & 1965.8/15.6 & 1920.4/35.9&0.63/0.72 (55\%)\\
CF&DMH-2&1913.5/70.1 & 1839.8/23.7 & 1940.4/45.7 & 1986.8/44.3 & 1956.6/42.5 & 2028.0/99.8 & 1951.6/12.7 & 1953.4/41.9&0.64/0.73 (60\%)\\
R&DMH-2&1974.2/93.1 & \textbf{1839.2}/\textbf{23.5} & 1961.6/71.6 & 2003.7/50.8 & 1988.6/40.6 & 2026.9/87.8 & 1942.2/12.4 & 1943.7/41.7&0.60/0.69 (51\%)\\
F&DMH-2&1908.4/64.5 & 1856.0/24.4 & 1939.9/45.9 & 2000.1/48.9 & 1942.5/47.7 & 1986.3/89.4 & 1988.7/24.4 & 1904.5/31.5&0.65/0.73 (65\%)\\
ISR&DMH-3&1919.7/56.4 & 1945.8/41.2 & \textbf{1862.4}/31.1 & 1969.3/54.5 & 1941.3/33.1 & 1972.9/46.5 & 1951.2/11.4 & 1935.4/42.6&0.68/0.79 (70\%)\\
CF&DMH-3&1915.7/48.5 & 1926.0/42.7 & 1880.0/\textbf{23.6} & 2020.3/59.7 & 1930.2/38.8 & 1976.2/59.5 & 1994.1/14.6 & 1914.7/31.4&0.63/0.79 (65\%)\\
R&DMH-3&1913.7/66.7 & 1916.5/37.8 & 1872.0/25.8 & 1981.1/54.7 & 1944.5/44.7 & 2045.5/76.5 & \textbf{1924.8}/12.0 & 1925.3/43.4&0.68/0.74 (64\%)\\
F&DMH-3&1922.6/67.7 & 1952.4/47.8 & 1888.0/30.1 & 2027.7/57.4 & 1944.3/35.5 & 2013.3/66.1 & 1957.3/17.0 & 1913.2/38.4&0.62/0.74 (55\%)\\
ISR&DMH-4&1933.5/78.4 & 1972.3/43.1 & 1917.0/45.4 & \textbf{1900.0}/34.1 & 1958.9/39.1 & 2038.6/102.1 & 1959.8/18.9 & 1991.0/54.4&0.60/0.68 (63\%)\\
CF&DMH-4&1896.8/69.0 & 1952.0/44.6 & 1920.7/42.7 & 1927.0/\textbf{29.3} & 1959.0/43.2 & 2000.7/76.6 & 1965.0/18.0 & 1958.8/52.3&0.64/0.73 (65\%)\\
R&DMH-4&1960.4/81.7 & 1971.9/46.7 & 1899.9/61.2 & 1915.6/32.5 & 1911.0/33.6 & 2023.6/104.2 & 1979.6/15.8 & 1916.2/38.6&0.64/0.69 (54\%)\\
F&DMH-4&1917.7/77.8 & 1962.0/48.1 & 1959.8/56.7 & 1953.0/37.2 & 1934.0/36.0 & 1968.1/74.9 & 2000.4/25.4 & 1899.4/36.1&0.61/0.71 (58\%)\\
ISR&DMH-5&1883.8/65.6 & 1940.1/48.8 & 1940.8/58.3 & 1972.4/47.7 & 1888.3/21.7 & 1960.5/70.5 & 2010.6/27.8 & 1892.2/24.2&0.65/0.74 (62\%)\\
CF&DMH-5&1941.2/85.8 & 1935.3/39.0 & 1951.4/59.9 & 2000.7/58.5 & \textbf{1813.6}/\textbf{12.1} & 1951.5/106.7 & 1972.6/22.7 & 1898.7/31.8&0.66/0.68 (55\%)\\
R&DMH-5&1946.0/85.7 & 1945.4/39.2 & 1912.7/63.8 & 1994.2/53.7 & 1870.0/20.8 & 1982.0/102.3 & 1982.5/17.8 & 1911.6/26.4&0.64/0.69 (57\%)\\
F&DMH-5&1924.4/84.8 & 1928.8/46.0 & 1936.5/71.1 & 1965.6/46.9 & 1896.2/29.2 & 1982.7/104.1 & 1991.8/17.9 & 1920.0/35.0&0.64/0.66 (52\%)\\
ISR&DMH-6&1843.9/44.2 & 1965.5/56.2 & 1969.6/53.5 & 1962.5/37.8 & 1988.1/44.4 & 1864.4/37.7 & 1955.4/12.3 & 1909.5/29.5&0.68/0.81 (68\%)\\
CF&DMH-6&1895.8/49.7 & 1914.2/35.3 & 1944.3/47.2 & 1998.3/51.1 & 1974.6/40.6 & 1850.0/33.1 & 1978.4/18.8 & 1902.1/31.3&0.68/0.81 (77\%)\\
R&DMH-6&1875.8/55.6 & 1932.2/47.6 & 1943.9/41.1 & 1966.1/44.2 & 1958.7/50.3 & 1871.4/39.3 & 1977.0/15.4 & 1901.5/34.6&0.69/0.79 (71\%)\\
F&DMH-6&1890.1/54.9 & 1949.8/47.3 & 1907.5/42.8 & 1974.9/51.5 & 1947.3/40.5 & \textbf{1829.2}/\textbf{30.5} & 2006.0/25.5 & 1897.3/33.8&0.70/0.78 (75\%)\\
ISR&DMH-7&1964.7/81.9 & 1966.8/40.4 & 1911.0/41.8 & 2003.7/54.4 & 2004.0/45.8 & 2092.5/119.8 & 1943.4/11.2 & 1956.1/37.8&0.53/0.66 (56\%)\\
CF&DMH-7&1917.0/74.4 & 1951.4/44.8 & 1924.0/44.9 & 2028.8/48.7 & 1949.6/43.1 & 2064.3/109.0 & 1930.0/\textbf{3.4} & 1919.7/41.4&0.59/0.69 (70\%)\\
R&DMH-7&1960.3/85.3 & 1996.9/54.5 & 1923.5/39.2 & 2047.6/56.2 & 1997.9/47.1 & 2082.1/107.5 & 1966.1/15.6 & 1921.1/40.7&0.49/0.63 (55\%)\\
F&DMH-7&1935.6/79.2 & 2003.3/57.7 & 1929.8/51.4 & 2013.4/59.1 & 1941.0/33.9 & 2076.3/95.5 & 1963.0/17.5 & 1906.1/36.4&0.55/0.65 (51\%)\\
ISR&DMH-8&1973.4/105.1 & 1966.9/41.5 & 1928.8/52.3 & 2043.2/63.9 & 1917.1/26.6 & 2080.5/133.7 & 1972.0/18.8 & \textbf{1835.0}/27.2&0.58/0.62 (59\%)\\
CF&DMH-8&1881.8/59.5 & 1958.5/49.9 & 1907.9/44.2 & 1955.9/39.0 & 1954.3/33.7 & 2028.8/86.9 & 1984.4/25.2 & 1884.7/23.9&0.65/0.75 (61\%)\\
R&DMH-8&1956.6/88.0 & 1987.7/50.2 & 1965.8/63.7 & 2001.0/52.3 & 1920.5/32.2 & 2028.1/104.6 & 1942.0/8.4 & 1850.8/27.5&0.61/0.67 (56\%)\\
F&DMH-8&1944.8/95.8 & 1950.8/43.2 & 1925.4/64.0 & 2004.6/46.0 & 1894.5/26.5 & 2024.9/132.1 & 2000.2/35.9 & 1844.4/29.8&0.62/0.61 (54\%)\\

\bottomrule
 \multicolumn{5}{l}{``ISR": Intrinsic stochastic ranking with rank-based fitness} & \multicolumn{5}{l}{``CF": Weighted sum of raw rewards and penalties} \\
  \multicolumn{5}{l}{``R": Raw reward with rank-based fitness} & \multicolumn{3}{l}{``F": Raw reward}\\
\multicolumn{5}{l}{``AIS": Adaptive instance sampler} & \multicolumn{3}{l}{``Uniform": Fixed instances}& \multicolumn{3}{l}{``Random": Randomly selecting instances}\\

    \end{tabular}
    }

    \label{tab:ablation_unseen_single}
\end{table*}

\end{document}